\documentclass[letter,english]{article}

\usepackage{geometry}
\usepackage[T1]{fontenc}
\usepackage[latin9]{inputenc}

\usepackage{bm}
\usepackage{amsmath}
\usepackage{amssymb} 
\usepackage[unicode=true,
 bookmarks=false,
 breaklinks=false,pdfborder={0 0 1},colorlinks=false]
 {hyperref}
\hypersetup{
 colorlinks,citecolor=blue,filecolor=blue,linkcolor=blue,urlcolor=blue}

\usepackage{multirow}
\usepackage{xcolor,colortbl}
\definecolor{Gray}{gray}{0.85}
\usepackage{enumitem}

\usepackage{amsthm}
\usepackage{cite}  
\usepackage{comment}
\usepackage{natbib}
\usepackage{booktabs,mathtools}

\usepackage{graphicx}
\usepackage[linesnumbered,ruled,vlined]{algorithm2e}
\usepackage{algorithmic}

\usepackage{float}
\usepackage{multirow}
\usepackage{cancel}
\usepackage{pifont}

\usepackage{subcaption}

\usepackage{dsfont}
\usepackage{color}
\definecolor{yjc}{RGB}{125,0,0}
\definecolor{jiw}{RGB}{10,148,15}
\definecolor{lxs}{RGB}{138,43,226}

\newcommand{\blue}[1]{\textcolor{blue}{#1}}

\DeclareMathOperator*{\argmax}{arg\,max}

\allowdisplaybreaks

\newcommand{\Datab}{\mathcal{D}}
\newcommand{\pib}{\mu}
\renewcommand{\hat}{\widehat}
 
\newcommand{\synset}{\cT}
\newcommand{\syn}{\iota}
\newcommand{\nsyn}{\phi}
\newcommand{\logpart}{\zeta_0}
\newcommand{\logpartm}{\zeta_1}

\newcommand{\clipavgC}{C^{\star}_{\mathsf{avg}}}

\newcommand{\fedqavg}{{\sf{FedAsynQ}}\xspace}

\newcommand{\pfedq}{{\sf FedLCB-Q}\xspace}

\newcommand{\defn}{\coloneqq}

\newcommand{\one}{\mathbb{I}}

\newcommand{\mprob}{\mathbb{P}}
\newcommand{\mexp}{\mathbb{E}}

\newcommand{\cA}{\mathcal{A}}

\newcommand{\cD}{\mathcal{D}}

\newcommand{\cS}{{\mathcal{S}}}
\newcommand{\cT}{{\mathcal{T}}}

\newcommand{\mymid}{\,|\,} 

\usepackage{scalerel,stackengine}
\stackMath
\newcommand\reallywidehat[1]{%
\savestack{\tmpbox}{\stretchto{%
  \scaleto{%
    \scalerel*[\widthof{\ensuremath{#1}}]{\kern-.6pt\bigwedge\kern-.6pt}%
    {\rule[-\textheight/2]{1ex}{\textheight}}
  }{\textheight}%
}{0.5ex}}%
\stackon[1pt]{#1}{\tmpbox}%
}
\newcommand\reallywidecheck[1]{%
\savestack{\tmpbox}{\stretchto{%
  \scaleto{
    \scalerel*[\widthof{\ensuremath{#1}}]{\kern-.6pt\bigwedge\kern-.6pt}%
    {\rule[-\textheight/2]{1ex}{\textheight}}
  }{\textheight}%
}{0.5ex}}%
\stackon[1pt]{#1}{\scalebox{-1}{\tmpbox}}%
}

\title{Federated Offline Reinforcement Learning:\\ Collaborative Single-Policy Coverage Suffices}
 \author{
 	Jiin Woo\thanks{Department of Electrical and Computer Engineering, Carnegie Mellon University, Pittsburgh, PA 15213, USA.} \\ 
	CMU \\
	\and
	Laixi Shi\thanks{Department of Computing Mathematical Sciences, California Institute of Technology, CA 91125, USA.}\\
 	Caltech 
 	\and
 	Gauri Joshi\footnotemark[1] \\ 	 
	CMU
	\and
 	Yuejie Chi\footnotemark[1] \\ 	 
  	CMU
 	} 
\date{February 2024}

\begin{document}

\theoremstyle{plain} \newtheorem{lemma}{\textbf{Lemma}}
\newtheorem{proposition}{\textbf{Proposition}}
\newtheorem{theorem}{\textbf{Theorem}}
\newtheorem{corollary}{\textbf{Corollary}}
\newtheorem{assumption}{Assumption}
\newtheorem{definition}{Definition}
\newtheorem{claim}{\textbf{Claim}}
\theoremstyle{remark}\newtheorem{remark}{\textbf{Remark}}

\maketitle

\begin{abstract}
Offline reinforcement learning (RL), which seeks to learn an optimal policy using offline data, has garnered significant interest due to its potential in critical applications where online data collection is infeasible or expensive. This work explores the benefit of federated learning for offline RL, aiming at collaboratively leveraging offline datasets at multiple agents. Focusing on finite-horizon episodic tabular Markov decision processes (MDPs), we design \pfedq, a variant of the popular model-free Q-learning algorithm tailored for federated offline RL. \pfedq updates local Q-functions at agents with novel learning rate schedules and aggregates them at a central server using importance averaging and a carefully designed pessimistic penalty term.
Our sample complexity analysis reveals that, with appropriately chosen parameters and synchronization schedules, \pfedq achieves linear speedup in terms of the number of agents without requiring high-quality datasets at individual agents, as long as the local datasets collectively cover the state-action space visited by the optimal policy, highlighting the power of collaboration in the federated setting. 
In fact, the sample complexity almost matches that of the single-agent counterpart, as if all the data are stored at a central location, up to polynomial factors of the horizon length.
Furthermore, \pfedq is communication-efficient, where the number of communication rounds is only linear with respect to the horizon length up to logarithmic factors.
\end{abstract}

 \smallskip
 

\setcounter{tocdepth}{2}
\tableofcontents


\section{Introduction}
Offline RL \citep{levine2020offline}, also known as batch RL, addresses the challenge of learning a near-optimal policy using offline datasets collected a priori, without further interactions with an environment. Fueled by the cost-effectiveness of utilizing pre-collected datasets compared to real-time explorations, offline RL has received increasing attention. However, the performance of offline RL crucially depends on the quality of offline datasets due to the lack of additional interactions with the environment, where the quality is determined by how thoroughly the state-action space is explored during data collection.

Encouragingly, recent research \citep{rashidinejad2021bridging,shi2022pessimistic,xie2021policy,li2022settling} indicates that being more conservative on unseen state-action pairs, known as the principle of pessimism, enables learning of a near-optimal policy even with partial coverage of the state-action space, as long as the distribution of datasets encompasses the trajectory of the optimal policy. However, acquiring high-quality datasets that have good coverage of the optimal policy poses challenges because it requires the state-action visitation distribution induced by a behavior policy employed for data collection to be very close to the optimal policy. Alternatively, multiple datasets can be merged into one dataset to supplement insufficient coverage of one other, but this may be impractical when offline datasets are scattered and cannot be easily shared due to privacy and communication constraints.
 
 \paragraph{Federated offline RL.}
Driven by the need to harvest multiple datasets to address insufficient coverage, there is a growing interest in implementing offline RL in a federated manner without the need to share datasets \citep{zhou23fed,woo23blessing,khodadadian22fedrlspeedup}. For model-based RL, a study has proposed a federated variant of pessimistic value iteration \citep{zhou23fed}, which requires sharing of model estimates. On the other hand, for model-free RL, while \citet{woo23blessing} introduced a federated Q-learning algorithm that achieves linear speedup with collaborative coverage of agents, due to the absence of pessimism, it still carries the risk of overestimation on state-action pairs that are insufficiently covered by the agents. Indeed, it remains unknown whether the principle of pessimism can be implemented in federated offline RL to eliminate the risk of overestimation, while fully utilizing the collaborative coverage provided by agents, and without sharing datasets or model estimates.  

Our goal in this paper is to develop a federated variant of Q-learning \citep{watkins1992q} for offline RL, which allows agents to learn a near-optimal Q-function with improved sample efficiency and relaxed coverage assumption. In the single-agent case, pessimism is implemented by penalizing the value estimates by subtracting a penalty term measuring the uncertainty of the estimates \citep{yan2022efficacy,shi2022pessimistic}.
However, federated settings are communication-constrained, implying that agents only have a limited chance of synchronization and they perform multiple local updates without knowing other agents' training progress. Allowing multiple local updates leads to higher uncertainty of local Q-estimates beyond the control of the pessimism penalty, potentially impacting both sample complexity and communication efficiency. This underscores the technical challenge of incorporating pessimism while managing local updates and raises the question: 
\begin{center}
\emph{How to judiciously incorporate the principle of pessimism in federated RL without hurting its sample and communication efficiency?}
\end{center}

\subsection{Our contribution}

This work presents a federated Q-learning algorithm with pessimism for offline RL, which achieves {\em linear speedup} and {\em low communication cost}, while requiring only {\em collaborative coverage of the optimal policy}. Formally, we consider episodic finite-horizon tabular Markov decision processes (MDPs) with $S$ states, $A$ actions, and horizon length $H$. A total number of $M$ agents, each with $K$ trajectories (collected using its local behavior policy), collaborate in a federated setting with the help of a central server to learn the optimal policy. Our main contributions are summarized as below; see also Table~\ref{table:sample-comparison} for a detailed comparison.

\begin{itemize}
\item \textbf{Federated Q-learning for offline RL.}
We propose a federated offline Q-learning algorithm named \pfedq, which involves iterative local updates at agents and global aggregation at a central server with scheduled synchronizations. We introduce essential components that implement pessimism compensating for the uncertainty in both local and global Q-function updates. First, to address the uncertainty arising from independent local updates, we employ {\em learning rate rescaling} at local agents and {\em importance averaging} at server aggregation. The former restricts the drifts of local Q-estimates by rapidly decreasing the learning rates during local updates, and the latter reduces uncertainty of the aggregated Q-estimates by assigning smaller weights to rarely updated local values. Additionally, for every global aggregation, a {\em global penalty} calculated based on aggregated visitation counts is subtracted from the aggregated global Q-estimate.
These design choices play a crucial role in achieving both sample and communication efficiency while preventing the overestimation of the Q-function.

\item  \textbf{Linear speedup with collaborative single-policy coverage.} 
Our analysis of sample complexity of \pfedq (see Theorem~\ref{thm:pess-fedq-regret}) demonstrates that \pfedq finds an $\varepsilon$-optimal policy, as long as the total number of samples per agent $T = KH$ exceeds 
$$ \widetilde{O} \left (\frac{ H^7 S \clipavgC}{M \varepsilon^2}  \right), $$
where $\clipavgC$ denotes the average single-policy concentrability coefficient of all agents (see \eqref{eq:clipped-avg-concentrability-assumption} for the formal definition). This shows linear speedup in terms of number agents $M$, which is achieved with a significantly weaker data requirement at individual agents than prior art. In truth, each agent affords to have a non-expert dataset collected by a sub-optimal behavior policy, as long as all agents collectively cover the state-action pairs visited by the optimal policy, even they don't cover the entire state-action space as in \citet{woo23blessing}. The bound nearly matches the sample complexity obtained for a single-agent pessimistic Q-learning algorithm \citep{shi2022pessimistic} with a similar Hoeffding-style penalty, up to a factor of $H$, as if all the datasets are processed at a central location.

\item \textbf{Low communication cost.}
  Under appropriate choices of synchronization schedules, \pfedq requires approximately $\widetilde{O}(H)$ rounds of synchronizations to achieve the targeted accuracy (see Corollary~\ref{cor:fedq-comm}), which is almost independent with the size of the state-action space and the number of agents. The analysis suggests that frequent synchronizations are not necessary, outperforming prior art \citep{woo23blessing}.

\end{itemize}
 
 \begin{table*}[t]
\centering
\resizebox{\textwidth}{!}{
\begin{tabular}{c|c|c|c|c|c }
\toprule 
  \multirow{2}{*}{ type   }&   \multirow{2}{*}{  reference} & number of   &   \multirow{2}{*}{  coverage } &  sample & communication        \tabularnewline  
 &    &   agents &    & complexity & rounds     \tabularnewline \toprule
 \multirow{4}{*}{model-based}   &  {{\sf VI-LCB} \citep{xie2021policy}}   \vphantom{$\frac{1^{7^{7^7}}}{7^{7^{7^7}}}$} & $1$ &  single & $\frac{H^6 S C^{\star}}{\varepsilon^2} $  \vphantom{$\frac{7^{7^{7}}}{7^{7^{7^{7^{7^7}}}}}$}  & -   \tabularnewline 
\cline{2-6}   &   {{\sf PEVI-Adv} \citep{xie2021policy}}  \vphantom{$\frac{1^{7^{7^7}}}{7^{7^{7^7}}}$} & $1$  & single & $\frac{H^4 S C^{\star}}{\varepsilon^2} $  \vphantom{$\frac{7^{7^{7}}}{7^{7^{7^{7^{7^7}}}}}$}    & - \tabularnewline
\cline{2-6}   &   {{\sf VI-LCB} \citep{li2022settling}}  \vphantom{$\frac{1^{7^{7^7}}}{7^{7^{7^7}}}$} & $1$  & single & $\frac{H^4 S C^{\star}}{\varepsilon^2} $  \vphantom{$\frac{7^{7^{7}}}{7^{7^{7^{7^{7^7}}}}}$}   & -  \tabularnewline    \toprule 
\multirow{6}{*}{model-free}   &   {{\sf LCB-Q} \citep{shi2022pessimistic}}  & $1$  &  single & $\frac{H^6 S C^{\star}}{\varepsilon^2} $  \vphantom{$\frac{7^{7^{7}}}{7^{7^{7^{7^{7^7}}}}}$}  & - \tabularnewline
   \cline{2-6}  
     &   {{\sf LCB-Q-Adv} \citep{shi2022pessimistic}} \vphantom{$\frac{1^{7^{7^7}}}{7^{7^{7^7}}}$}  & $1$ &  single & $ \frac{H^4 S C^{\star}}{\varepsilon^2} $   \vphantom{$\frac{7^{7^{7}}}{7^{7^{7^{7^{7^7}}}}}$}   & - \tabularnewline
   \cline{2-6}  
     &   {\fedqavg \citep{woo23blessing}} \vphantom{$\frac{1^{7^{7^7}}}{7^{7^{7^7}}}$}  & $M$ &   collaborative & $ \frac{H^6 }{M d_{\mathsf{avg}}\varepsilon^2} $   \vphantom{$\frac{7^{7^{7}}}{7^{7^{7^{7^{7^7}}}}}$} & $\frac{HM}{d_{\mathsf{avg}}}$   \tabularnewline
     \cline{2-6}  
&   {\pfedq (Theorem~\ref{thm:pess-fedq-regret})} \vphantom{$\frac{1^{7^{7^7}}}{7^{7^{7^7}}}$}  & $M$ &  collaborative & $ \frac{H^7 S \clipavgC}{M \varepsilon^2} $ \vphantom{$\frac{7^{7^{7}}}{7^{7^{7^{7^{7^7}}}}}$}   & $H$  \tabularnewline \toprule 
\end{tabular} }
\caption{Comparison of sample complexity upper bounds of model-based and model-free algorithms for offline RL to learn an $\varepsilon$-optimal policy in finite-horizon non-stationary MDPs, where logarithmic factors and burn-in costs are hidden.  
  Here, $S$ is the size of state space, $A$ is the size of action space, $H$ is the horizon length, $M$ is the number of agents, $C^{\star}$ and $\clipavgC$ denote the single-policy concentrability and the average single-policy concentrability, respectively (cf. \eqref{equ:concentrability-assumption} and \eqref{eq:avg_occupancy_dist}), and $d_{\mathsf{avg}}$ is the minimum entry of the average stationary state-action occupancy distribution of all agents. We follow standard conversion to translate the best sample complexity in \citet{woo23blessing} to the finite-horizon setting for comparison.
} 
\label{table:sample-comparison}
\end{table*}

\subsection{Related work}
\paragraph{Offline RL.}
 Offline RL addresses the problem of learning improved policies from a logged static dataset. The main challenge of offline RL is how to reliably estimate the values of unseen or rarely visited state-action pairs. To tackle this challenge, most offline RL algorithms prevent agents from taking uncertain actions by regularizing the policy to be close to the behavior policy \citep{fujimoto2019off,siegel2020keep, fujimoto2021minimalist} or penalizing value estimates on out-of-distribution state-action pairs \citep{kumar2020conservative,liu2020provably,kostrikovoffline,wu2019behavior}, which is also known as the principle of pessimism.
Recently, the pessimistic approach has been developed and theoretically studied for various RL settings, such as model-based approaches \citep{xie2021policy,rashidinejad2021bridging,kidambi2020morel,yu2020mopo,jin2021pessimism,li2022settling,yin2021towards,kim23model}, policy-based approaches \citep{xie2021bellman,zanette2021provable}, and model-free approaches \citep{shi2022pessimistic,yan2022efficacy,uehara2023offline}.
Most of these works have focused on the single-agent case and suggested that the state-action visitation distribution  induced by the behavior policy should cover that of the optimal policy \citep{rashidinejad2021bridging,shi2022pessimistic,yan2022efficacy}, and the distribution mismatch among the two visitation distributions governs the hardness of offline RL \citep{li2022settling}. Another interesting work \citep{shi23perturb} considered offline RL from multiple perturbed data sources, requiring a centralized setting in which an agent has full access to all the datasets.

\paragraph{Federated RL.}
There has been an increasing interest in federated and distributed RL, driven by the need to address more realistic constraints, including privacy, communication efficiency, and data heterogeneity, as well as training speedup. 
Recent works have investigated federated RL from various perspectives, such as robustness to adversarial attacks \citep{wu2021byzantine,fan2021fedrlfaulttolerant}, environment or task heterogeneity \citep{yang23fednpg,jin2022fedrlhete,wang23fedtd,zhou23fed}, as well as sample and communication complexities under asynschronous sampling \citep{khodadadian22fedrlspeedup,woo23blessing} and online sampling \citep{zheng2023federated,zhang2024finite}.
For model-based RL, \citet{zhou23fed} studied a pessimistic variant of value iteration with multi-task offline datasets under the federated setting and showed the improved sample efficiency by sharing representations of common task structures.
However, for model-free RL, although \citet{woo23blessing} provided a federated Q-learning algorithm that achieves linear speedup in terms of the number of agents with relaxed coverage assumption for individual agents, it still requires agents to cover the entire state-action space uniformly due to the lack of pessimism.

\paragraph{Q-learning.} Characterizing the finite-sample complexity of single-agent Q-learning has been examined extensively under various data collection and function approximation schemes, including but not limited the synchronous setting \citep{even2003learning,beck2012error,li2021syncq,wainwright2019stochastic}, the asynchronous and offline setting \citep{li2021asyncq,li2021syncq,qu2020finite,yan2022efficacy,shi2022pessimistic}, the online setting \citep{jin2018q,bai2019provably,wang2019q}, under function approximation \citep{fan2019theoretical,chen2019performance,xu2020finite}, to mention just a few.

\paragraph{Notation.}  In this paper, we use $\Delta(\cS)$ to refer to the probability simplex over a set $\cS$, and $[K]\coloneqq \{1,\cdots,K\}$ for any positive integer $K>0$. In addition, $f(\cdot)=\widetilde{O} ( g(\cdot) )$ or $f \lesssim g $ (resp.~$f(\cdot)=\widetilde{\Omega} ( g(\cdot) )$ or $f\gtrsim g$) 
indicates that $f(\cdot)$ is order-wise not larger than (resp.~not smaller than) $g(\cdot)$ up to some logarithmic factors. The notation $f \asymp g $ signifies that both  $f \lesssim g$ and $f \gtrsim g$ simultaneously hold. 


\section{Background and problem formulation}
\label{sec:background}

\subsection{Background}

\paragraph{Basics of episodic finite-horizon MDPs.} Consider an episodic finite-horizon MDP represented by 
$$\mathcal{M}= \big(\mathcal{S},\mathcal{A},H, \{P_h\}_{h=1}^H, \{r_h\}_{h=1}^H \big),$$ 
where $\mathcal{S}$ is the state space of size $S$,  $\mathcal{A}$ is the action space of size $A$, $H$ is the horizon length, $P_h : \cS \times \cA \rightarrow \Delta (\cS) $ and $r_h: \cS \times \cA \rightarrow [0,1]$ denote the probability transition kernel and the reward function at the $h$-th time step $(1\leq h\leq H)$, respectively.
 
A policy is denoted by $\pi =\{\pi_h\}_{h=1}^H$, where $\pi_h: \mathcal{S} \rightarrow \Delta(\mathcal{A})$ specifies the probability distribution over the action space at time step $h$ in state $s$. With slight abuse of notation, we use $\pi_h(s)$ to denote the selected action when the policy $\pi_h$ is deterministic. For $h=1,\ldots, H$, the value function $V_h^{\pi}(s)$ of policy $\pi$ is defined as the expected cumulative rewards starting from state $s$ at step $h$ by following $\pi$, i.e., 
\begin{align}
	\label{eq:def_Vh}
	V^{\pi}_{h}(s) &\defn  \mathbb{E} 
	\left[  \sum_{t=h}^{H} r_{t}\big(s_{t},a_t \big) \,\Big|\, s_{h}=s \right], 
\end{align}
where the expectation is taken over the randomness of the trajectory $\{s_t, a_t, r_t\}_{t=h}^H $ induced by the policy $\pi$ as well as the MDP transitions according to $a_t \sim \pi_t(\cdot \mymid s_t)$ and $s_{t+1} \sim P_t(\cdot \mymid s_t, a_t)$. Similarly, the Q-function $Q^{\pi}_h(s ,a)$ of a policy $\pi$ at step $h$ in state-action pair $(s,a)$ is defined as
\begin{align} 
	\label{eq:def_Qh}
	Q^{\pi}_{h}(s,a) & \defn r_{h}(s,a)+ \mathbb{E} \left[  \sum_{t=h +1}^{H} r_t (s_t, a_t ) \,\Big|\, s_{h}=s, a_h  = a\right] ,
\end{align}
where the expectation is again over the randomness induced by $\pi$ and the MDP transitions. 

It is well-known \citep{puterman2014markov} that one can always find a deterministic {\em optimal} policy $\pi^{\star} =\{\pi_h^{\star}\}_{h=1}^H$, which maximizes the value function (resp.~the Q-function) {\em simultaneously} over all states (resp.~state-action pairs) among all policies. 
The resulting optimal value function $V^{\star} =\{ V_h^{\star} \}_{h=1}^H $ and optimal Q-functions $Q^{\star} =\{ Q_h^{\star} \}_{h=1}^H $ are denoted respectively by  
\begin{align*}
	V_h^{\star}(s) & \coloneqq V_h^{\pi^{\star}}(s) = \max_{\pi} V_h^{\pi}(s), \qquad
	Q_h^{\star}(s,a)  \coloneqq Q_h^{\pi^{\star}}(s,a)  = \max_{\pi} Q_h^{\pi}(s,a) 
\end{align*}
for any $(s,a,h)\in \cS\times \cA \times [H]$. Given an initial state distribution $\rho\in\Delta(\cS)$, the expected value of a given policy $\pi$ and that of the optimal policy $\pi^{\star}$ at the initial step are defined respectively by
\begin{align}\label{eq:defn-V-rho}
	V_1^{\pi}(\rho) \coloneqq \mathbb{E}_{s_1\sim \rho} \big[ V_1^\pi (s_1) \big] 
	 \qquad \text{and}  \qquad
	V_1^{\star}(\rho) \coloneqq \mathbb{E}_{s_1\sim \rho} \big[ V_1^\star (s_1) \big].
\end{align}

\paragraph{Bellman equations.} Of crucial importance are the Bellman equations that connect the value functions across different time steps \citep{bertsekas2017dynamic}. For any policy $\pi$, it follows that
\begin{align}
	\label{eq:bellman}
 Q^{\pi}_{h}(s,a)=r_{h}(s,a)+ \mathbb{E}_{s'\sim P_{h,s,a}} \big[V^{\pi}_{h+1}(s') \big]
\end{align}
for all $(s,a,h)\in \cS\times \cA\times [H]$, where $V^{\pi}_{H+1}(s) =0$ for any $s\in\cS$. Moreover, Bellman's optimality equation says that
\begin{align} \label{eq:bellman_optimality}
	Q^{\star}_{h}(s,a)=r_{h}(s,a)+ \mathbb{E}_{s'\sim P_{h,s,a}} \big[V^{\star}_{h+1}(s') \big]
\end{align}
for all $ (s,a,h)\in \cS\times \cA\times [H]$, and the optimal policy satisfies $\pi_h^{\star}(s) = \argmax_{a\in\cA} Q^{\star}_{h}(s,a)$.

\subsection{Problem formulation: federated offline RL}
	\label{sec:offline-concentrability}

In offline RL, one has access to a offline dataset containing episodes collected by following some behavior policy. Here, we formulate a federated version of the offline RL problem with $M$ agents, where each agent has access to a local offline dataset. For $1\leq m\leq M$, the offline dataset $\Datab^{m}$ at agent $m$ is composed of $K$ episodes,\footnote{For simplicity, we assume all the agents have the same number of episodes. It is straightforward to generalize to the scenario when the local offline datasets have different sizes.} each generated independently according to a behavior policy  $\pib^m = \{ \pib_{h}^m \}_{h=1}^H$, resulting in
 \[
	\Datab^m \coloneqq  \Big\{  \big( s_{k,1}^m,\, a_{k,1}^m, \, r_{k,1}^m,\, \ldots, s_{k,H}^m,\, a_{k,H}^m, \, r_{k,H}^m  \big) \Big\}_{k=1}^{K} ,
\]
where the initial state $s_{k,1}^m \sim \rho $ is drawn from some initial state distribution $\rho \in \Delta(\cS)$, $s_{k,h}^m\, a_{k,h}^m, \, r_{k,h}^m$ are the state, action and reward at step $h$ in the $k$-th episode, $a_{k,h}^m \sim  \pib_{h}^m(\cdot \mymid s_{k,h}^m)$ and $r_{k,h}^m = r_h(s_{k,h}^m\, a_{k,h}^m)$.

\paragraph{Goal.} The goal of federated offline RL is to learn an $\varepsilon$-optimal policy $\widehat{\pi} =\{ \widehat{\pi}_h \}_{h=1}^H$ satisfying
\[   
	V_1^{\star}(\rho) - V_1^{\widehat{\pi}}(\rho) 
	\leq \varepsilon 
\]
using the history dataset $\Datab = \big\{ \Datab^m \big\}_{1\leq m\leq M}$ without sharing the local offline datasets, with the help of a parameter server. Furthermore, it is greatly desirable to achieve as high accuracy as possible, in a memory- and communication-efficient manner.

\paragraph{Metric.} Obviously, the success of offline RL highly relies on the quality of the history dataset. In order to define the metric, let us first introduce the occupancy distributions $d_h^{\pi}(s)$ and $d_h^{\pi}(s,a)$ induced by policy $\pi$ at step $h$, given by
\begin{align} \label{eq:visitation_dist}
	d_h^{\pi}(s) & \defn \mathbb{P}(s_h = s \mymid s_1 \sim \rho, \pi),  \qquad
	d_h^{\pi}(s, a)  \defn \mathbb{P}(s_h = s \mymid s_1 \sim \rho, \pi) \, \pi_h(a \mymid s).
\end{align}
Recent works \citep{rashidinejad2021bridging,xie2021policy,shi2022pessimistic} have advocated the notion of {\em single-policy concentrability}, which measures the mismatch between the occupancy distributions induced by the optimal policy $\pi^{\star}$ and the behavior policy $\mu$, with the benefit that this assumes away the need for the offline dataset to cover the entire state-action space, which is often impractical. \citet{li2022settling} offered a more refined notion called {\em single-policy clipped concentrability}, defined as follows.

\begin{definition}[single-policy clipped concentrability]
\label{assumption}
The single-policy clipped concentrability coefficient $C^\star \in [1/S, \infty)$ of a behavior policy $\mu$ is defined to be the smallest quantity that satisfies 
\begin{equation}\label{equ:concentrability-assumption}
    \max_{(h, s, a) \in [H] \times \mathcal{S} \times \cA} \frac{ \min \{ d^{\pi^\star}_h(s,a), \, 1/S\}}{d^{\mu}_h(s,a)} \leq C^\star,
\end{equation}
where we adopt the convention $0/0=0$. 
\end{definition}
The single-policy clipped concentrability coefficient $C^{\star} <\infty$ is finite whenever the behavior policy covers the state-action pairs visited by the {\em optimal} policy, rather than having to cover the entire state-action space. Recall that since $\pi^{\star}$ is deterministic, $d^{\pi^\star}_h(s,a) = d_h^{\pi^{\star}}(s) \mathbb{I}(a = \pi_h^{\star}(s))$, that is, $d^{\pi^\star}_h(s,a)$ is non-zero only for the optimal action $a = \pi_h^{\star}(s)$. Compared with the unclipped counterpart introduced in \citet{rashidinejad2021bridging}, the clipping of the occupancy distribution $d^{\pi^\star}_h(s,a)$ by the threshold $1/S$ ensures that $C^{\star}$
will not be excessively large when $d_h^{\pi^{\star}}(s) $ is highly concentrated in a small number of states in state space.

In the federated setting, we further introduce a tailored notion that highlights the potential benefit of collaborative learning in the presence of multiple agents. For ease of notation, denote 
$$ d_h^m(s) = d_h^{\mu^m}(s)    \quad \mbox{and} \quad  d_h^m(s,a) = d_h^{\mu^m}(s,a) $$
as the occupancy distributions induced by the behavior policy $\mu^m$ at agent $m$. Based on these, we define the average occupancy distributions as
\begin{equation}\label{eq:avg_occupancy_dist}
d_h^{\mathsf{avg}}(s) =  \frac{1}{M} \sum_{m=1}^M d_h^{m}(s)    \quad \mbox{and} \quad   {d}^{\mathsf{avg}}_h(s,a) = \frac{1}{M} \sum_{m=1}^M d_h^{m}(s,a).
\end{equation}

\begin{definition}[average single-policy clipped concentrability]
The average single-policy concentrability coefficient $\clipavgC \in [1/S, \infty)$ of multiple behavior policies $\{\mu^m\}_{m \in [M]}$ is defined to be the smallest quantity that satisfies 
\begin{equation}\label{eq:clipped-avg-concentrability-assumption}
    \max_{(h,s,a) \in [H] \times \cS \times \cA} \frac{\min\{d_h^{\pi^{\star}}(s,a), 1/S\}}{ {d}^{\mathsf{avg}}_h(s,a)}\le \clipavgC,
\end{equation}
where we adopt the convention $0/0=0$. 
\end{definition}

An important implication of the above definition is that, as long as the agents {\em collaboratively} cover the state-action pairs visited by the optimal policy, the average single-policy concentrability coefficient $\clipavgC<\infty$ is finite. Therefore, this is much weaker than the coverage requirement in the single-agent case.
  

\section{Proposed algorithm and theoretical guarantees}

In this section, we first introduce the proposed model-free federated offline RL algorithm called \pfedq, followed by its theoretical performance guarantees.

\begin{algorithm}[t]
  \begin{algorithmic}[1] 
    \STATE \textbf{Parameters:} horizon length $H$, number of agents $M$, total number of episodes per agent $K$, synchronization schedule $\synset(K)$, target error $\delta \in (0,1)$, $\logpartm = \log\left(\frac{ SA K^2 M H}{\delta} \right)$, $c_B >0$. 
    \STATE  \textbf{Initialization:} set $Q_{0,h}^m(s,a) = 0$, $V_{0,h}^m(s) = 0$, $N_{0,h}^m(s,a) = 0$, $n_{0,h}^m(s,a) = 0$, $N_{0,h}(s,a) = 0$, $n_{0,h}(s,a) = 0$ for all $(m, s,a,h) \in [M] \times \cS \times \cA \times [H+1]$. 
    
  \For{$k=1,\cdots,K$ }{ 
 
    \tcc{Update the local Q-estimate and visitation counts at each agent}
    $(Q_{k,h}^m, n_{k,h}^m)$ = \texttt{Local-Q-learning()};
    
\If{$k \in \synset(K)$ }{

 \tcc{Agent-to-server communication}
 Agents communicate $Q_{k,h}^m$ and $n_{k,h}^m$ to the server;

    \tcc{Global pessimistic averaging in a server}
    $(Q_{k,h}, V_{k,h}, \pi_{k,h})$ = \texttt{Global-pessimistic-averaging();}\\
    
     \tcc{Server-to-agent communication}
    Server communication $Q_{k, h}$, $V_{k,h}$ and $N_{k,h}$ to agents;
    
        \tcc{Synchronize local Q-estimates}
    \For{$(m, s,a,h) \in [M] \times \cS \times \cA \times [H]$ }{
        $Q_{k, h}^m (s, a) = Q_{k, h} (s, a)$, $V_{k, h}^m(s) = V_{k, h}(s)$
    }
}
}

    \textbf{return:} $\hat{Q} = \{Q_{K,h}\}_{h \in [H]}$ and $\hat{\pi}= \{\pi_{K,h}\}_{h \in [H]}$.
  \end{algorithmic} 
  \caption{Federated pessimistic Q-learning (\pfedq) }
  \label{alg:pfedq}
\end{algorithm}

\begin{figure}[t]
\centering
\includegraphics[trim={0.1cm 0 0 0.1cm},clip, width=0.6\textwidth]{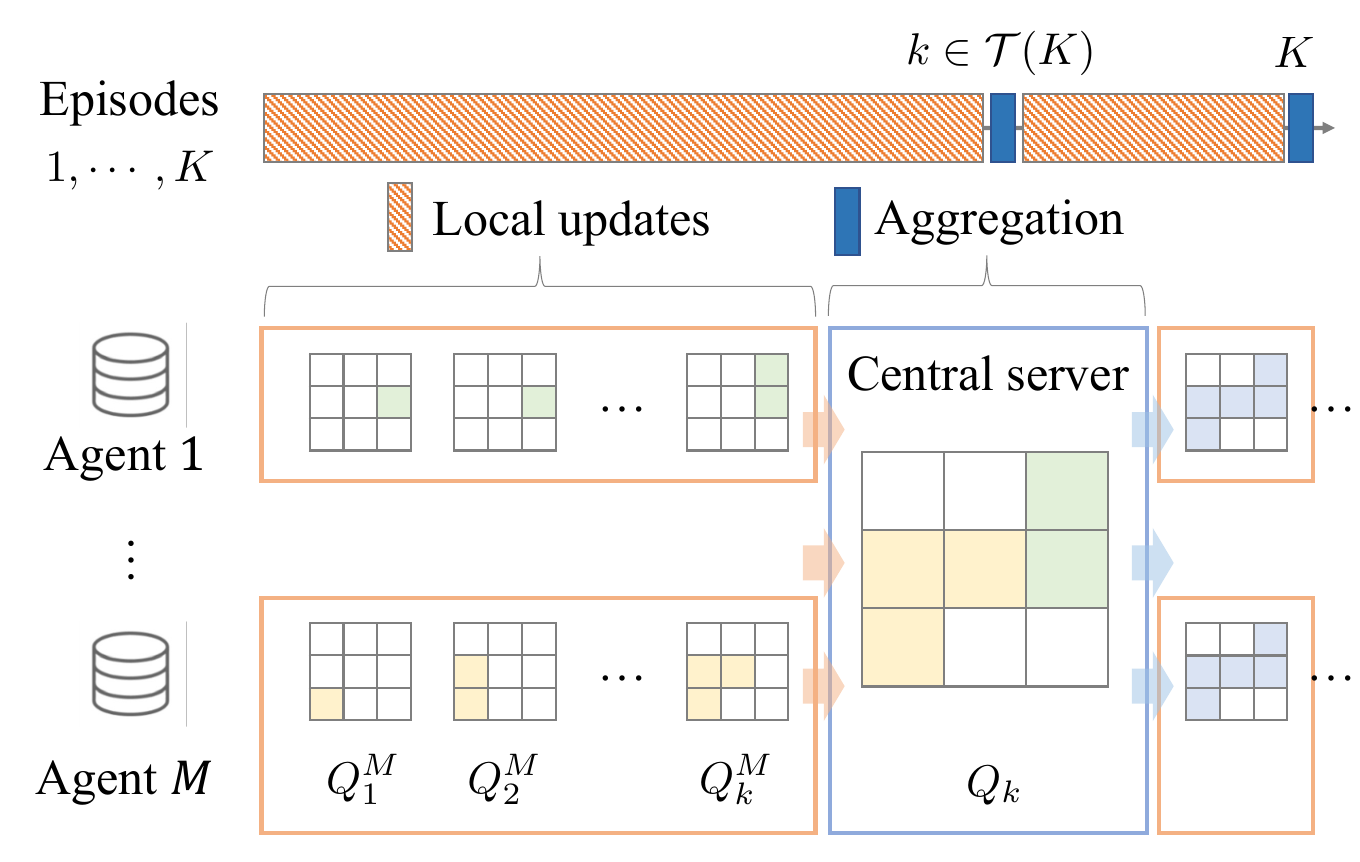}
\caption{\pfedq with $M$ agents and a central server. Each agent $m$ performs local updates on its local Q-table $Q_k^m$ for each $k$th episode in a local history dataset $\cD^m$. When synchronization is scheduled at $k \in \synset(K)$, the agents send their local Q-tables to the server and the server aggregates the Q-tables into a global Q-table and synchronizes local Q-tables.}
\label{fig:fedlcbq-alg}
\end{figure}

\subsection{Algorithm description}

We introduce a federated variant of Q-learning algorithm for offline RL, called \pfedq, that learns a near-optimal Q-function without overestimation on  unseen components of the state-action space. The complete description of \pfedq is provided in Algorithm~\ref{alg:pfedq}, with its agent-end and server-end subroutines described in Algorithm~\ref{alg:pfedq-agents} and Algorithm~\ref{alg:pfedq-server} respectively. On a high level, \pfedq performs local Q-function updates at all the agents using its own local offline dataset, and occasionally, globally aggregates the local estimates in a pessimistic fashion at a central server. To facilitate flexible communication patterns, we follow a synchronization schedule $\synset(K)$, which contains the indices of episodes where communication occurs between the agents and the server.

To begin, \pfedq initializes the local estimate  ($Q_{0, h}^m$ and $V_{0,h}^m$)  at each agent $m \in [M]$ and the global estimates ($Q_{0, h}$ and $V_{0,h}$) at the server as follows:
\begin{subequations}
\begin{align}
    Q_{0, h}^m(s, a) = 0, \qquad
    V_{0,h}^m(s,a) = 0, \qquad \mbox{for all } (s,a,h) \in \cS \times \cA \times [H+1],   \\
    Q_{0, h}(s, a) = 0, \qquad
    V_{0,h}(s,a) = 0, \qquad\mbox{for all } (s,a,h) \in \cS \times \cA \times [H+1]. 
\end{align}
\end{subequations}
Then, \pfedq proceeds the following steps for each episode $k \in [K]$. 

\begin{enumerate}
    \item \textbf{Local updates:}
    Each agent $m $ samples the $k$th trajectory $\{(s_{k,h}^m,a_{k,h}^m,r_{k,h}^m)\}_{h=1}^H$ from its local offline datasets $\cD^m$. For each step $h \in [H]$, agent $m$ updates its local Q-estimate $Q_{k, h}^m$ as follows:
    \begin{align} \label{eq:localq-update}
         Q_{k, h}^m(s, a) =
         \begin{cases}
             (1- \eta_{k,h}^m(s,a)) Q_{k-1, h}^m(s,a) + \eta_{k,h}^m(s,a) (r_{k,h}^m + V_{k-1,h+1}^m (s_{k,h+1}^m)) &~\text{if}~(s,a) = (s_{k,h}^m,a_{k,h}^m) \\
             Q_{k-1, h}^m(s,a) &~\text{otherwise}
         \end{cases},
    \end{align}
    where $\eta_{k,h}^m(s,a)$ is the learning rate, whose schedule will be specified later, and $V_{k-1,h}^m (s)$ is set as
    \begin{align} \label{eq:vlocal-frozen}
        V_{k-1,h}^m (s) = V_{\syn(k), h}^m(s) = V_{\syn(k), h}(s),  \quad \mbox{for all } (m,s,h,k) \in [M]\times\cS\times [H+1] \times [K],  
    \end{align}
    where $\syn(k)$ denotes the most recent episode where aggregation occurs before the $k$th episode, i.e., 
    $$\syn(k) \defn \max_{k'}\left\{1\leq k' < k: k' \in \synset(K)\right \} . $$

    \item \textbf{Pessimistic aggregation:}   If synchronization is scheduled at episode $k$, i.e., $k \in \synset(K)$, each agent sends its local Q-estimate to a central server for aggregation after finishing the local update for the $k$th episode. Then, the server updates the global Q-estimate $Q_{k, h}$  by averaging the local Q-estimates and subtracting a penalty as follows:
\begin{align} \label{eq:globalq-update}
  \forall (s,a) \in \cS\times\cA: \qquad Q_{k, h} (s, a) = \left (\sum_{m=1}^M \alpha_{k,h}^m(s,a) Q_{k, h}^m(s, a) \right) - B_{k,h}(s,a) ,
\end{align}
where $\alpha_{k,h}^m = [\alpha_{k,h}^m(s,a)]_{(s,a)\in \cS\times\cA} \in [0,1]^{SA}$ is an entry-wise weight matrix assigned to agent $m$ for each $h \in [H]$, 
and $B_{k,h}(s,a)$ is a penalty term (to be specified later below) that introduces the pessimism preventing the overestimation of unseen state-action pairs. Accordingly, the global value estimate is updated as
\begin{align} \label{eq:vglobal-monotone}
  \forall (s,a) \in \cS\times\cA: \qquad V_{k, h}(s) = \max \left \{ V_{\syn(k), h}(s), \, \max_{a \in \cA} Q_{k, h}(s, a) \right \}.
\end{align}
where the outer maximum ensures a monotonic update, as we explain later in the analysis. If $V_{k, h}(s) = \max_{a \in \cA} Q_{k, h}(s, a)$, the global policy is updated as $\pi_{k, h}(s) = \argmax_{a \in \cA} Q_{k, h}(s, a)$, otherwise $\pi_{k, h}(s) = \pi_{\syn(k), h}(s)$.
After aggregation, the server sends the global Q-function and value estimates to every agent, where
\begin{equation} \label{eq:sync}
\forall ( k ,m)  \in  \synset(K) \times [M]: \qquad Q_{k, h}^m =Q_{k, h}, \quad V_{k, h}^m =V_{k, h} .
\end{equation}  

\end{enumerate}

\begin{algorithm}[t]
  \begin{algorithmic}[1] 
 \STATE   \For{$m= 1, \cdots, M$ }{  
     Sample the $k$-th trajectory $\{(s_{k,h}^m,a_{k,h}^m,r_{k,h}^m,s_{k,h+1}^m)\}_{h=1}^H$ from $\cD^m$
 
    \For{$h= 1, \cdots, H$}{
    \For{$(s,a) \in \cS \times \cA$}{
      $Q_{k,h}^m(s,a) = Q_{k-1,h}^m(s,a)$, $V_{k,h}^m(s) = V_{k-1,h}^m(s)$
    }
    
    \tcp{Update the local counters and learning rates}
    {$n_{k,h}^m(s_{k,h}^m,a_{k,h}^m) = n_{k-1,h}^m(s_{k,h}^m,a_{k,h}^m)+1$ }
    
    {$\eta_{k,h}^m (s_{k,h}^m,a_{k,h}^m)= \frac{M (H+1)}{ N_{\syn(k),h}(s_{k,h}^m,a_{k,h}^m) + M (H+1) n_{k,h}^m(s_{k,h}^m,a_{k,h}^m)}$}
    
    \tcp{Update local Q-estimates}
    $Q_{k, h}^m(s_{k,h}^m,a_{k,h}^m) = \big(1- \eta_{k,h}^m (s_{k,h}^m,a_{k,h}^m) \big) Q_{k-1, h}^m(s_{k,h}^m,a_{k,h}^m) + \eta_{k,h}^m(s,a) (r_{k,h}^m + V_{k-1,h+1}^m (s_{k,h+1}^m))$
    }
    }
  \end{algorithmic} 
  \caption{\texttt{Local-Q-learning} (agents)}
  \label{alg:pfedq-agents}
\end{algorithm}

\begin{algorithm}[t]
  \begin{algorithmic}[1] 
\STATE 
    \For{$(s,a,h) \in \cS \times \cA \times [H]$ }{    
    \tcp{Update the average counter}
    $n_{k, h}(s,a) =  \sum_{m=1}^M n_{k,h}^m(s,a)$, $N_{k, h}(s,a) = N_{\syn(k),h}(s,a) + n_{k, h}(s,a)$
    
    \tcp{Compute global penalty and averaging weights}
    $B_{k,h}(s,a) = \frac{ (H+1) n_{k, h}(s,a) }{N_{k, h}(s,a)+Hn_{k, h}(s,a)}  \sqrt{\frac{c_B \logpartm^2 H^4}{ N_{k, h}(s,a)}}$ if $N_{k, h}(s,a)>0$, otherwise, $B_{k,h}(s,a)=0$
    
    \For{$m=1 \cdots M$}{
    $\alpha_{k,h}^m(s,a) = \frac{1}{M} \frac{N_{\syn(k), h}(s,a) + M(H+1) n_{k,h}^m(s,a)}{N_{k, h}(s,a) + H n_{k, h}(s,a)}$ if $n_{k,h}^m(s,a)>0$, otherwise, $\alpha_{k,h}^m(s,a) = \frac{1}{M}$
    }

    \tcp{Update global Q-estimates}
    $Q_{k, h} (s, a) = \sum_{m=1}^M \alpha_{k,h}^m(s,a) Q_{k, h}^m(s, a)  - B_{k,h}(s,a)  $ \\
    $V_{k, h}(s) = \max \left \{ V_{\syn(k), h}(s), \max_{a \in \cA} Q_{k, h}(s, a) \right \}$ \\ \label{algo:monotone-of-V}
    $\pi_{k,h}(s) = \argmax_{a \in \cA} Q_{k, h} (s,a)$ if $V_{k, h}(s) = \max_{a \in \cA} Q_{k, h}(s, a)$, otherwise, $\pi_{k,h}(s) = \pi_{\syn(k),h}(s)$ \label{algo:line-pi-k}
    }

  \end{algorithmic} 
  \caption{\texttt{Global-pessimistic-averaging} (server)}
  \label{alg:pfedq-server}
\end{algorithm}

At the end of $K$ episodes, \pfedq outputs a global Q-estimate $\hat{Q}_h (s,a) = Q_{K,h}(s,a)$ for all $(s,a,h) \in \cS\times\cA\times [H]$ and a solution policy $\hat{\pi}_h(s) = \pi_{K,h}(s)$ for all $(s,h) \in \cS\times[H]$. For simplicity, we assume that the aggregation step always occurs after the last episode $K$, i.e., $K \in \synset(K)$.

\subsection{Choices of key parameters}\label{sec:key-parameters}

The success of \pfedq relies on careful and judicious selections of key algorithmic parameters, in a data-driven manner, which we detail below. To begin, let us introduce the following useful notation, which pertains to the counters for visits of agents on each state-action pair $(s,a) \in \cS \times \cA$. For any $(m, k, h) \in [M] \times [K] \times [H]$,
\begin{itemize}
    \item $n_{k,h}^m(s,a)$: the number of episodes in the interval $(\syn(k), k]$ during which agent $m$ visits $(s,a)$ at step $h$, i.e., $n_{k,h}^m(s,a) \defn |\{\syn(k) < i \le k: (s_{i,h}^m,a_{i,h}^m)=(s,a)\}|$.
    \item $N_{k,h}^m(s,a)$: the number of episodes in the interval $[1, k]$ during which agent $m$ visits $(s,a)$ at step $h$, i.e., $N_{k,h}^m(s,a) \defn |\{1 \le i \le k: (s_{i,h}^m,a_{i,h}^m)=(s,a)\}|$.
    \item $n_{k,h}(s,a)$: the total number of episodes in the interval $(\syn(k), k]$ during which all agents visit $(s,a)$ at step $h$, i.e., $n_{k,h}(s,a) \defn \sum_{m=1}^M n_{k,h}^m(s,a) = |\{\syn(k) < i \le k: (s_{i,h}^m,a_{i,h}^m)=(s,a)\}|$.
    \item $N_{k,h}(s,a)$: the total number of episodes in the interval $[1, k]$ during which all agents visit $(s,a)$ at step $h$, i.e., $N_{k,h}(s,a) \defn \sum_{m=1}^M N_{k,h}^m(s,a) = |\{1 \le i \le k: (s_{i,h}^m,a_{i,h}^m)=(s,a)\}|$.
\end{itemize}

\paragraph{Pessimism in the federated RL.} 
In offline RL, pessimism is key to preventing the overestimation of Q-function on unseen state-action space. For a single-agent case, the pessimism is implemented by subtracting a penalty term computed based on the visiting counter of an agent for each state-action pair, which makes the estimation highly dependent on the quality of agents' datasets \citep{rashidinejad2021bridging}. For example, when an agent has non-expert data collected using a highly sub-optimal behavior policy, it is inevitable to subtract a large penalty for optimal actions that cannot be reached with the agent's behavior policy, and this leads to slow convergence or convergence to a sub-optimal policy close to the behavior policy. In the federated setting, from the perspective of a server, as the aggregated information from multiple agents increases confidence, it is natural to be less pessimistic compared to an individual agent. Based on this intuition, given some prescribed probability $\delta\in(0,1)$, we suggest a global penalty computed with the aggregated counters of agents at $k \in \synset(K)$:
\begin{align} \label{eq:def-global-penalty}
  B_{k,h}(s,a) &\defn 
    \begin{cases}
      \frac{ (H+1)n_{k,h}(s,a) }{N_{k,h}(s,a)+Hn_{k,h}(s,a)}  \sqrt{\frac{c_B \logpartm^2 H^4}{N_{k,h}(s,a)}} &~~\text{if}~~ N_{k,h}(s,a) >0 \\
        0 &~~\text{if}~~ N_{k,h}(s,a) =0
    \end{cases}  ,
\end{align}
where $\logpartm = \log\left(\frac{ SA M K^2 H}{\delta} \right)$ and $c_B$ is some positive constant.
Here, the penalty for each state-action pair decreases as long as  the agents collectively explore the state-action pair enough. This relaxes the dependency on an individual agent and prevents the estimated policy from being restricted to a local behavior policy.

\paragraph{Local update uncertainty.} To guarantee that the pessimism introduced by the global penalty is enough to prevent overestimation on rarely seen state-action pairs, the penalty should dominate the uncertainty of the Q-estimates. However, when agents independently update their own local Q-estimates without frequent communication, the global penalty, which is subtracted only at the aggregation step, may fail to cover the increasing uncertainty of the local Q-estimates during local updates.
To handle this, we propose a choice of key parameters (learning rates $\eta_{k,h}^m$ and averaging weights $\alpha_{k,h}^m$) that effectively controls the uncertainty arising from the local updates as follows.
\begin{itemize}
  \item \textbf{Importance averaging.}
    In the federated setting, agents have offline datasets with heterogeneous distributions induced by different behavior policies, leading to imbalanced uncertainty of local Q-estimates.
    To minimize the uncertainty of the averaged estimate, we propose the following entrywise weighting scheme for averaging:
    \begin{align}
      \alpha_{k,h}^m(s,a) &\defn 
    \begin{cases}
          \frac{1}{M} \frac{N_{\syn(k),h}(s,a) + (H+1) M n_{k,h}^m(s,a)}{ N_{k,h}(s,a) + H n_{k,h}(s,a)} &~~\text{if}~~ n_{k,h}(s,a)>0 \\
          \frac{1}{M} &~~\text{if}~~ n_{k,h}(s,a)=0 \\
        \end{cases}. \label{eq:alpha-def}
    \end{align}
 By assigning smaller weights to less frequently updated local Q-estimates with smaller $n_{k,h}^m(s,a)$, which has high uncertainty, the averaged Q-estimate can always maintain an uncertainty level low enough to be dominated by the global penalty, regardless of the heterogeneity in local data distributions. The idea aligns with the notion of importance averaging introduced by \citet{woo23blessing}, which favors frequently updated local Q-values. 
Nevertheless, our approach differs in that, unlike \citet{woo23blessing}, where the assigned weights are determined solely based on local counters $n_{k,h}^m$ in a myopic manner, our weights, factoring in the global counter $N_{\syn(k),h}$, limit bias towards specific agents as the training of local Q-estimates stabilizes.
The weighting scheme, mindful of the entire training progress, prevents some local values that have undergone intense updates recently from dominating the global learning of the Q-function, preserving the information accumulated through old updates.

    \item \textbf{Learning rates rescaling.}
      Local updates without synchronization increase the deviation of local Q-estimates, and this increases the variance of the global Q-estimate at aggregation. However, requiring agents to communicate frequently may be too stringent for many applications in the federated setting. 
      To address this issue, we propose a novel choice of learning rate that exhibits slower decay based on a global counter $N_{\syn(k),h}$, and faster decay during local updates according to the local counter $n_{k,h}^m$:
      \begin{align}
        \eta_{k,h}^m(s,a) &\defn \frac{M (H+1)}{ N_{\syn(k),h}(s,a) + M (H+1) n_{k,h}^m(s,a)}. \label{eq:eta-def}
      \end{align}
      The rescaling of the learning rate is crucial to obtain linear  speedup without frequent synchronizations. The gradual decay with a global counter allows more aggressive updates of the Q-estimates once collective information from all agents is aggregated, which enables convergence speedup.
On the other hand, the fast decrease in learning rates during local updates ensures that agents adaptively slow down their drifts and maintain low variance of their local Q-estimates, without overly restricting the length of local updates. We will further discuss how this effectively reduces the variance of local estimates in Section~\ref{sec:analysis-basicfacts}.
  \end{itemize}
The computation of the global penalty \eqref{eq:def-global-penalty} and importance averaging \eqref{eq:alpha-def} at a server requires local counters $n_{k,h}^m(s,a)$ from every agent, and determining the learning rates \eqref{eq:eta-def} at each agent requires access to recently aggregated global counters $N_{\syn(k),h}(s,a)$. Therefore, for \pfedq with the specified parameters choices, agents and a server additionally exchange the updated local and global counters at every aggregation step.

\subsection{Theoretical guarantees}

Given the parameters described above, we now give sample complexity guarantees on the performance of the proposed \pfedq algorithm.

\begin{theorem} \label{thm:pess-fedq-regret}
Consider $\delta \in (0,1)$ and let $\hat{\pi}$ be the solution policy of \pfedq.
If a synchronization schedule $\synset(K)$ is independent of trajectories in datasets $\cD$ and satisfies 
\begin{align} \label{eq:tau-condition}
    \tau_1 \le \sqrt{\frac{H^2 S \clipavgC K }{M}} ~~\text{and}~~ \frac{\tau_{u+1}}{\tau_{u}} \le 1+\frac{2}{H}
\end{align}
for any $u\ge1$, where $\tau_u$ is the number of episodes between the $(u-1)$-th and the $u$-th aggregations. Denoting the total number of samples per agent $T=KH$, the following holds:
    \begin{align} \label{eq:pess-fedq-regret}
      V_{1}^{\star}(\rho) - V_{1}^{\hat{\pi}}(\rho)
      \le c \left( \sqrt{\frac{H^7 S \clipavgC \logpartm^2}{MT}} + \frac{H^4 S \clipavgC \logpartm}{MT} 
    \right)
    \end{align}
    at least with probability $1-\delta$, where $\logpartm = \log\left(\frac{ SA M K^2 H}{\delta} \right)$ and $c>0$ is some universal constant.
\end{theorem}

Theorem~\ref{thm:pess-fedq-regret} implies that as long as the initial synchronization occurs early and the synchronization intervals do not increase too rapidly (cf. \eqref{eq:tau-condition}), \pfedq is guaranteed to find an $\varepsilon$-optimal policy, i.e., $V_{1}^{\star}(\rho) - V_{1}^{\hat{\pi}}(\rho) \le \varepsilon$, for any target accuracy  $\varepsilon\in (0,H]$, if the total number of samples per agent $T$  exceeds
$$ \widetilde{O} \left (\frac{ H^7 S \clipavgC}{M \varepsilon^2}  \right). $$

A few implications are in order.

\paragraph{Linear speedup without expert datasets.} The value function gap shows linear speedup with respect to the number of agents $M$, highlighting the benefit of collaboration. Notably, the guarantee holds even when every agent has low-quality datasets collected by some sub-optimal behavior policy, as long as agents' local data distributions collectively cover the distribution of the optimal policy, where the average single-policy concentrability $\clipavgC$ (cf. \eqref{eq:clipped-avg-concentrability-assumption}) is finite. On the other end, when performing offline RL using a single agent, it requires that the behavior policy of the single agent individually cover the optimal policy, i.e., $C^{\star}<\infty$ (cf. \eqref{equ:concentrability-assumption}), which is much more stringent. Therefore, federated offline RL enables policy learning that otherwise will not be possible in the single-agent setting. Specializing to the case $M=1$, our bound nearly matches the sample complexity bound $\widetilde{O} \Big( \frac{H^6SC^{\star}}{\varepsilon^2}\Big)$ obtained for a single-agent pessimistic Q-learning algorithm with a similar Hoeffding-style penalty \citep{shi2022pessimistic}, up to a factor of $H$.

\paragraph{Comparison with offline RL using shared datasets.}
To benchmark the tightness of our bound, let us consider the minimax lower bound of the sample complexity for single-agent offline RL \citep{li2022settling}, as if we collect all the agents' datasets at a central location. Note that the effective single-policy concentrability coefficient (cf. \eqref{equ:concentrability-assumption}) for the combined datasets $\cD_{\mathsf{all}}= \cup_{m=1}^M \cD^m$ becomes
\begin{align}
    \max_{(h, s, a) \in [H] \times \mathcal{S} \times \cA} \frac{ \min \{ d^{\pi^\star}_h(s,a), \, 1/S\}}{\sum_{m=1}^Md^{m}_h(s,a)}
    = \max_{(h, s, a) \in [H] \times \mathcal{S} \times \cA} \frac{ \min \{ d^{\pi^\star}_h(s,a), \, 1/S\}}{M d_h^{\mathsf{avg}}(s,a)} 
    = \frac{\clipavgC}{M} ,
\end{align}
leading to the minimax lower bound \citep{li2022settling}
$$\widetilde{\Omega} \left(  \frac{H^4S \clipavgC}{M\varepsilon^2} \right) .$$
 Comparing with  the sample complexity bound of \pfedq, obtained as $\widetilde{O} \Big (\frac{ H^7 S \clipavgC}{M \varepsilon^2} \Big )$, this suggests that the performance of \pfedq is near-optimal up to polynomial factors of $H^3$ even when compared with the single-agent counterpart assuming shared access to all agents' datasets.

\paragraph{Communication efficiency.} 

\begin{figure}[t]
\centering
\includegraphics[trim={0.1cm 0 0 0.1cm},clip, width=0.55\textwidth]{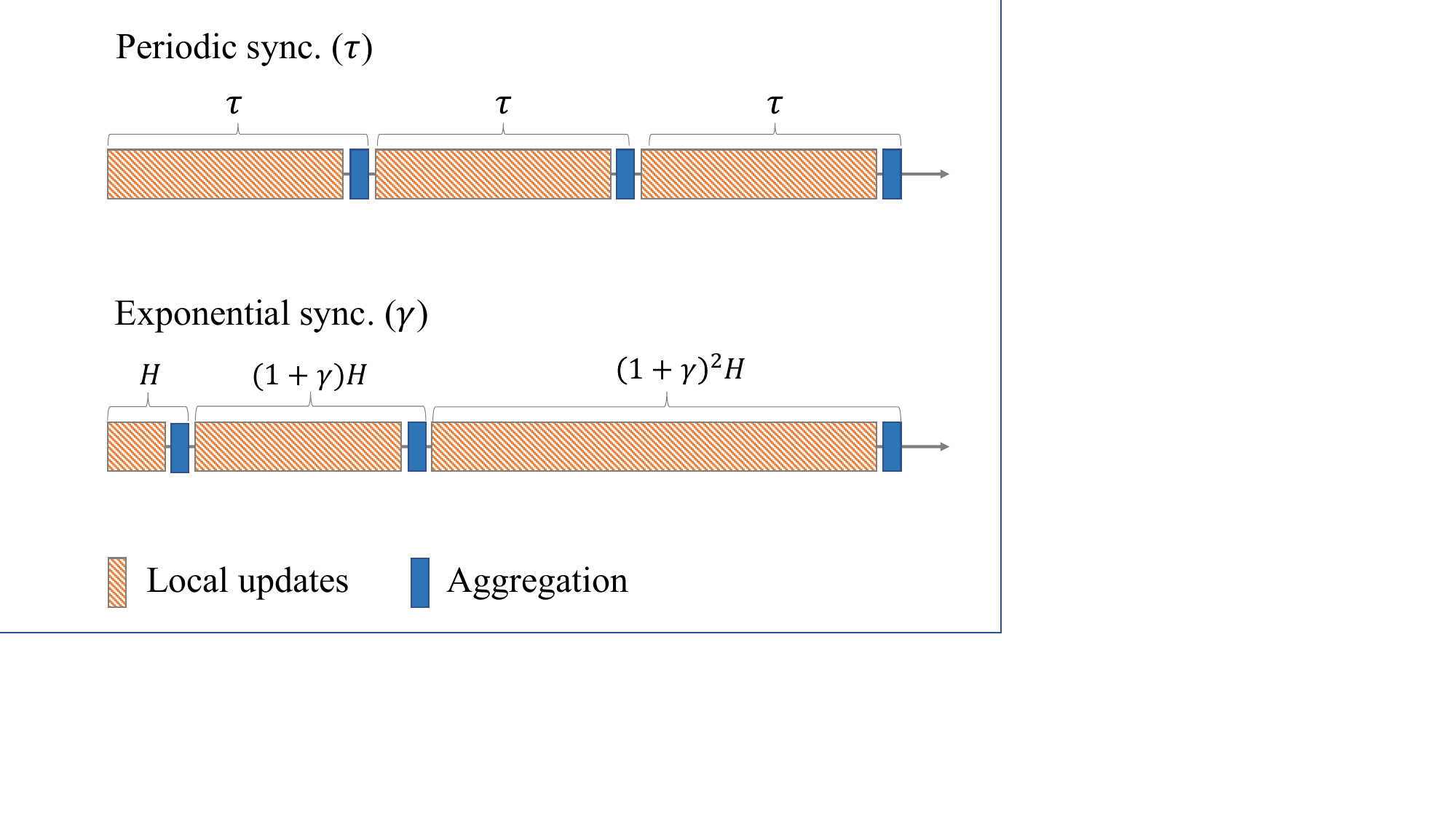}
\caption{Illustration of the periodic synchronization with constant period $\tau$ and the exponential synchronization with a rate $\gamma$.}
\label{fig:sync-schedule}
\end{figure}

Theorem~\ref{thm:pess-fedq-regret} suggests initiating the first synchronization early and avoiding rapid increases in synchronization intervals (cf. \eqref{eq:tau-condition})
to ensure fast convergence. This is attributed to large deviations among agents in the early stages, arising due to coarse Q-estimates and large learning rates, which diminish as training proceeds.
For communication efficiency, it is essential to design a synchronization schedule that meets the constraints with the least number of synchronizations. We investigate the following two specific synchronization schedules for \pfedq:
 \begin{enumerate}[label=(\alph*)]
 \item \textbf{Periodic synchronization:} For a fixed period $\tau \ge 1$, communication between agents and a server is available for every $\tau$ episodes, i.e., $\tau_i = \tau$ for all $i \ge1 $, and we denote the synchronization schedule as $\synset_{\mathsf{period}}(K, \tau)$.
 \item \textbf{Exponential synchronization: } For a fixed ratio $\gamma>0$, initializing $\tau_1=H$, set $\tau_i =  \lfloor (1+\gamma) \tau_{i-1} \rfloor$ for each $i\ge 2$. Under this scheduling, agents communicate frequently at initial iterations, but the period between aggregation steps increases exponentially with the rate of $(1+\gamma)$ and synchronization occurs rarely as training proceeds enough. We denote the synchronization schedule as $\synset_{\mathsf{exp}}(K, \gamma).$
\end{enumerate}

Now, we analyze the number of communication rounds required to achieve a target accuracy, for each scheduling scheme.
\begin{corollary} \label{cor:fedq-comm}
  For any given $\delta \in (0,1)$ and target error $\varepsilon \in (0, \min \{H, \frac{H^3 S \clipavgC}{M} \} ]$, suppose the total number of samples per agent $T=KH$ satisfies
  \begin{align*}
    T \asymp \frac{ H^7 S \clipavgC}{M \varepsilon^2},
  \end{align*}
  and \pfedq performs under the periodic synchronization scheduling, i.e., $\synset(K) = \synset_{\mathsf{period}}(K, \tau)$, with $\tau \asymp \sqrt{\frac{H S\clipavgC T}{M}}$, or the exponential synchronization scheduling, i.e., $\synset(K) = \synset_{\mathsf{exp}}(K, \gamma)$, with $\gamma = \frac{2}{H}$. Then, each schedule requires the number of synchronizations at most
  \begin{subequations}
  \begin{align}
    \mathsf{(Periodic)} &\quad |\synset_{\mathsf{period}}(K, \tau)| \lesssim \sqrt{\frac{MK}{H^2S\clipavgC }} 
    , \\ 
    \mathsf{(Exponential)} &\quad |\synset_{\mathsf{exp}}(K,\gamma)| \lesssim H,
  \end{align}
  \end{subequations}
  respectively, and the solution policy $\hat{\pi}$ of \pfedq is guaranteed to be an $\varepsilon$-optimal policy at least with probability $1- \delta$.
\end{corollary}

Corollary~\ref{cor:fedq-comm} implies that \pfedq requires only $\widetilde{O}(H)$ aggregations to achieve the target accuracy under appropriate synchronization schedules, such as the exponential synchronization schedule. Notably, the number of communication rounds is nearly independent of the size of the state-action space, the total number of episodes, or the number of agents, and this outperforms prior art \citep{woo23blessing}.
Furthermore, analysis suggests that exponential synchronization with a modest rate $\gamma = 2/H$ is a key to achieving such communication efficiency.
With our strategic choices of learning rates, local Q-estimates stabilize as training proceeds, and thus agents can perform more local updates than previous rounds without increasing uncertainty beyond the control of the global pessimism penalty.
Exponential synchronization reduces the number of synchronizations by capturing the additional room for local updates arising from the stabilization of Q-estimates.
On the other hand, periodic synchronization does not exploit this benefit, even if we set the period $\tau$ maximally under \eqref{eq:tau-condition} due to which it necessitates more communication rounds, which increase with $K$ and $M$.

\section{Analysis}
\label{sec:analysis}
In this section, we will outline useful properties of \pfedq and the key steps of the proof of Theorem~\ref{thm:pess-fedq-regret}, deferring the details, such as proofs of supporting lemmas, to Appendix~\ref{sec:tech_lemmas} and \ref{sec:proofs_main}.

Throughout the paper, we adopt the following shorthand notation 
\begin{equation} \label{eq:transition_vector}
P_{h,s,a} \coloneqq P_h(\cdot \mymid s,a ) \in [0,1]^{1\times S},
\end{equation}
which represents the transition probability vector given the current state-action pair $(s,a)$ at step $h$. In addition, define $P_{k,h}^m \in \{0,1\}^{1\times S}$ as the empirical transition vector at step $h$ of the $k$-th episode at agent $m$, namely
\begin{equation}
P_{k,h}^m (s) = \mathbb{I}(s = s_{k,h+1}^m), \quad \mbox{for all }s \in\mathcal{S}. 
\end{equation}

These are the notations pertaining to the counters for visits of agents on each state-action pair $(s,a) \in \cS \times \cA$. For any $(m, k, h) \in [M] \times [K] \times [H]$,
\begin{itemize}
    \item $l_{k,h}^m(s,a)$: a set of episodes in the interval $(\syn(k), k]$ during which agent $m$ visits $(s,a)$ at step $h$, i.e., $l_{k,h}^m(s,a) \defn \{\syn(k) < i \le k: (s_{i,h}^m,a_{i,h}^m)=(s,a)\}$.
    \item $L_{k,h}^m(s,a)$: a set of episodes in the interval $[1, k]$ during which agent $m$ visits $(s,a)$ at step $h$, i.e. $L_{k,h}^m(s,a) \defn \{1 \le i \le k: (s_{i,h}^m,a_{i,h}^m)=(s,a)\}$.
\end{itemize}

We also introduce the following notation related to the synchronization schedule $\synset(K)$. For any positive integer $k$ and $u$,  
\begin{itemize}
    \item $t_u$: the index of episodes, after which the $u$th synchronization occurs.
    \item $\tau_u$: the number of local updates (episodes) taken between the $(u-1)$th and the $u$th synchronizations.
    \item $\syn(k)$: the most recent episode where the aggregation occurs before the $k$th episode.
    \item $\nsyn(k)$: the minimum index of aggregation occurring after $k$-th episode.
\end{itemize}

\subsection{Basic facts}
\label{sec:analysis-basicfacts}

\paragraph{Error recursion of Q-estimates.} We begin with the following key error decomposition of the Q-estimate at each synchronization, whose proof is provided in Appendix~\ref{proof:qerr-dcomp-base}.  

\begin{lemma}[Q-estimation error decomposition] \label{lemma:qerr-dcomp-base}
Consider a Q-function $Q^{\pi} = \{Q^{\pi}_{h}(s,a)\}_{[H]\times\cS\times\cA}$ and value function $V^{\pi} = \{V^{\pi}_{h}(s)\}_{[H]\times\cS}$ induced by a policy $\pi$. 
Then, for any $[H]\times\cS\times\cA$ and $k \in \synset(K)$, the error between $Q_{h}^{\pi}$ and $Q_{k,h}$ is decomposed as follows:
\begin{align}\label{eq:qerror-decomp-base}
    Q_{h}^{\pi}(s,a) -  Q_{k, h} (s, a) 
        &= \underbrace{  \omega_{0,k,h}(s,a) (Q_{h}^{\pi}(s,a) -  Q_{0, h} (s, a)) }_{=: D_1^{\pi}(s,a,k,h):~\textsf{initialization error}}\cr
    &\qquad + \underbrace{  \sum_{m=1}^M \sum_{i \in L_{k,h}^m(s,a)} \omega_{i,k,h}^m(s,a) (P_{h,s,a} - P_{i,h}^m) V^{m}_{i-1, h+1}  }_{=:D_2(s,a,k,h):~\textsf{transition variance}} \cr
    &\qquad+ \underbrace{  \sum_{u=1}^{\nsyn(k)}  B_{t_u,h}(s,a) \prod_{u'=u+1}^{\nsyn(k)} \lambda_{u',h}(s,a)  }_{=: D_3(s,a,k,h):~\textsf{global penalty}} \cr
  &\qquad + \underbrace{  \sum_{m=1}^M \sum_{i \in L_{k,h}^m(s,a)} \omega_{i,k,h}^m(s,a) P_{h,s,a} (V^{\pi}_{h+1}-V^{m}_{i-1, h+1})  }_{=:D_4^{\pi}(s,a,k,h):~\textsf{recursion}} ,
\end{align}
where $L_{k,h}^m(s,a) \defn \{1 \le i \le k: (s_{i,h}^m,a_{i,h}^m)=(s,a)\}$ and $l_{k,h}^m(s,a) \defn \{\syn(k) < i \le k: (s_{i,h}^m,a_{i,h}^m)=(s,a)\}$. And, for simplicity, we use the shortened notations defined as 
\begin{subequations}
\begin{align}  
      \label{eq:lr-lambda} \lambda_{v,h}(s,a) &=
    \begin{cases}
      1 &~~\text{if}~~ N_{k,h}(s,a)=0 \\
      \frac{ N_{\syn(k),h}(s,a) }{N_{k,h}(s,a) + H n_{k,h}(s,a)} &~~\mbox{otherwise}
    \end{cases}, \quad v = \phi(k), \\  
     \omega_{0,k,h}^m(s,a)&  =     \begin{cases}
      1 &~~\text{if}~~ N_{k,h}(s,a)=0 \label{eq:def-omega0k}\\
     0 &~~\mbox{otherwise}
    \end{cases},  \\
        \omega_{i,k,h}^m(s,a)&  = \frac{ H+1}{N_{k,h}(s,a) +  H n_{k,h}(s,a)}
        \left (\prod_{x=\nsyn(i)}^{\nsyn(k)-1} \frac{ N_{t_x,h}(s,a)}{N_{t_x,h}(s,a) + H n_{t_x,h}(s,a)} \right), \quad i\in L_{k,h}^m(s,a). \label{eq:def-omegaik}
\end{align}
\end{subequations}
\end{lemma}

\begin{figure}
  \centering
  \begin{subfigure}[b]{0.48\textwidth}
    \centering
    \includegraphics[width=\textwidth]{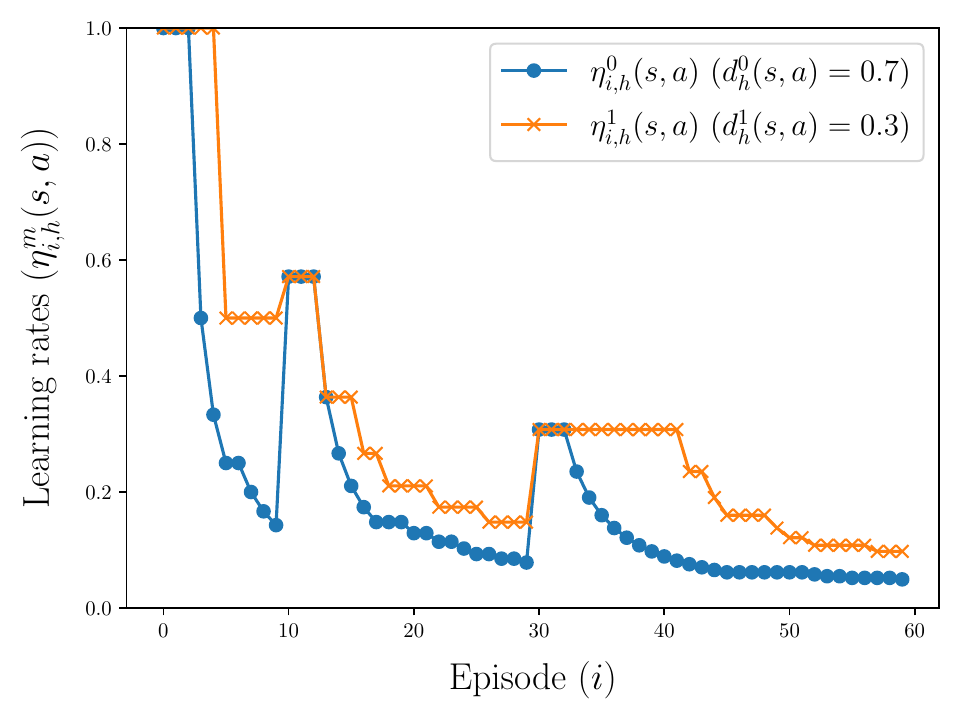}
    \caption{Rescaled learning rates}
    \label{fig:lr-rescaled}
  \end{subfigure}
  \hfill
  \begin{subfigure}[b]{0.48\textwidth}
    \centering
    \includegraphics[width=\textwidth]{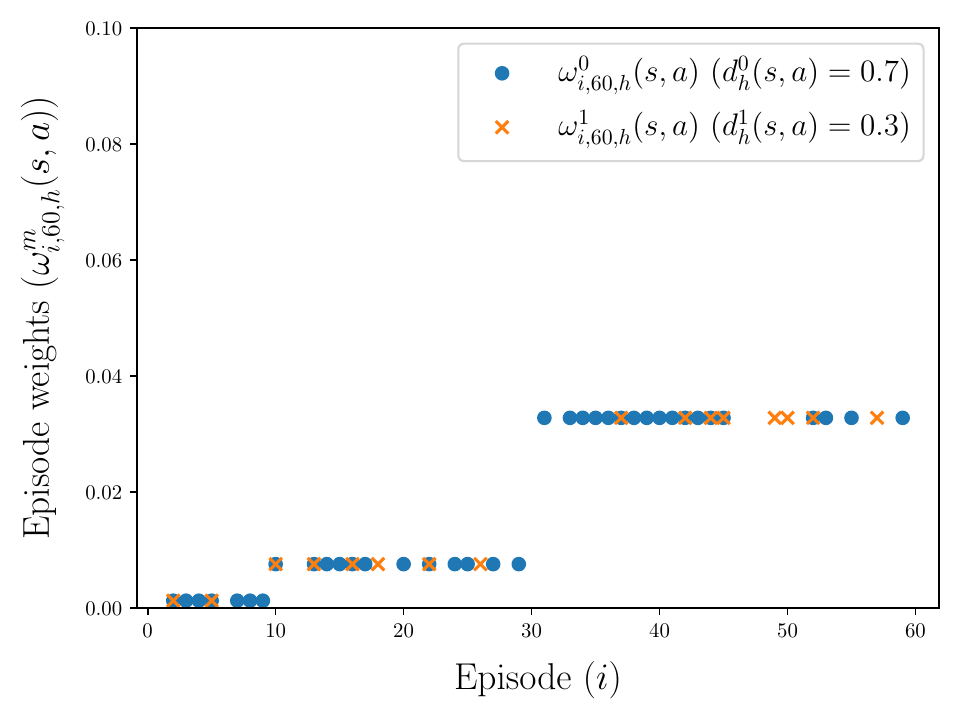}
    \caption{Episode weights }
    \label{fig:omega-rescaled}
  \end{subfigure}
  \caption{Illustration of the rescaled learning rates ($\eta_{i,h}^m(s,a)$) and the episode weights ($\omega_{i,60,h}^m(s,a)$) induced by the learning rates of two agents $m=0,1$ for episodes $1\le i \le 60$, where $H=5$, the occupancy distribution of each agent on $(s,a,h)\in \cS \times \cA \times [5]$ is $d_h^0(s,a) = 0.7$ and $d_h^1(s,a) = 0.3$, respectively, and the synchronization schedule is $\synset(60)= \{ 10,30, 60 \}$.}
  \label{fig:rescaled}
\end{figure}

\paragraph{Equally favoring episodes within the same local update round.}
According to the decomposition \eqref{eq:qerror-decomp-base} in Lemma~\ref{lemma:qerr-dcomp-base}, for any $(s,a,h) \in \cS \times \cA \times [H]$, the $Q$-estimation error at episode $k$ significantly depends on the weighted sum of transition difference for each episode where the local update occurs, namely $D_2(s,a,k,h)$. Intuitively, the weight $\omega_{i,k,h}^m(s,a)$ assigned to each episode $i$ balances the accumulation of information from old and new updates. Our choice of learning rates, which decreases fast during local updates, as illustrated in Figure~\ref{fig:lr-rescaled}, ensures that the weight $\omega_{i,k,h}^m(s,a)$ within the same local update round is always equal for all episodes and agents, as shown in \eqref{eq:def-omegaik} and Figure~\ref{fig:omega-rescaled}. The uniform weights allow the transition information of each episode to be accumulated evenly, regardless of other transitions that occur in future episodes or other agents' episodes. This is essential to keep variance arising from local updates low, especially when a synchronization interval is long.
Assigning equal weight to every episode allows to fully utilize transitions observed during local updates without forgetting old information, regardless of the length of the synchronization interval.

\paragraph{Bounded visitation counters.} We next introduce the following lemma regarding the visitation counters, whose proof is provided in Appendix~\ref{proof:avgvisit_bound}.
\begin{lemma}[Concentration bound on the visitation counters] \label{lemma:avgvisit_bound}
    Consider any $\delta \in (0,1)$ and some universal constant $c_1 >0$, and let  
    \begin{align} \label{eq:definition_mua}
      \logpart \defn \log\left(\frac{2 |\cS||\cA| K H}{\delta} \right) ~~\text{and}~~
      K_0(s,a,h) \defn \frac{4  \logpart }{c_1 Md_h^{\mathsf{avg}}(s,a)}. 
  \end{align}
   Then, for all $(s,a,h) \in \cS \times \cA \times [H]$, the following holds
   \begin{subequations} \label{eq:banana}
   \begin{align}
   \text{when }k \ge K_0(s,a,h): \qquad  & \frac{1}{2} k M  d_{h}^{\mathsf{avg}}(s,a)   \le  N_{k,h}(s,a) \le 2k M  d_{h}^{\mathsf{avg}}(s,a) , \label{eq:banana-1} \\
    \text{when }k \le K_0(s,a,h): \qquad  &  N_{k,h}(s,a) \le 8 \logpart/c_1 \label{eq:banana-2}
   \end{align}
   \end{subequations}
   with probability at least $1-\delta$.
\end{lemma}

\paragraph{Monotonic and pessimistic global value updates.} Note that the global value estimate is always monotonically non-decreasing, i.e., for $k' , k \in \synset(K)$ it holds
\begin{equation}\label{eq:monotonic_V}
\forall s \in \cS: \qquad V_{k,h}(s) \ge V_{k',h}(s) \quad \mbox{when } k' \le k , 
\end{equation}
which follows directly from the update rule \eqref{eq:vglobal-monotone}. Moreover, we have the following important lemma regarding the pessimistic property of the value estimate, whose proof is provided in 
 Appendix~\ref{proof:pessimistic_value}.

\begin{lemma}[Pessimistic global value] \label{lemma:pessimistic_value}
   Recall $Q_{k,h}$, $V_{k,h}$, and $\pi_{k,h}$ in Algorithm~\ref{alg:pfedq}. 
   Let $\pi_{k} = \{\pi_{k,h}\}_{h \in [H]}$.
   Given any $\delta \in (0,1)$, for all $(k,h) \in \synset(K) \times[H]$, it holds with probability at least $1-\delta$ that
    \begin{subequations} 
    \begin{align}
        \label{eq:var_penalty_dom} \forall (s,a) \in \cS\times\cA : \qquad  & | D_2  (s,a,k,h) | \le D_3 (s,a,k,h) \leq \sqrt{\frac{4  c_B \logpartm^2 H^4 }{ \max\{ N_{k,h}(s,a), 1 \} } } , \\
            \label{eq:Q_lower} \forall (s,a) \in \cS \times \cA: \qquad & Q_{k, h} (s,a) \le Q_{h}^{\pi_{k}}(s,a) \le  Q_{h}^{\star}(s,a), \\
        \label{eq:V_lower} \forall s \in \cS : \qquad & V_{k, h} (s) \le V_{h}^{\pi_{k}}(s) \le  V_{h}^{\star}(s).
    \end{align}
    \end{subequations}
\end{lemma}
In words, Lemma~\ref{lemma:pessimistic_value} makes concrete the role of the penalty term in dominating the variability of the value estimates due to stochastic transitions, and ensures that the estimated value is a pessimistic estimate of the true optimal value function.

\subsection{Proof of Theorem~\ref{thm:pess-fedq-regret}}
Now we are ready to provide the proof of Theorem~\ref{thm:pess-fedq-regret}, which is divided into several key steps as follows.

\paragraph{Step 1: decomposition of the performance gap.} The performance gap between the solution policy $\hat{\pi}$ of Algorithm~\ref{alg:pfedq} after $K$ episodes and the optimal policy $\pi^{\star}$  can be bounded as follows:
\begin{align} \label{eq:optgap-to-regret}
        V_{1}^{\star}(\rho) - V_{1}^{\hat{\pi}}(\rho) 
        &= \mexp_{s_1 \sim \rho}\left[ V_{1}^{\star}(s_1)\right] - \mexp_{s_1 \sim \rho}\left[ V_{1}^{\pi_{K}}(s_1)\right] \cr
        &\stackrel{\mathrm{(i)}}{\le} \mexp_{s_1 \sim \rho}\left[ V_{1}^{\star}(s_1)\right] - \mexp_{s_1 \sim \rho}\left[ V_{K, 1} (s_1)\right] \cr
        &\stackrel{\mathrm{(ii)}}{\le}  \frac{1}{K} \sum_{v=1}^{\nsyn(K)} \tau_v \big (\mexp_{s_1 \sim \rho}\left[ V_{1}^{\star}(s_1)\right] - \mexp_{s_1 \sim \rho}\left[ V_{t_v, 1} (s_1)\right] \big ) \cr
        &=  \frac{1}{K} \sum_{v=1}^{\nsyn(K)}  \tau_v \sum_{s \in \cS} \underbrace{d_1^{\pi^{\star}}(s)}_{=\rho(s)} \left (  V_{1}^{\star}(s) -  V_{t_v, 1} (s)  \right ) \cr
        &\le  \frac{1}{K}  \max_{h \in [H]}  \sum_{v=1}^{\nsyn(K)} \tau_v \sum_{s \in \cS} d_h^{\pi^{\star}}(s) \left(  V_{h}^{\star}(s) -  V_{t_v, h} (s)  \right) ,
\end{align}
where (i) follows from Lemma~\ref{lemma:pessimistic_value}, and (ii) follows from the monotonicity property in \eqref{eq:monotonic_V} and $\sum_{v=1}^{\nsyn(K)} \tau_v = K$.

Since $\pi^{\star} = \{\pi^\star_h\}_{h\in[H]}$ is deterministic, for any $k \in \synset(K)$ and $h \in [H]$, it follows that
\begin{align} \label{eq:regretv-to-q}
    \sum_{s \in \cS} d_h^{\pi^{\star}}(s) \left (  V_{h}^{\star}(s) -  V_{k, h} (s)  \right )  
    &= \sum_{s \in \cS}  d_h^{\pi^{\star}}(s, \pi_h^{\star}(s)) \left (  V_{h}^{\star}(s) -  V_{k, h} (s)  \right )   \cr
    &\le \sum_{s \in \cS}  d_h^{\pi^{\star}}(s, \pi_h^{\star}(s)) \big (  Q_{h}^{\star}(s,\pi_h^{\star}(s)) -  Q_{k, h} (s, \pi_h^{\star}(s))  \big ) ,
\end{align} 
where the inequality holds because $ Q_{k, h}(s, \pi_h^{\star}(s)) \le \max_{a \in \cA} Q_{k, h}(s, a) \le V_{k, h} (s)$ due to \eqref{eq:vglobal-monotone}. 

To continue, applying Lemma~\ref{lemma:qerr-dcomp-base} by setting $\pi = \pi^{\star}$, the Q-estimate error after $k$ episodes is decomposed as follows:  
\begin{align}\label{eq:q-decomp}
    Q_{h}^{\star}(s,a) -  Q_{k, h} (s, a)  & =  D_1^{\pi^{\star}}(s,a,k,h) + D_2 (s,a,k,h) +D_3(s,a,k,h) +D_4^{\pi^{\star}}(s,a,k,h) \nonumber \\
    & \leq D_1^{\pi^{\star}}(s,a,k,h) +  D_4^{\pi^{\star}}(s,a,k,h) + 2D_3(s,a,k,h),
\end{align}
where the second line follows from Lemma~\ref{lemma:pessimistic_value}. Finally, inserting the decomposition \eqref{eq:q-decomp} and \eqref{eq:regretv-to-q} back into \eqref{eq:optgap-to-regret}, we control the performance gap with the following terms:  
\begin{align} \label{eq:optgap-to-decomp}
 & V_{1}^{\star}(\rho) - V_{1}^{\hat{\pi}}(\rho) \nonumber \\
&  \le   \frac{1}{K}  \max_{h \in [H]}  \sum_{v=1}^{\nsyn(K)} \tau_v \sum_{s \in \cS} d_h^{\pi^{\star}}(s) \left[ D_1^{\pi^{\star}}(s,\pi_h^{\star}(s),t_v,h) +  D_4^{\pi^{\star}}(s,\pi_h^{\star}(s),t_v, h) + 2D_3(s,\pi_h^{\star}(s),t_v,h) \right] \nonumber \\
& =:   \frac{1}{K}  \max_{h \in [H]}   \left( D_{1,h} + D_{4,h} + 2 D_{3,h} \right) ,
\end{align}
for which we shall aim to bound each term individually, adopting the following short-hand notation:
\begin{align} \label{eq:short-hand-main}
D_{i,h} &\defn \sum_{v=1}^{\nsyn(K)} \tau_v \sum_{s \in \cS} d_h^{\pi^{\star}}(s)  D_i^{\pi^{\star}}(s,\pi_h^{\star}(s),t_v,h) \qquad  \text{for } i\in \{1, 4\}, \notag\\
D_{3,h} &\defn \sum_{v=1}^{\nsyn(K)} \tau_v \sum_{s \in \cS} d_h^{\pi^{\star}}(s) D_3(s,\pi_h^{\star}(s),t_v,h).
\end{align}

\paragraph{Step 2: Bounding the decomposed terms.}
Here, we derive the bound of the decomposed terms separately as follows under the event  that \eqref{eq:banana} holds, which is denoted as $\mathcal{E}_0$ and holds with probability at least $1-\delta$.
\begin{itemize}
\item \textbf{Bounding $D_{1,h}$. }
  Using the fact that $0 \le Q_{h}^{\star}(s,\pi_h^{\star}(s)) -  Q_{0, h} (s, \pi_h^{\star}(s)) \le H$, which follows from Lemma~\ref{lemma:pessimistic_value}, it follows 
        \begin{align}
          D_{1,h}
          &= \sum_{v=1}^{\nsyn(K)} \tau_v \sum_{s\in \cS} d_h^{\pi^{\star}}(s, \pi_h^{\star}(s)) \omega_{0,t_v,h}(s,\pi_h^{\star}(s)) (Q_{h}^{\star}(s,\pi_h^{\star}(s)) -  Q_{0, h} (s, \pi_h^{\star}(s))) \cr 
          &\le\sum_{v=1}^{\nsyn(K)} \tau_v \sum_{s\in \cS} d_h^{\pi^{\star}}(s, \pi_h^{\star}(s)) \omega_{0,t_v,h}(s,\pi_h^{\star}(s)) H \cr
          & = H \sum_{s\in \cS} d_h^{\pi^{\star}}(s, \pi_h^{\star}(s)) \sum_{v=1}^{\nsyn(K)} \tau_v \one\{N_{t_v,h}(s,\pi_h^{\star}(s)) = 0\} ,  \label{eq:D1_part1}
          \end{align}
       where the last line follows from  \eqref{eq:def-omega0k}. To continue, note that  
          \begin{align*}
    \sum_{v=1}^{\nsyn(K)} \tau_v \one\{N_{t_v,h}(s,\pi_h^{\star}(s)) = 0\}     
   & =    \sum_{v \in [\nsyn(K)]: t_v \leq K_0(s,\pi_h^{\star}(s),h)}  \tau_v \one\{N_{t_v,h}(s,\pi_h^{\star}(s)) = 0\} \\
  &   \leq  K_0(s,\pi_h^{\star}(s),h) ,
          \end{align*}
 since under the event $\mathcal{E}_0$, $N_{t_v,h}(s,\pi_h^{\star}(s)) >0$ when $t_v > K_0(s,\pi_h^{\star}(s),h)$. Plugging the above inequality and the definition of $ K_0(s,\pi_h^{\star}(s),h) $ back to \eqref{eq:D1_part1} leads to
          \begin{align} \label{eq:D1-final}
        D_{1,h}        &\le H \sum_{s\in \cS} d_h^{\pi^{\star}}(s, \pi_h^{\star}(s))   K_0(s,\pi_h^{\star}(s),h) \nonumber \\
        & = H \sum_{s\in \cS} \frac{\min\{d_h^{\pi^{\star}}(s, \pi_h^{\star}(s)), 1/S\}}{d_{h}^{\mathsf{avg}}(s,\pi_h^{\star}(s))} \left ( \frac{12 \logpart}{ M}  \right ) \frac{d_h^{\pi^{\star}}(s, \pi_h^{\star}(s))}{\min\{d_h^{\pi^{\star}}(s, \pi_h^{\star}(s)), 1/S\}}\cr
        & \lesssim  \frac{H \clipavgC S }{M} ,
        \end{align}
        where the last line follows from the definition of $\clipavgC$ and the fact that $$\sum_{s\in \cS}  \frac{d_h^{\pi^{\star}}(s, \pi_h^{\star}(s))}{\min\{d_h^{\pi^{\star}}(s, \pi_h^{\star}(s)), 1/S\}} \le \sum_{s\in \cS}  \left( 1+ d_h^{\pi^{\star}}(s, \pi_h^{\star}(s)) S \right) =  \sum_{s\in \cS}  \left( 1+ d_h^{\pi^{\star}}(s) S \right) = 2S. $$
        
      \item \textbf{Bounding $D_{3,h}$. } 
    The range of $D_3(s,a,k,h)$ is bounded as shown in the following lemma, whose proof is provided in Appendix~\ref{proof:d3-bound}.
    \begin{lemma} \label{lemma:d3-bound}
          For any $(s,a,h) \in \cS \times \cA \times [H]$ and $k \in \synset(K)$, if $N_{k,h}(s,a) =0$, $D_3(s,a,k,h) = 0$, and if, $N_{k,h}(s,a) >0$, the following holds:
         \begin{align} \label{eq:d3-bound-N}
           D_3(s,a,k,h) \in   \left [\sqrt{\frac{ c_B \logpartm^2 H^4 }{N_{k,h}(s,a)}} , \sqrt{\frac{4  c_B \logpartm^2 H^4 }{N_{k,h}(s,a)}}  \right] .
        \end{align}
    \end{lemma}        
    With the above lemma in hand, recalling \eqref{eq:short-hand-main} gives 
    \begin{align} \label{eq:D3-part1}
       D_{3,h}
          &= \sum_{v=1}^{\nsyn(K)} \tau_v \sum_{s\in \cS} d_h^{\pi^{\star}}(s, \pi_h^{\star}(s))  D_3(s,\pi_h^{\star}(s),t_v,h)  \cr
          &\le \sum_{s\in \cS} d_h^{\pi^{\star}}(s, \pi_h^{\star}(s)) \sum_{v=1}^{\nsyn(K)}  \tau_v    \sqrt{\frac{4  c_B \logpartm^2 H^4 }{\max\{N_{t_v,h}(s,\pi_h^{\star}(s)),1\}}}  .
        \end{align}
        According to Lemma~\ref{lemma:avgvisit_bound}, $N_{t_v,h}(s,a) \ge \frac{1}{2} t_v M  d_{h}^{\mathsf{avg}}(s,a)$ holds if $t_v \ge K_0(s,a,h)$ under the event $\mathcal{E}_0$. Therefore,
        \begin{align}
        \sum_{v=1}^{\nsyn(K)} \tau_v \sqrt{\frac{H^4 }{\max\{N_{t_v,h}(s,a),1\}}}  
        &\leq \sum_{v: t_v \le K_0(s,a,h)} \tau_v H^2 + \sum_{v: t_v > K_0(s,a,h)} \tau_v \sqrt{\frac{ H^4 }{\max\{N_{t_v,h}(s,a),1\}}} \cr
          &\lesssim   H^2 K_0(s, a,h)
          + \sum_{v: t_v > K_0(s,a,h)}  \tau_v \sqrt{\frac{ H^4 }{\max\{N_{t_v,h}(s,a),1\}}} \cr
           &\lesssim   H^2 K_0(s, a,h)
          + \sum_{v=1}^{\nsyn(K)}  \tau_v \sqrt{\frac{H^4}{M t_v d_h^{\mathsf{avg}}(s, a)}} .
        \end{align}
        Plugging the above inequality and the definitions of $ K_0(s,\pi_h^{\star}(s),h) $ (cf.~\eqref{eq:definition_mua}) and $\clipavgC$ to \eqref{eq:D3-part1}, we obtain
        \begin{align} \label{eq:D3-final}
       D_{3,h}
          &\lesssim \frac{H^2}{M } \sum_{s\in \cS} \frac{d_h^{\pi^{\star}}(s, \pi_h^{\star}(s))}{d_{h}^{\mathsf{avg}}(s,\pi_h^{\star}(s))} 
          + \sum_{v=1}^{\nsyn(K)} \sum_{s\in \cS} d_h^{\pi^{\star}}(s, \pi_h^{\star}(s))  \tau_v \sqrt{\frac{H^4}{M t_v d_h^{\mathsf{avg}}(s, \pi_h^{\star}(s))}} \cr 
          &\lesssim \frac{H^2 \clipavgC }{M } \sum_{s\in \cS} \frac{d_h^{\pi^{\star}}(s, \pi_h^{\star}(s))}{\min\{d_h^{\pi^{\star}}(s, \pi_h^{\star}(s)), 1/S\}} 
           +   \sum_{v=1}^{\nsyn(K)} \sqrt{\frac{H^4 \clipavgC \tau_v^2}{ M t_v }}  \sum_{s\in \cS}  \sqrt{\frac{(d_h^{\pi^{\star}}(s, \pi_h^{\star}(s)))^2}{\min\{d_h^{\pi^{\star}}(s, \pi_h^{\star}(s)), 1/S\}}}  \cr 
          &\stackrel{\mathrm{(i)}}{\lesssim} \frac{H^2 \clipavgC S}{M} 
            + \sqrt{\frac{H^4 \clipavgC S}{M}} \sum_{v=1}^{\nsyn(K)}\sqrt{\tau_v} \sqrt{\frac{\tau_v}{t_v }}   \cr
          &\stackrel{\mathrm{(ii)}}{\lesssim}  \frac{H^2 \clipavgC S}{M}   
          + \sqrt{\frac{H^4 S K \clipavgC }{M}},
        \end{align}
        where (i) holds due to the Cauchy-Schwarz inequality and the fact that 
        $$\sum_{s\in \cS}  \frac{d_h^{\pi^{\star}}(s, \pi_h^{\star}(s))}{\min\{d_h^{\pi^{\star}}(s, \pi_h^{\star}(s)), 1/S\}} \le \sum_{s\in \cS}  \left( 1+ d_h^{\pi^{\star}}(s, \pi_h^{\star}(s)) S \right) =  \sum_{s\in \cS}  \left( 1+ d_h^{\pi^{\star}}(s) S \right) = 2S,$$
        and the last line (ii) follows from the Cauchy-Schwarz inequality and the fact that $\sum_{v=1}^{\nsyn(K)} \tau_v = K$ and $\sum_{v=1}^{\nsyn(K)}  \frac{\tau_v}{t_v} \le 1+\log{K}$, with the latter following from Lemma~\ref{lemma:seqsum} (see Appendix~\ref{sec:tech_lemmas}).

      \item \textbf{Bounding $D_{4,h}$. }
    In the following lemma, whose proof is provided in Appendix~\ref{proof:vgap_recursion}, we extract the recursive formulation of $D_{4,h}$  as follows.
    \begin{lemma}\label{lemma:vgap_recursion}
      Consider any $\delta \in (0,1)$. For any $h \in [H]$, the following holds with probability at least \blue{$1-\delta$}:
      \begin{align} \label{eq:vgap_recursion}
        &\sum_{v=1}^{\nsyn(K)} \tau_v \sum_{(s,a) \in \cS \times \cA} d_h^{\pi^{\star}}(s, a) \sum_{m=1}^M \sum_{i \in L_{t_v,h}^m(s,a)} \omega_{i,t_v,h}^m(s,a) P_{h,s,a} (V^{\star}_{h+1}-V_{\syn(i), h+1}) \cr
        &\lesssim \sigma_{\mathsf{aux}} + \left(1+ \frac{1}{H} \right)  \sum_{u=1}^{\nsyn(K)} \tau_u \sum_{s \in \cS}  d_{h+1}^{\pi^{\star}}(s)  ( V^{\star}_{h+1} (s)-V_{t_{u-1}, h+1} (s))   , 
      \end{align}
    where $
      \sigma_{\mathsf{aux}} =
       \sqrt{ \frac{H^2  K S \clipavgC}{M}  }  + \frac{H^2 S \clipavgC   }{M}  $.
  \end{lemma}

\end{itemize}

\paragraph{Step 3: Recursion.}
Combining the bounds of the decomposed errors (cf.~\eqref{eq:D1-final}, \eqref{eq:D3-final}, and \eqref{eq:vgap_recursion}), for any $h \in [H]$, we obtain the following recursive relation:
\begin{align}
    &\sum_{v=1}^{\nsyn(K)} \tau_v \sum_{s \in \cS} d_1^{\pi^{\star}}(s) \left (  V_{h}^{\star}(s) -  V_{t_v, h} (s)  \right ) \cr
    &\lesssim  \theta_{K} + \Big(1+\frac{1}{H} \Big) \sum_{u=1}^{\nsyn(K)} \tau_{u} \sum_{s \in \cS} d_1^{\pi^{\star}}(s) \left (  V_{h+1}^{\star}(s) -  V_{t_{u-1}, h+1} (s)  \right ) \cr
  &\stackrel{\mathrm{(i)}}{\lesssim} (\theta_{K} + H \tau_1 )  + \Big(1+\frac{1}{H} \Big) \sum_{u=1}^{\nsyn(K)-1} \tau_{u+1} \sum_{s \in \cS} d_1^{\pi^{\star}}(s) \left (  V_{h+1}^{\star}(s) -  V_{t_u, h+1} (s)  \right ) \cr
  &\stackrel{\mathrm{(ii)}}{\lesssim} (\theta_{K} + H \tau_1 )  + \Big(1+\frac{2}{H} \Big)^2 \sum_{u=1}^{\nsyn(K)-1} \tau_{u} \sum_{s \in \cS} d_1^{\pi^{\star}}(s) \left (  V_{h+1}^{\star}(s) -  V_{t_u, h+1} (s)  \right ) ,
\end{align}
where (i) holds because $V_{h+1}^{\star}(s) - V_{t_u, h+1} (s) \le H$ and (ii) holds due to the condition $\frac{\tau_{u+1}}{\tau_u} \le 1+\frac{2}{H}$ for all $1 \le u \le \nsyn(K)$ and the fact that $V_{h+1}^{\star}(s) \ge V_{t_u, h+1} (s)$ shown in Lemma~\ref{lemma:pessimistic_value}, and we denote
\begin{align}
  \theta_{k}
  &\defn
    \frac{H \clipavgC S}{M}
    + \frac{H^2 \clipavgC S}{M} 
  + \sqrt{\frac{H^4 S \clipavgC k}{M}}
    + \sqrt{ \frac{H^2  k S \clipavgC}{M}  }  + \frac{ H^2 S \clipavgC}{M}
\end{align}
for any $k \in [K]$. Then, by invoking the recursion $(H-h+1)$ times, it follows that
\begin{align}
    &\sum_{v=1}^{\nsyn(K)} \tau_v \sum_{s \in \cS} d_1^{\pi^{\star}}(s) \left (  V_{h}^{\star}(s) -  V_{t_v, h} (s)  \right ) \cr
    &\lesssim  (\theta_{K} + H \tau_1) + \Big(1+\frac{2}{H} \Big)^2 (\theta_{t_{\nsyn(K)-1}} + H \tau_1)  + \Big(1+\frac{2}{H} \Big)^4 \sum_{u=1}^{\nsyn(K)-2} \tau_{u} \sum_{s \in \cS} d_1^{\pi^{\star}}(s) \left (  V_{h+2}^{\star}(s) -  V_{t_u, h+2} (s)  \right ) \cr
    &\lesssim  (\theta_{K} + H \tau_1) + \Big(1+\frac{2}{H} \Big)^2 (\theta_{t_{\nsyn(K)-1}} + H \tau_1)  + \cdots + \Big(1+\frac{2}{H} \Big)^{2(H-h+1)} (\theta_{t_{\nsyn(K)-H+h-1}} + H \tau_{1}) \cr
    &\lesssim H \theta_{K} + H^2 \tau_1
\end{align}
where the second line follows from the fact that $V_{H+1}^{\star}(s) -  V_{k, H+1} (s) = 0$ for any $k \in [K]$, and the last line holds because $\theta_k \le \theta_{K}$ for any $k \le K$ and $(1+\frac{2}{H})^{2(H-h+1)} \le (1+\frac{2}{H})^{2H} \le e^4$.

Finally, by plugging the above bound into \eqref{eq:optgap-to-regret}, we obtain the bound of the performance gap as follows:
\begin{align}
V_{1}^{\star}(\rho) - V_{1}^{\hat{\pi}}(\rho) 
&  \le   \frac{1}{K}  \max_{h \in [H]}  \sum_{v=1}^{\nsyn(K)} \tau_v \sum_{s \in \cS} d_h^{\pi^{\star}}(s) \left (  V_{h}^{\star}(s) -  V_{t_v, h} (s)  \right ) \cr
    &\lesssim \frac{1}{K} (H \theta_{K} + H^2 \tau_1 ) \cr
  &\lesssim 
    \frac{H^3 S \clipavgC }{MK} 
    + \sqrt{\frac{H^6 S \clipavgC }{M K}}
    + \frac{H^2 \tau_1 }{K} \stackrel{T=HK}{\lesssim}  
    \sqrt{\frac{H^7 S \clipavgC }{M T}}  + \frac{H^4 S \clipavgC }{MT}    ,
\end{align}
where the last line holds if $\tau_1 \le \sqrt{\frac{H S \clipavgC T}{M}}$, and this completes the proof.

\section{Discussions}
\label{sec:conclusions}

We investigated federated offline RL, which enables multiple agents with history datasets to collaboratively learn an optimal policy, without sharing datasets. 
We proposed a federated offline Q-learning algorithm called \pfedq, which iteratively performs local updates with rescaled learning rates at agents, and global aggregation with weighted averaging and global penalty at a server, which effectively controls the uncertainty in both local and global Q-estimates. 
Our sample complexity analysis demonstrates that \pfedq achieves linear speedup in terms of the number of agents requiring only collective coverage of agents' datasets over the distribution of the optimal policy, not restricted to the quality of individual datasets.
Furthermore, we showed that \pfedq is communication-efficient, requiring only $\widetilde{O}(H)$ synchronizations under the exponential synchronization scheduling.
For future exploration, this work paves the way for many interesting directions, some of which are outlined below.

\begin{itemize}

\item {\em Tightening $H$ dependency.}
  Although our sample complexity bound is nearly optimal with respect to most salient problem parameters, such as state space size and single-policy concentrability coefficient, it falls short of optimality in terms of horizon length compared to the minimax sample complexity lower bound in the single-agent setting \citep{xie2021policy}. Closing this gap and improving sample complexity with variance reduction techniques, as proposed by \citet{shi2022pessimistic}, will be an interesting avenue for future exploration.

\item {\em Beyond episodic tabular MDPs.}
Extending episodic tabular MDPs, it would be interesting to broaden our analysis framework to encompass other RL settings, including, the infinite-horizon setting and the integration of function approximation.

\item {\em Improving robustness.}
  Our work focuses on a scenario in which agents collect datasets from a common MDP without any disturbances. Yet, in real-world scenarios, some agents may possess datasets collected from perturbed MDPs. This introduces the need for additional considerations regarding robustness, as discussed in \citet{shi23perturb}. Therefore, enhancing our work to effectively handle the variability or noisiness of MDPs would be a compelling avenue for improvement.

\item {\em Multi-task RL.}  In many applications with multiple clients, multi-task learning, where clients have heterogeneous goals, holds significant interest due to diversity in clients. It will be of great interest to extend our work to the multi-task RL setting \citep{yang23fednpg,jin2022fedrlhete,zhou23fed}, which enables agents to learn their own optimal policies for their personalized goals while benefiting from collaboration by sharing common features of tasks.
  
\end{itemize}

\section*{Acknowledgement}

This work is supported in part by the grants NSF CCF-2007911, CCF-2106778, CNS-2148212,  ONR N00014-19-1-2404 to Y. Chi, and NSF-CCF 2007834, CCF-2045694, CNS-2112471,  ONR N00014-23-1-2149 to G. Joshi.

\bibliographystyle{apalike}
\bibliography{bibfileRL}

\appendix

\section{Technical lemmas}
\label{sec:tech_lemmas}

\paragraph{Freedman's inequality.}
We provide a user-friendly version of Freedman's inequality \citep{freedman1975tail}. See \citet[Theorem~6]{li2021syncq} for more details.

\begin{theorem}[{\citet[Theorem~6]{li2021syncq}}]\label{thm:Freedman}
Consider a filtration $\mathcal{F}_0\subset \mathcal{F}_1 \subset \mathcal{F}_2 \subset \cdots$,
and let $\mathbb{E}_{k}$ stand for the expectation conditioned
on $\mathcal{F}_k$. 
Suppose that $Y_{n}=\sum_{k=1}^{n}X_{k}\in\mathbb{R}$,
where $\{X_{k}\}$ is a real-valued scalar sequence obeying

\[
  \left|X_{k}\right|\leq R\qquad\text{and}\qquad\mathbb{E}_{k-1} \big[X_{k}\big]=0\quad\quad\quad\text{for all }k\geq1
\]
for some quantity $R<\infty$. 
We also define
\[
W_{n}\coloneqq\sum_{k=1}^{n}\mathbb{E}_{k-1}\left[X_{k}^{2}\right].
\]
In addition, suppose that $W_{n}\leq\sigma^{2}$ holds deterministically for some given quantity $\sigma^2<\infty$.
Then for any positive integer $m \geq1$, with probability at least $1-\delta$
one has
\begin{equation}
\left|Y_{n}\right|\leq\sqrt{8\max\Big\{ W_{n},\frac{\sigma^{2}}{2^{m}}\Big\}\log\frac{2m}{\delta}}+\frac{4}{3}R\log\frac{2m}{\delta}.\label{eq:Freedman-random}
\end{equation}
\end{theorem}

We next present a basic analytical result that is useful in the proof.
\begin{lemma} \label{lemma:seqsum}
        Consider any sequence $\{ x_z \}_{z=1, \cdots, Z}$ where $x_z \ge 1$ for all $z$ and let $X_z = \sum_{z'=1}^z x_{z'}$. Then, for any $Z\ge1$, it follows that
        \begin{align*}
            X(Z) = \sum_{z=1}^{Z} \frac{x_z}{X_z}\le 1 + \log{X_{Z}}.
        \end{align*}
        \end{lemma}
        \begin{proof}
        For $Z=1$, $X(1) = \frac{x_1}{x_1} = 1$.
        For $Z>1$, suppose the claim holds for $Z-1$. Then, it holds for $Z$ as follows: 
        \begin{align}
            X(Z) = X(Z-1) + \frac{x_Z}{X_Z} 
            &\le  1+ \log{X_{Z-1}} + 1- \frac{X_{Z-1}}{X_Z} \cr
            &\le  1+ \log{X_{Z-1}} - \log{\left( \frac{X_{Z-1}}{X_Z} \right) } 
            = 1+ \log{X_{Z}},
        \end{align}
where the first inequality follows from the induction hypothesis and $x_Z = X_{Z} - X_{Z-1}$, the second inequality      follows from $\log{y} \le y-1$ for any $y>0$. By induction, this completes the proof.
        \end{proof}
        
 Last but not least, we have the following useful properties regarding the parameters introduced in \eqref{eq:def-omegaik}.
\begin{lemma} \label{lemma:lr_bound}
  For any $(s,a,h) \in \cS\times\cA\times[H]$, $k' \le  k \in \synset(K)$, where we denote $u=\nsyn(k)$, and $i \in L_{k,h}^m(s,a)$. Then, it follows that:
  \begin{subequations}
  \begin{align}    
    \label{eq:lr-wmax} \omega_{i,k,h}^m(s,a) &\le \frac{2H}{N_{k,h}(s,a) +  H n_{k,h}(s,a)} ,\\
    \label{eq:lr-wsum} \sum_{m=1}^M \sum_{j \in L_{k,h}^m(s,a)} \omega_{j,k,h}^m(s,a) & \le 1 ,\\
    \label{eq:lr-wsum-partial} \sum_{m=1}^M \sum_{j \in l_{k',h}^m(s,a)} \omega_{j,k,h}^m(s,a) & \le \frac{ (H+1)n_{k', h}}{N_{k,h} +  H n_{k,h}}, \\
    \label{eq:lr-wsqsum} \sum_{m=1}^M \sum_{j \in L_{k,h}^m(s,a)} (\omega_{i,k,h}^m(s,a))^2 &\le \frac{2H}{N_{k,h}(s,a) +  H n_{k,h}(s,a)} ,\\
    \label{eq:lr-wsum-future}\sum_{v \ge u}^{\infty} n_{t_v, h}(s,a) \sum_{m=1}^M \sum_{i \in l_{k,h}^m(s,a)} \omega_{i,t_v,h}^m(s,a)
    &  \le n_{k,h}(s,a) \left(1+\frac{1}{H}\right).
    \end{align}
  \end{subequations}
\end{lemma}

\begin{proof}
    For notation simplicity, we will omit $(s,a)$ for the following proofs. Moreover, $u=\nsyn(k)$ and $t_u = k$.

    \paragraph{Proof of \eqref{eq:lr-wmax}.} 
    Recalling the definition of $\omega_{i,k,h}^m$ in \eqref{eq:def-omegaik} and using the fact that $H \ge 1$,
    \begin{align} \label{eq:lr-wmax-proof}
    \omega_{i,k,h}^m
      &= \frac{ H+1}{N_{k,h} +  H n_{k,h}}
        \left (\prod_{x=\nsyn(i)}^{\nsyn(k)-1} \frac{ N_{t_x,h}}{N_{t_x,h} + H n_{t_x,h}} \right) 
      \le \frac{2H}{N_{k,h} +  H n_{k,h}}.
    \end{align}

    \paragraph{Proof of \eqref{eq:lr-wsum}.} By rearranging the terms,
    \begin{align}
      \sum_{m=1}^M \sum_{j \in L_{k,h}^m(s,a)} \omega_{j,k,h}^m
      &= \sum_{v=1}^{\nsyn(k)} \sum_{m=1}^M \sum_{j\in l_{t_v, h}^m}\frac{ H+1}{N_{t_v,h} +  H n_{t_v,h}}
      \left (\prod_{x=v+1}^{\nsyn(k)} \frac{ N_{t_{x-1},h}}{N_{t_x,h} + H n_{t_x,h}} \right) \cr
      &= \sum_{v=1}^{\nsyn(k)} \frac{ (H+1)n_{t_v,h} }{N_{t_v,h} +  H n_{t_v,h}}
      \left (\prod_{x=v+1}^{\nsyn(k)} \frac{ N_{t_{x-1},h}}{N_{t_x,h} + H n_{t_x,h}} \right) \cr
      &= \sum_{v=1}^{\nsyn(k)} \left ( 1- \frac{ N_{t_{v-1},h} }{N_{t_v,h} +  H n_{t_v,h}} \right)
      \left (\prod_{x=v+1}^{\nsyn(k)} \frac{ N_{t_{x-1},h}}{N_{t_x,h} + H n_{t_x,h}} \right) \cr
      &= \sum_{v=1}^{\nsyn(k)} \left ( \prod_{x=v+1}^{\nsyn(k)} \frac{ N_{t_{x-1},h}}{N_{t_x,h} + H n_{t_x,h}} - \prod_{x=v}^{\nsyn(k)} \frac{ N_{t_{x-1},h}}{N_{t_x,h} + H n_{t_x,h}} \right) \cr
      &= 1 - \prod_{x=1}^{\nsyn(k)} \frac{ N_{t_{x-1},h}}{N_{t_x,h} + H n_{t_x,h}}  \le 1. 
    \end{align}

    \paragraph{Proof of \eqref{eq:lr-wsum-partial}.} Let $v = \nsyn(k')$, i.e., $k' = t_v$. Similarly to the proof of \eqref{eq:lr-wsum}, by arranging some terms, we obtain the upper bound as follows:
    \begin{align} \label{eq:lr-wsum-partial-deriv}
      \sum_{m=1}^M \sum_{j \in l_{k',h}^m(s,a)} \omega_{j,k,h}^m(s,a)
      &= \sum_{m=1}^M \sum_{j \in l_{t_v,h}^m(s,a)} 
        \frac{ H+1}{N_{t_v,h} +  H n_{t_v,h}}
      \left (\prod_{x=v+1}^{\nsyn(k)} \frac{ N_{t_{x-1},h}}{N_{t_x,h} + H n_{t_x,h}} \right) \cr
      &=  \frac{ (H+1)n_{t_v,h} }{N_{t_v,h} +  H n_{t_v,h}}
      \left (\prod_{x=v+1}^{\nsyn(k)} \frac{ N_{t_{x-1},h}}{N_{t_x,h} + H n_{t_x,h}} \right) \cr
      &= \frac{ (H+1) n_{t_v,h}}{N_{k,h} +  H n_{k,h}}
        \left (\prod_{x=v}^{\nsyn(k)-1} \frac{ N_{t_{x},h}}{N_{t_x,h} + H n_{t_x,h}} \right) \cr
      &\le \frac{ (H+1)n_{k', h}}{N_{k,h} +  H n_{k,h}}.
    \end{align}
    
    \paragraph{Proof of \eqref{eq:lr-wsqsum}.} Using the bound in \eqref{eq:lr-wmax} and \eqref{eq:lr-wsum},
    \begin{align}
      \sum_{m=1}^M \sum_{j \in L_{k,h}^m} (\omega_{j,k,h}^m)^2
      = \left ( \max_{m \in [M], j \in L_{k,h}^m} \omega_{j,k,h}^m \right) \sum_{m=1}^M \sum_{j \in L_{k,h}^m} \omega_{j,k,h}^m
      \le  \max_{m \in [M], j \in L_{k,h}^m} \omega_{j,k,h}^m  
      \le \frac{2H}{N_{k,h} +  H n_{k,h}}.
    \end{align}
    
    \paragraph{Proof of \eqref{eq:lr-wsum-future}.} Recall that $k=t_u$. Then, reusing the intermediate result derived in \eqref{eq:lr-wsum-partial-deriv},
    \begin{align}
      \sum_{v \ge u}^{\infty} n_{t_v, h}(s,a) \sum_{m=1}^M \sum_{i \in l_{t_u,h}^m(s,a)} \omega_{i,t_v,h}^m(s,a)
      &=  \sum_{v \ge u}^{\infty}n_{t_v, h} \frac{(H+1) n_{t_u,h}}{N_{t_v,h} +  H n_{t_v,h}}
        \Bigg (\prod_{x=u}^{v-1} \underbrace{ \frac{ N_{t_x,h}}{N_{t_x,h} + H n_{t_x,h}}  }_{\defn \beta_{x,h}} \Bigg) \cr
      &= (H+1) n_{t_u,h} \sum_{v \ge u}^{\infty}\frac{n_{t_v, h} }{N_{t_v,h} +  H n_{t_v,h}}
        \left (\prod_{x=u}^{v-1} \beta_{x,h} \right) \cr  
      &= (H+1) n_{t_u,h} \sum_{v \ge u}^{\infty} \frac{1}{H}(1-\beta_{v,h})
        \left (\prod_{x=u}^{v-1} \beta_{x,h} \right) \cr
      &\le n_{k,h} \left(1+\frac{1}{H} \right).
    \end{align}

\end{proof}
  
\section{Proofs for main results}
\label{sec:proofs_main}

\subsection{Proof of Lemma~\ref{lemma:qerr-dcomp-base}} \label{proof:qerr-dcomp-base}
For any $(h,s,a)\in [H]\times\cS\times\cA$ and $k \in \synset(K)$, according to the pessimistic aggregation update rule in \eqref{eq:globalq-update}, the estimate error of Q function at the $k$-th iteration can be written as follows: 
\begin{align}  \label{eq:globalq-to-localq}
    Q_{h}^{\pi}(s,a) -  Q_{k, h} (s, a) 
    &= Q_{h}^{\pi}(s,a) - \left( \sum_{m=1}^M \alpha_{k,h}^m(s,a) Q_{k, h}^m(s, a) \right) + B_{k,h}(s,a) \cr
    &= \sum_{m=1}^M \alpha_{k,h}^m(s,a) \left( Q_{h}^{\pi}(s,a)- Q_{k, h}^m(s, a)  \right) + B_{k,h}(s,a),
\end{align}
where the last equality holds by the fact $\sum_{m=1}^M \alpha_{k,h}^m(s,a) =1$.  

Then, invoking the local update rule in \eqref{eq:localq-update}, for any $i$ such that $(s_{i,h}^m,a_{i,h}^m)=(s,a)$, the local Q-estimate error at each agent $m$ can be written as follows:
\begin{align}
    &Q_{h}^{\pi}(s,a)- Q_{i, h}^m(s, a) \cr
    & = (1-\eta_{i,h}^m(s,a))(Q_{h}^{\pi}(s,a)- Q_{i-1, h}^m(s, a)) + \eta_{i,h}^m(s,a) (Q_{h}^{\pi}(s,a) - r_h(s,a) - P_{i,h}^m V_{i-1,h+1}^m ) \cr
    & = (1-\eta_{i,h}^m(s,a))(Q_{h}^{\pi}(s,a)- Q_{i-1, h}^m(s, a)) + \eta_{i,h}^m(s,a) ( r_h(s,a) + P_{h,s,a} V_{h+1}^{\pi} - r_h(s,a) - P_{i,h}^m V_{i-1,h+1}^m ) \cr
    & = (1-\eta_{i,h}^m(s,a))(Q_{h}^{\pi}(s,a)- Q_{i-1, h}^m(s, a)) \cr
    &\quad +  \eta_{i,h}^m(s,a)  P_{h,s,a} (V_{h+1}^{\pi}  -  V_{i-1,h+1}^m) +  \eta_{i,h}^m(s,a)  (P_{h,s,a}  - P_{i,h}^m) V_{i-1,h+1}^m ,
\end{align}
where the second line follows from the Bellman's equation.
Then, by invoking the relation recursively, the local Q-estimate error at each agent $m$ obeys the following relation:  
\begin{align} \label{eq:localq-recursion}
    Q_{h}^{\pi}(s,a)- Q_{k, h}^m(s, a) 
    & = \prod_{i \in l_{k,h}^m(s,a) } (1-\eta_{i,h}^m(s,a)) \left( Q_{h}^{\pi}(s,a)- Q_{\syn(k), h}(s, a) \right) \cr
    &+ \sum_{i \in l_{k,h}^m(s,a)} \eta_{i,h}^m(s,a) \prod_{ \{ j > i: j \in l_{k,h}^m(s,a) \}} (1-\eta_{j,h}^m(s,a)) P_{h,s,a} (V_{h+1}^{\pi}-V^{m}_{i-1, h+1}) \cr
    &+ \sum_{i \in l_{k,h}^m(s,a)} \eta_{i,h}^m(s,a) \prod_{ \{ j > i: j \in l_{k,h}^m(s,a) \}} (1-\eta_{j,h}^m(s,a)) (P_{h,s,a} - P_{i,h}^m) V^{m}_{i-1, h+1} ,
\end{align}
where $l_{k,h}^m(s,a)$ denotes a set of episodes where agent $m$ has visited $(s,a)$ at step $h$ within $(\syn(k), k]$.

By inserting \eqref{eq:localq-recursion} to \eqref{eq:globalq-to-localq} and letting $v = \nsyn(k)$, we obtain the following recursive relation for $u$-th local updates:
\begin{align}
  &Q_{h}^{\pi}(s,a) -  Q_{k, h} (s, a) \cr
    & =  \underbrace{\left (\sum_{m=1}^M \alpha_{k,h}^m(s,a) \prod_{i \in l_{k,h}^m(s,a) } (1-\eta_{i,h}^m(s,a)) \right) }_{\defn \lambda_{v,h}(s,a)} \left( Q_{h}^{\pi}(s,a)- Q_{\syn(k), h}(s, a) \right) + B_{k,h}(s,a) \cr
    &\quad+ \sum_{m=1}^M \sum_{i \in l_{k,h}^m(s,a)} \left ( \alpha_{k,h}^m(s,a)  \eta_{i,h}^m(s,a) \prod_{ \{ j > i: j \in l_{k,h}^m(s,a) \}} (1-\eta_{j,h}^m(s,a)) \right) P_{h,s,a} (V_{h+1}^{\pi}-V^{m}_{i-1, h+1}) \cr
  &\quad +  \sum_{m=1}^M  \sum_{i \in l_{k,h}^m(s,a)} \left (  \alpha_{k,h}^m(s,a) \eta_{i,h}^m(s,a) \prod_{ \{ j > i: j \in l_{k,h}^m(s,a) \}} (1-\eta_{j,h}^m(s,a)) \right) (P_{h,s,a} - P_{i,h}^m) V^{m}_{i-1, h+1} \cr
    & = \lambda_{v,h}(s,a) \left( Q_{h}^{\pi}(s,a)- Q_{\syn(k), h}(s, a) \right) + B_{k,h}(s,a) \cr
    &\quad + \frac{ (H+1)}{N_{t_v,h}(s,a) +  H n_{t_v,h}(s,a)}  \sum_{m=1}^M \sum_{i \in l_{k,h}^m(s,a)} P_{h,s,a} (V_{h+1}^{\pi}-V^{m}_{i-1, h+1}) \cr
      &\quad + \frac{ (H+1)}{N_{t_v,h}(s,a) +  H n_{t_v,h}(s,a)}  \sum_{m=1}^M \sum_{i \in l_{k,h}^m(s,a)}  (P_{h,s,a} - P_{i,h}^m) V^{m}_{i-1, h+1}.
\end{align}
Here, the last line holds by invoking the definitions in \eqref{eq:alpha-def} and \eqref{eq:eta-def} and observing with abuse of notation (omit $(s,a)$ when it is clear) 
\begin{align}
  &  \alpha_{k,h}^m(s,a) \eta_{i,h}^m(s,a) \prod_{ \{ j > i: j \in l_{k,h}^m(s,a) \}} (1-\eta_{j,h}^m(s,a))  \cr
  & =  \frac{1}{M}\frac{ N_{\syn(k),h} + M (H+1) n_{k,h}^m}{N_{k,h} + H n_{k,h}} \frac{M(H+1)}{N_{\syn(i),h} + M (H+1) n_{i,h}^m}  
    \left (\prod_{j=1}^{n_{k,h}^m-n_{i,h}^m} \Big(\frac{N_{\syn(i),h} + M (H+1) (n_{i,h}^m+j-1)}{N_{\syn(i),h} + M (H+1) (n_{i,h}^m+j)} \Big) \right) \cr
  & =  \frac{1}{M}\frac{ N_{\syn(k),h} + M (H+1) n_{k,h}^m}{N_{k,h} + H n_{k,h}} \frac{M(H+1)}{N_{\syn(i),h} + M (H+1) n_{i,h}^m}
    \frac{N_{\syn(i),h} + M (H+1) n_{i,h}^m}{N_{\syn(i),h} + M (H+1) n_{k,h}^m} \cr
    &=\frac{(H+1)}{N_{k,h} + H n_{k,h}} =  \frac{(H+1)}{N_{t_v,h} + H n_{t_v,h}} 
\end{align}
where the last line holds since $\syn(i)= \syn(k)$ for $i \in l_{k,h}^m(s,a)$ and $k\in \cT(K)$ leads to $k = t_{\phi(k)} = t_v$.

Then, by invoking the above recursive relation for each aggregation, the Q-estimate error after $k$ episodes is decomposed as follows:
\begin{align}
    &Q_{h}^{\pi}(s,a) -  Q_{k, h} (s, a) \cr
  &=    \underbrace{\prod_{u=1}^{\nsyn(k)} \lambda_{u,h}(s,a) }_{\defn \omega_{0,k,h}(s,a)}(Q_{h}^{\pi}(s,a) -  Q_{0, h} (s, a))
    + \sum_{u=1}^{\nsyn(k)}  B_{t_u,h}(s,a) \prod_{x=u+1}^{\nsyn(k)} \lambda_{x,h}(s,a)  \cr
    &\qquad +   \sum_{u=1}^{\nsyn(k)}  \sum_{m=1}^M   \sum_{i \in l_{t_u,h}^m(s,a)}  \underbrace{ \left (\frac{H+1}{N_{t_u,h} + H n_{t_u,h}}   \prod_{x=u+1}^{\nsyn(k)}  \lambda_{x,h}(s,a) \right)}_{\defn \omega_{i,k,h}(s,a)}(P_{h,s,a} - P_{i,h}^m) V^{m}_{i-1, h+1} \cr
    &\qquad +  \sum_{u=1}^{\nsyn(k)} \sum_{m=1}^M  \sum_{i \in l_{t_u,h}^m(s,a)}  \left (\frac{H+1}{N_{t_u,h} + H n_{t_u,h}}   \prod_{x=u+1}^{\nsyn(k)}  \lambda_{x,h}(s,a) \right)  P_{h,s,a} (V_{h+1}^{\pi}-V^{m}_{i-1, h+1})      \cr
    &=  \omega_{0,k,h}(s,a) (Q_{h}^{\pi}(s,a) -  Q_{0, h} (s, a)) \cr
    &\qquad + \sum_{m=1}^M \sum_{i \in L_{k,h}^m(s,a)} \omega_{i,k,h}^m(s,a) (P_{h,s,a} - P_{i,h}^m) V^{m}_{i-1, h+1}  \cr
    &\qquad+  \sum_{u=1}^{\nsyn(k)}  B_{t_u,h}(s,a) \prod_{x=u+1}^{\nsyn(k)} \lambda_{x,h}(s,a)   \cr
  &\qquad +  \sum_{m=1}^M \sum_{i \in L_{k,h}^m(s,a)} \omega_{i,k,h}^m(s,a) P_{h,s,a} (V_{h+1}^{\pi}-V^{m}_{i-1, h+1}) .
\end{align}
Here, $\lambda_{u,h}(s,a)$, $\omega_{0,k,h}(s,a)$, and $\omega_{i,k,h}(s,a)$ can be simply written as described in \eqref{eq:lr-lambda}, \eqref{eq:def-omega0k}, and \eqref{eq:def-omegaik}, respectively, which will be proved momentarily. For notational simplicity, we omit $(s,a)$ in the derivations.

\paragraph{Proof of \eqref{eq:lr-lambda}.} Consider $k = t_v$. First, consider a case that $N_{\syn(k),h}=0$. 
    If $n_{k,h} = 0$, $\lambda_{v,h}  = \sum_{m=1}^M \alpha_{k,h}^m = 1$. Otherwise, if $n_{k,h} > 0$, where there exists at least one agent $m \in [M]$ that visits the state-action at least once until $k$-th episode, it follows that
    \begin{align}    
      \lambda_{v,h}
      &= \sum_{m=1}^M \frac{1}{M} \frac{  (H+1) M n_{k,h}^m}{ (H+1)  n_{k,h}}  \prod_{j=1}^{n_{k,h}^m} \left(\frac{ M (H+1) (j-1)}{ M (H+1) j} \right) \cr
      &= \sum_{m\in [M]: n_{k,h}^m = 0}^M \underbrace{\frac{  n_{k,h}^m}{ n_{k,h}}}_{=0}   + \sum_{m\in [M]: n_{k,h}^m>0}^M \frac{  n_{k,h}^m}{ n_{k,h}}  \underbrace{  \prod_{j=1}^{n_{k,h}^m} \left(\frac{ (H+1) (j-1)}{(H+1) j} \right) }_{=0 } 
      = 0.
    \end{align}
    On the other hand, when $N_{\syn(k),h}>0$,
    \begin{align}    
      \lambda_{v,h}
      &= \sum_{m=1}^M \frac{1}{M} \frac{ N_{\syn(k),h} + M (H+1) n_{k,h}^m}{N_{\syn(k),h} + (H+1) n_{k,h}}  \prod_{j=1}^{n_{k,h}^m} \left(\frac{N_{\syn(k),h} + M (H+1) (j-1)}{N_{\syn(k),h} + M (H+1) j} \right) \cr
      &= \sum_{m=1}^M \frac{1}{M} \frac{ N_{\syn(k),h} + M (H+1) n_{k,h}^m}{N_{\syn(k),h} + (H+1) n_{k,h}}  \frac{N_{\syn(k),h} }{N_{\syn(k),h} + M (H+1)n_{k,h}^m} 
      =  \frac{ N_{\syn(k),h}}{N_{k,h} + H n_{k,h}}  .
    \end{align}

\paragraph{Proof of \eqref{eq:def-omega0k}.} According to \eqref{eq:lr-lambda}, if $N_{k,h}(s,a)=0$, then $\lambda_{u,h}(s,a) = 1$ for all $1 \le u \le \nsyn(k)$. Thus, $\omega_{0,k,h}(s,a) =1 $. Otherwise, let the epsiode when $(s,a)$ is visited at step $h$ by any of the agents for the first time be $j$. Then, $\lambda_{\nsyn(j),h} = 0$ because $N_{\syn(j),h}(s,a)=0$. Thus, if $N_{k,h}(s,a)>0$, it always holds that $\omega_{0,k,h}(s,a)=\prod_{u=1}^{\nsyn(k)}\lambda_{u,h}(s,a) = 0$.

\paragraph{Proof of \eqref{eq:def-omegaik}.} For $i$ such that $\nsyn(i) = u$, by rearranging terms and applying \eqref{eq:lr-lambda},
    \begin{align} 
    \omega_{i,k,h}^m
      &= \frac{ (H+1)}{N_{t_{u},h} +  H n_{t_{u},h}}
        \left (\prod_{x=u+1}^{\nsyn(k)} \frac{ N_{t_{x-1},h}}{N_{t_x,h} + H n_{t_x,h}} \right) \cr
      &= \frac{ H+1}{N_{k,h} +  H n_{k,h}}
        \left (\prod_{x=u}^{\nsyn(k)-1} \frac{ N_{t_x,h}}{N_{t_x,h} + H n_{t_x,h}} \right). 
    \end{align}

\subsection{Proof of Lemma~\ref{lemma:avgvisit_bound}} \label{proof:avgvisit_bound}

Consider any given $\delta \in (0,1)$ and $(k,s,a,h) \in [K]\times\cS\times\cA\times[H]$. Note that $N_{k,h}^m(s,a) \sim \mathsf{Binomial}(k,d_h^{m}(s,a))$ for all $m\in[M]$. Then recall the definition of $N_{k,h}(s,a)$ in Section~\ref{sec:key-parameters}, we can view $N_{k,h}(s,a) = \sum_{m=1}^M N_{k,h}^m(s,a)$ as a sum of $kM$ independent Bernoulli variables with expectation $\nu \defn \mexp[N_{k,h}(s,a)] = k M d_h^{\mathsf{avg}}(s,a)$.
Therefore, applying Chernoff bound (see \citet[Theorem 4.4]{mitzenmacher2005probability}) yields:
\begin{subequations} \label{eq:chernoff-simple}
\begin{align}
    \forall t \in [0,1] \quad &: \quad \mprob(\left| N_{k,h}^m(s,a) -\nu\right| \ge \nu t) \le \exp\left( -c_1\nu t^2   \right), \label{eq:chernoff-simple-1}\\
  \forall t \ge 1 \quad &: \quad \mprob(N_{k,h}^m(s,a) - \nu \ge t\nu) \le \exp\left( -c_1 \nu t \right) \label{eq:chernoff-simple-3},
\end{align}
for some universal constant $c_1 >0$.
\end{subequations}

Armed with above facts and notations, now we are ready to prove \eqref{eq:banana}.
First, applying \eqref{eq:chernoff-simple-1} with $t=\frac{1}{2}$, we arrive at:
  \begin{align}\label{eq:lemma2-result1}
    \mprob \left(\left| N_{k,h}^m(s,a) -\nu\right| \ge \frac{\nu}{2} \right) \le \exp\left(-\frac{c_1 \nu}{4}\right) \le \delta, 
  \end{align}
  where the last line follows from the condition that $\nu = k M d_h^{\mathsf{avg}}(s,a) \ge \frac{4}{c_1}  \log{(\frac{1}{\delta})}$.

To continue, when $\nu = k M d_h^{\mathsf{avg}}(s,a) \le \frac{4}{c_1}  \log{(1/\delta)}$, applying \eqref{eq:chernoff-simple-3} with $t=\frac{ 4 \log{(1/\delta)}}{\nu c_1} \ge 1$ gives: 
  \begin{align}\label{eq:lemma2-result2}
    \mprob \left(N_{k,h}^m(s,a) - \nu \ge \frac{4\log{(1/\delta)}}{c_1} \right) \le \exp(-4 \log{(1/\delta)}) \le \delta. 
  \end{align}

  Summing up \eqref{eq:lemma2-result1} and \eqref{eq:lemma2-result2} and taking the union bound  over $(k,s,a,h) \in [K]\times\cS\times\cA\times[H]$ complete the proof by showing that:
  \begin{align*}
   \text{when } k \ge \frac{4 }{c_1 Md_h^{\mathsf{avg}}}\log\left(\frac{|\cS||\cA| K H}{\delta}\right)~: &\quad \frac{kM d_h^{\mathsf{avg}}}{2 } = \frac{\nu}{2} \le N_{k,h}^m(s,a) \le \frac{3\nu}{2} \leq  2 kM d_h^{\mathsf{avg}} ,\\
   \text{when } k \le \frac{4 }{c_1Md_h^{\mathsf{avg}}}\log\left(\frac{|\cS||\cA| K H}{\delta}\right) ~: &\quad N_{k,h}^m(s,a) \le \frac{8}{c_1}\log\left(\frac{|\cS||\cA| K H}{\delta}\right)
  \end{align*}
  holds with probability at least $1-2\delta$.

\subsection{Proof of Lemma~\ref{lemma:pessimistic_value}} \label{proof:pessimistic_value}

\subsubsection{Proof of \eqref{eq:var_penalty_dom}}
Noticing that the \eqref{eq:var_penalty_dom} involves two terms of interest, and we start with bounding $D_2(s,a,k,h)$.
For any $(s,a,h) \in \cS \times \cA \times [H]$ and any $k \in \synset(K)$, we can rewrite $D_2(s,a,k,h)$ as
\begin{align}
    D_2(s,a,k,h)
    &=  \sum_{i=1}^k \sum_{m=1}^M X_{i,k,h}^m(s,a),
\end{align}
where $X_{i,k,h}^m(s,a) = \omega_{i,k,h}^m(s,a) (P_{h,s,a} - P_{i,h}^m) V^{m}_{i-1, h+1} \one\{(s_{i,h}^m, a_{i,h}^m) = (s,a)\}$. 
To continue, we first introduce Lemma~\ref{lemma:d2-bound}, whose proof is provided in Appendix~\ref{proof:d2-bound}.
   \begin{lemma} \label{lemma:d2-bound}
        For any $(k,s,a,h) \in \cS \times \cA \times [H]$ and $N \in [1, MK]$, let 
        \begin{align}
            \widetilde{X}_{i,k,h}^m(s,a;N) = \widetilde{\omega}_{i,k,h}^m(s,a;N) (P_{h,s,a} - P_{i,h}^m) V^{m}_{i-1, h+1} \one\{(s_{i,h}^m, a_{i,h}^m) = (s,a)\},
        \end{align}
        where 
        \begin{align}
            \widetilde{\omega}_{i,k,h}^m(s,a;N) \defn \frac{ H+1}{N +  H n_{k,h}(s,a)}
          \left (\prod_{x=\nsyn(i)}^{\nsyn(k)-1} \frac{ N_{t_x,h}(s,a)}{N_{t_x,h}(s,a) + H n_{t_x,h}(s,a)} \right)  I_{i,h}^m(s,a; N),
        \end{align}
        and $I_{i,h}^m(s,a; N) \defn \one \{ \sum_{m'=1}^M N_{i-1,h}^{m'}(s,a) + \sum_{m'=1}^m \one \{ (s_{i,h}^{m'},a_{i,h}^{m'}) = (s,a) \} \le N  \}$.
         Then, for any $\delta \in (0,1)$, the following holds:
         \begin{align}
            \left |\sum_{i=1}^k \sum_{m=1}^M \widetilde{X}_{i,k,h}^m(s,a;N) \right | 
            \le \sqrt{\frac{81 H^4 \logpartm^2 }{N}} 
        \end{align}
        at least with probability $1-\delta$, where we denote $\logpartm = \log\left(\frac{ |\cS||\cA| M K^2 H}{\delta} \right)$.
    \end{lemma}
Armed with the above lemma, for any $(s,a,k,h) \in \cS \times \cA \times [K] \times [H]$ where $k \in \synset(K)$, the following holds by setting $N = N_{k,h}(s,a)$: 
\begin{align} \label{eq:d2-bound-N}
    \text{when } N_{k,h}(s,a)>0: \quad |D_2(s,a,k,h)| 
    \le \left |\sum_{i=1}^k \sum_{m=1}^M \widetilde{X}_{i,k,h}^m(s,a;N_{k,h}(s,a)) \right | 
    \le \sqrt{\frac{81 H^4 \logpartm^2 }{N_{k,h}(s,a)}} 
\end{align}
with probability at least $1-\delta$. As it is obvious that $D_2(s,a,k,h) = 0$ when $N_{k,h}(s,a)=0$ from the definition of $D_2(s,a,k,h)$, we arrive at
\begin{align} \label{eq:d2-bound-N-all}
     |D_2(s,a,k,h)| 
    \le \left |\sum_{i=1}^k \sum_{m=1}^M \widetilde{X}_{i,k,h}^m(s,a;N_{k,h}(s,a)) \right | 
    \le \sqrt{\frac{81 H^4 \logpartm^2 }{N_{k,h}(s,a)}} .
\end{align}

Finally, combining the results for $D_2(s,a,k,h)$ (cf. \eqref{eq:d2-bound-N-all}) and $D_3(s,a,k,h)$ (cf. \eqref{eq:d3-bound-N} in Lemma~\ref{lemma:d3-bound}), we conclude that for any $(s,a,k,h) \in \cS \times \cA \times [K] \times [H]$ with $k \in \synset(K)$, it holds with probability at least  $1-\delta$ that
\begin{align}
    | D_2(s,a,k,h) | 
    \le \sqrt{\frac{81 H^4 \logpartm^2 }{N_{k,h}(s,a)}} 
    = \sqrt{\frac{ c_B \logpartm^2 H^4 }{N_{k,h}(s,a)}}
    \le D_3 (s,a,k,h).
\end{align}

\subsubsection{Proof of \eqref{eq:Q_lower} and \eqref{eq:V_lower}}
For all $(h,s,a, k)\in [H] \times \cS \times \cA \times \synset(K)$, it is clear that $Q_{h}^{\pi_{k}}(s,a) \le  Q_{h}^{\star}(s,a)$ and $V_{h}^{\pi_{k}}(s) \le  V_{h}^{\star}(s)$ by definition.
Hence, it suffices to show that
\begin{align*}
    Q_{k, h} (s,a) \le Q_{h}^{\pi_{k}}(s,a) \quad\text{and}\quad
    V_{k, h} (s) \le V_{h}^{\pi_{k}}(s)
\end{align*}
for all $(h,s,a, k)\in [H] \times \cS \times \cA \times \synset(K)$, which we will prove by an induction argument as below. 
\begin{itemize}
    \item \textbf{Base case. } When $h=H+1$, for all $(s,a, k)\in\cS\times\cA\times \synset(K)$, the relation always holds since $Q_{k, H+1} (s,a) = 0 \le Q_{H+1}^{\pi_{k}}(s,a)$ and $V_{k, H+1} (s) = 0 \le V_{H+1}^{\pi_{k}}(s)$ according to the definition of $Q_{k, H+1}$ and $V_{k, H+1}$, respectively.
    \item \textbf{Induction. } When $h \in[H]$, suppose the relation holds for $h+1$, i.e., $Q_{k, h+1} (s,a) \le Q_{h+1}^{\pi_{k}}(s,a)$ and $V_{k, h+1} (s) \le V_{h+1}^{\pi_{k}}(s)$ for all $(s,a, k)\in\cS\times\cA\times \synset(K)$. First, we will verify the Q-estimates at step $h$ are pessimistic. For any $(s,a,k) \in \cS \times \cA\times \synset(K)$, applying Lemma~\ref{lemma:qerr-dcomp-base},
\begin{align} \label{eq:q-pessimistic-decomp}
    Q_{h}^{\pi_{k}}(s, a) - Q_{k, h} (s, a)
    &= D_1^{\pi_k}(s,a,k,h) + D_2(s,a,k,h) +D_3(s,a,k,h) +D_4^{\pi_k}(s,a,k,h).
\end{align}
We control the above four terms one at a time.
Here, $D_1^{\pi_k}(s,a,k,h) \ge 0$ since $Q_{h}^{\pi_{k}}(s, a) \ge  Q_{0, h} (s, a) = 0$.
In addition, according to \eqref{eq:var_penalty_dom}, $|D_2(s,a,k,h)| \le D_3(s,a,k,h)$.
And it is clear that $D_{4} \ge 0$ due to
\begin{align} \label{eq:pessmistic-value-monotonic}
    V^{\pi_{k}}_{h+1}\ge V_{k,h+1} \ge V_{\syn(i), h+1},
\end{align}
where the first inequality holds by the induction assumption, and the last inequality arises from the monotonicity of the global value update in \eqref{eq:vglobal-monotone}. 
Therefore, it is clear that for any $(s,a, k)\in\cS\times\cA\times \synset(K)$, the Q-estimates at step $h$ are pessimistic, i.e., 
\begin{align}
    Q_{h}^{\pi_{k}}(s, a) - Q_{k, h} (s, a) \ge 0.
\end{align}
Next, to show that value estimates at step $h$ are pessimistic,
    recalling the global update in \eqref{eq:vglobal-monotone},
    \begin{align}\label{eq:value-to-Q}
        V_{h}^{\pi_{k}}(s) - V_{k, h} (s)  
        &=  Q_{h}^{\pi_{k}}(s, \pi_{k,h}(s)) - \max\{\max_{a}Q_{k, h} (s, a), V_{\syn(k), h} (s) \} \cr 
        &=  Q_{h}^{\pi_{k}}(s, \pi_{k,h}(s)) - \max_{a} Q_{k_0, h} (s, a) \cr
        &=  Q_{h}^{\pi_{k}}(s, \pi_{k_0,h}(s)) -  Q_{k_0, h} (s, \pi_{k_0,h}(s)) \ge 0,
    \end{align}
    where $k_0$ denotes the most recent episode satisfying $V_{k, h} (s) = \max_{a} Q_{k_0, h} (s, a)$ and $k \ge k_0 \in \synset(K)$, and the last inequality holds because $\pi_{k,h}(s) = \pi_{k_0,h}(s)$ and $Q_{h}^{\pi_{k}}(s, a) -  Q_{k_0, h} (s, a) \ge 0$ can be similarly verified using \eqref{eq:q-pessimistic-decomp} and \eqref{eq:pessmistic-value-monotonic} for $k_0$. Now, we verify that $Q_{h}^{\pi_{k}}(s, a) \ge Q_{k, h} (s, a)$ and $V_{h}^{\pi_{k}}(s) \ge V_{k, h} (s)$ holds at step $h$ for any $(s,a, k)\in\cS\times\cA\times \synset(K)$, and this directly completes the induction argument.
\end{itemize}

\subsubsection{Proof of Lemma~\ref{lemma:d2-bound}} \label{proof:d2-bound}

To begin with, for any time step $h\in [H]$, we denote the expectation conditioned on the trajectories $j\leq i$ of all agent as 
\begin{align}
\forall (i,m)\in [k] \times [M]: \quad \mexp_{(i,m)} [ \cdot ]= \mexp \big[\cdot \mymid \big\{ s_{j,h}^{m'},a_{j,h}^{m'}, V_{j,h+1}^m \big\}_{j<i, m' \in [M]} , ~ \big\{ s_{i,h}^{m'},a_{i,h}^{m'} \big\}_{m' \le m} \big]. 
\end{align}
Armed with this notation, fixing $N$, it is easily verified that $\mexp_{(i,m)}[\widetilde{X}_{i,k}^m(s,a;N)] = 0$ since then $V_{i-1,h+1}^m$ can be regarded as fixed and $(P_{h,s,a} - P_{i,h}^m)$ is independent from $\widetilde{\omega}_{i,k,h}^m(s,a;N)$.

Consequently, we can apply Freedman's inequality (see the user-friendly version  provided in Theorem~\ref{thm:Freedman}) and control the term of interest for any $(s,a, k, h) \in \cS \times \cA \times [K] \times [H]$ and $N \in [1, MK]$ as below: 
        \begin{align} 
           \sum_{i=1}^k \sum_{m=1}^M \widetilde{X}_{i,k,h}^m(s,a;N)
          & \overset{\mathrm{(i)}}{\le} \sqrt{8 B_1 \logpartm } + \frac{4}{3} B_2  \logpartm
            \overset{\mathrm{(ii)}}{\le} \sqrt{\frac{32 H^4 \logpartm}{N}} + \frac{3 H^2 \logpartm}{N } 
            \le \sqrt{\frac{81 H^4 \logpartm^2 }{N}} 
        \end{align}
        at least with probability $1-\delta$. Here, (i) and (ii) arises from the following definition and facts about $B_1$ and $B_2$:
         \begin{align}
          \label{eq:e2-freedman-w} B_1 &\defn \sum_{i=1}^k \sum_{m=1}^M \mexp_{(i,m)}\left[\left(\widetilde{X}_{i,k,h}^m(s,a;N) \right)^2 \right]
                                       \le \frac{4H^4}{N},\\
        \label{eq:e2-freedman-b} B_2 &\defn \max_{(i,m) \in [k] \times [M]} \left | \widetilde{X}_{i,k,h}^m(s,a;N) \right | 
              \le \frac{2H^2}{N}  
        \end{align}
where the proofs of \eqref{eq:e2-freedman-w} and \eqref{eq:e2-freedman-b} are provided as below, respectively.

  \paragraph{Proof of \eqref{eq:e2-freedman-w}.} In view of that the events happen at any time step $h$ are independent from the transitions in later time steps including $P_{i,h}^m$, we have $\widetilde{\omega}_{i,k,h}^m(s,a;N)$ is independent from $(P_{h,s,a} - P_{i,h}^m) V^{m}_{i-1, h+1}$, which yields
        \begin{align}
          \sum_{i=1}^k \sum_{m=1}^M \mexp_{(i,m)}[(\widetilde{X}_{i,k,h}^m(s,a;N))^2]
          &= \sum_{i=1}^k \sum_{m=1}^M \mexp_{(i,m)}[(\widetilde{\omega}_{i,k,h}^m(s,a;N))^2] \mathsf{Var}_{P_{h,s,a}}(V^{m}_{i-1, h+1})\cr
          &\le H^2 \sum_{i=1}^k \sum_{m=1}^M \mexp_{(i,m)}[(\widetilde{\omega}_{i,k,h}^m(s,a;N))^2]  \cr
          &\le H^2 N \left (\frac{2H}{N} \right)^2  = \frac{4H^4}{N}, 
        \end{align}
        where the penultimate inequality holds by the fact that $|\widetilde{\omega}_{i,k,h}^m(s,a;N)| \le \frac{2H}{N}$.
        
\paragraph{Proof of \eqref{eq:e2-freedman-b}.}
        For any $(i,m,h) \in [k] \times [M] \times[H]$ and fixed $N \in [1, MK]$, it is observed that
        \begin{align}
          \left | \widetilde{X}_{i,k,h}^m(s,a;N) \right | & = \left| \widetilde{\omega}_{i,k,h}^m(s,a;N) (P_{h,s,a} - P_{i,h}^m) V^{m}_{i-1, h+1} \one\{(s_{i,h}^m, a_{i,h}^m) = (s,a)\}\right| \notag \\ 
          & \le |\widetilde{\omega}_{i,k,h}^m(s,a;N)| \cdot  \|P_{h,s,a} - P_{i,h}^m \|_{1} \cdot \|V^{m}_{i-1, h+1}\|_{\infty}  
          \le  \frac{2H^2}{N } ,
        \end{align}
       where the last inequality follows from  $\|V^{m}_{i-1, h+1}\|_{\infty} \le H$, $\|P_{h,s,a} - P_{i,h}^m \|_{1} \le 1$, and $|\widetilde{\omega}_{i,k,h}^m(s,a;N)|  \leq \frac{2H}{N}$.

\subsection{Proof of Lemma~\ref{lemma:d3-bound}} \label{proof:d3-bound} 
With slight abuse of notation, we will omit $(s,a)$ from some notation when it is clear from the context for simplicity in this proof.
 Recall the definition of $D_3(s,a,k,h)$ in \eqref{eq:qerror-decomp-base} and the global penalty defined in \eqref{eq:def-global-penalty}.
When $N_{k,h}(s,a)=0$, the global penalties are all $0$, which yields $D_3(s,a,k,h) = 0$.
Therefore, it suffices to focus on the case when $N_{k,h}(s,a)>0$ and show that for $c_B = 81$, $c_u =4$ and $c_l=1$, 
        \begin{align} \label{eq:e3-numvisit-bound} 
          D_3(s,a,k,h) = \sum_{u=1}^{\nsyn(k)} B_{t_u,h}(s,a) \prod_{u'=u+1}^{\nsyn(k)} \lambda_{u',h}(s,a)
          \in  \left [\sqrt{\frac{c_l c_B \logpartm^2 H^4 }{N_{k,h}(s,a)}} , \sqrt{\frac{c_u  c_B \logpartm^2 H^4 }{N_{k,h}(s,a)}}  \right] .
        \end{align}

Towards this, for any $(s,a)\in\cS \times \cA$, we consider a more general term as below: for any integer $z\ge 1$,
\begin{align}
    \sum_{u=1}^{z} B_{t_u,h} \prod_{u'=u+1}^{z} \lambda_{u',h} & =   \sum_{u=1}^{z} \frac{ (H+1)n_{t_u,h} }{N_{k,h}+Hn_{t_u,h}}  \sqrt{\frac{c_B \logpartm^2 H^4}{N_{t_u,h}}}  \prod_{u'=u+1}^{z} \lambda_{u',h} \notag \\
    & = \sqrt{c_B \logpartm^2 H^4} \sum_{u=1}^{z}  \sqrt{\frac{1}{N_{t_u,h}}} (1-\lambda_{u,h}) \prod_{u'=u+1}^{z} \lambda_{u',h} \notag \\
    & = \sqrt{c_B \logpartm^2 H^4} Y(z)
\end{align}
where the penultimate equality follows from 
$$\frac{ (H+1)n_{t_u,h}(s,a) }{N_{t_u,h}+Hn_{t_u,h}(s,a)} = 1-\lambda_{u,h}(s,a)$$ 
for all $(s,a)\in\cS\times \cA$, and the last equality arises by defining
\begin{align} 
          Y(z) \defn \sum_{u=1}^{z}  \sqrt{\frac{1}{N_{t_u,h}}} (1-\lambda_{u,h}) \prod_{u'=u+1}^{z} \lambda_{u',h}.
        \end{align}
As a result, to show \eqref{eq:e3-numvisit-bound}, it suffices to verify that
        \begin{align} \label{eq:e3-induction}
          Y(z) \in \left [\sqrt{\frac{c_l}{N_{t_z,h}(s,a)}} , \sqrt{\frac{c_u }{N_{t_z,h}(s,a)}}  \right],
        \end{align}
which we proceed by an induction argument. 

\paragraph{Proof of  \eqref{eq:e3-induction} by induction.} To begin with, for the basic case $z=1$, it is easily verified that  
        \begin{align}
            Y(1) 
            =\begin{cases}
                \sqrt{\frac{1}{N_{t_1,h}}} ~~&\text{if}~~ n_{t_1,h} > 0 \\
                0 ~~&\text{if}~~ n_{t_1,h} = 0
            \end{cases},
        \end{align}
since when $n_{t_1,h} > 0$ we have $\lambda_{1,h}(s,a) = 0$, and otherwise $\lambda_{1,h}(s,a) = 1$.
    Then suppose \eqref{eq:e3-induction} holds for $z-1$, namely,
    \begin{align} \label{eq:e3-induction-assumption}
         Y(z-1) \in \left [\sqrt{\frac{c_l}{N_{t_{z-1},h}}}, \sqrt{\frac{c_u }{N_{t_{z-1},h}}}  \right],
    \end{align}
    we hope to show \eqref{eq:e3-induction} holds for $z$. 
    Towards this, we first show the upper bound in \eqref{eq:e3-induction} holds for $z$ as follows:
        \begin{align}
          Y(z)
          &= Y(z-1) \lambda_{z,h} + \sqrt{\frac{1}{N_{t_z,h}}} (1-\lambda_{z,h}) \cr
          &\overset{\mathrm{(i)}}{\leq} \sqrt{\frac{c_u}{N_{t_{z-1},h}}} \frac{ N_{t_{z-1},h} }{N_{t_z,h}+ Hn_{t_z,h}} + \sqrt{\frac{1}{N_{t_z,h}}}  \frac{ (H+1)n_{t_z,h} }{N_{t_z,h}+Hn_{t_z,h}} \cr
          &\le \sqrt{\frac{c_u}{N_{t_z,h}}} \sqrt{\frac{N_{t_{z-1},h} }{N_{t_z,h}+ Hn_{t_z,h}} } + \sqrt{\frac{1}{N_{t_z,h}}}  \frac{ (H+1)n_{t_z,h} }{N_{t_z,h}+Hn_{t_z,h}} \cr
          &= \sqrt{\frac{c_u}{N_{t_z,h}}} \left ( \sqrt{\frac{ N_{t_{z-1},h} }{N_{t_z,h}+ Hn_{t_z,h}} } +  \sqrt{\frac{1}{c_u}} \frac{ (H+1)n_{t_z,h} }{N_{t_z,h}+Hn_{t_z,h}} \right) \cr
          &= \sqrt{\frac{c_u}{N_{t_z,h}}} \left ( \sqrt{\frac{ N_{t_{z-1},h} }{N_{t_z,h}+ Hn_{t_z,h}} } +  \sqrt{\frac{1}{c_u}} \left (1-\sqrt{\frac{  N_{t_{z-1},h} }{N_{t_z,h}+ Hn_{t_z,h}} } \right) \left (1+\sqrt{\frac{ N_{t_{z-1},h} }{N_{t_z,h}+ Hn_{t_z,h}} } \right ) \right) \cr
          &\le \sqrt{\frac{c_u}{N_{t_z,h}}} , \label{eq:Y-upper-bound}
        \end{align}
        where (i) follows from the induction assumption and $\frac{ (H+1)n_{t_z,h}(s,a) }{N_{t_z,h}+Hn_{t_z,h}(s,a)} = (1-\lambda_{z,h}(s,a))$ for all $(s,a)\in\cS\times \cA$, the penultimate equality holds by 
        $$1 - \frac{  N_{t_{z-1},h} }{N_{t_z,h}+ Hn_{t_z,h}} = \frac{N_{t_z,h} - N_{t_{z-1},h} + Hn_{t_z,h}}{N_{t_z,h}+ Hn_{t_z,h}} =\frac{ (H+1)n_{t_z,h} }{N_{t_z,h}+Hn_{t_z,h}},$$ and the last inequality arises from 
        $\sqrt{\frac{1}{c_u}} \left (1+\sqrt{\frac{ N_{t_{z-1},h} }{N_{t_z,h}+ Hn_{t_z,h}} } \right) \le 1$ as long as $c_u \ge 4$.
    
    Analogous to \eqref{eq:Y-upper-bound}, the lower bound of $Y(z)$ is derived as below:
        \begin{align}
          Y(z)
          &= Y(z-1) \lambda_{z,h} + \sqrt{\frac{1}{N_{t_z,h}}} (1-\lambda_{z,h}) \cr
          &\ge \sqrt{\frac{c_l}{N_{t_{z-1},h}}} \frac{ N_{t_{z-1},h} }{N_{t_z,h}+ H n_{t_z,h}} + \sqrt{\frac{1}{N_{t_z,h}}}  \frac{ (H+1)n_{t_z,h} }{N_{t_z,h}+ H n_{t_z,h}} \cr
          &\ge \sqrt{\frac{c_l}{N_{t_z,h}}} \frac{ N_{t_{z-1},h} }{N_{t_z,h}+ H n_{t_z,h}} + \sqrt{\frac{1}{N_{t_z,h}}}  \frac{ (H+1)n_{t_z,h} }{N_{t_z,h}+ H n_{t_z,h}} \cr
          &\ge \sqrt{\frac{c_l}{N_{t_z,h}}},
        \end{align}
        where the first inequality follows from the induction assumption and $\frac{ (H+1)n_{t_z,h}(s,a) }{N_{t_z,h}+Hn_{t_z,h}(s,a)} = (1-\lambda_{z,h}(s,a))$ for all $(s,a)\in\cS\times \cA$,
        and the last equality holds when $1 \ge c_l$. Finally, by induction arguments, \eqref{eq:e3-induction} holds for any $z \in \nsyn(K)$, and this completes the proof.

\subsection{Proof of Lemma~\ref{lemma:vgap_recursion}} \label{proof:vgap_recursion}
Recall the definition of $D_{4,h}$ (see \eqref{eq:short-hand-main} and \eqref{eq:qerror-decomp-base}), $D_{4,h}$ can be rewritten as follows: 
  \begin{align}
       D_{4,h}
       &  = \sum_{v=1}^{\nsyn(K)} \tau_v \sum_{(s,a) \in \cS \times \cA} d_h^{\pi^{\star}}(s, a) \sum_{m=1}^M \sum_{i \in L_{t_v,h}^m(s,a)} \omega_{i,t_v,h}^m(s,a) P_{h,s,a} (V^{\star}_{h+1}-V_{\syn(i), h+1}) \cr
      &\stackrel{\mathrm{(i)}}{=} \sum_{v=1}^{\nsyn(K)} \tau_v   \sum_{(s,a) \in \cS \times \cA} d_h^{\pi^{\star}}(s, a) \sum_{u=1}^{v} P_{h,s,a}(V^{\star}_{h+1}-V_{t_{u-1}, h+1})  \underbrace{ \sum_{m=1}^M \left (\sum_{i \in l_{t_u,h}^m(s,a)} \omega_{i,t_v,h}^m(s,a)  \right ) }_{ =:  \psi_{u,v,h}(s,a)}\cr
      & =  \sum_{(s,a) \in \cS \times \cA}  \sum_{v=1}^{\nsyn(K)} \sum_{t_{v-1} <j \le t_v} d_h^{\pi^{\star}}(s, a) \sum_{u=1}^{v}  P_{h,s,a}(V^{\star}_{h+1}-V_{t_{u-1}, h+1}) \psi_{u,v,h}(s,a)  , \label{eq:recursion-original}
    \end{align}
where (i) holds by rewriting the sum as $\sum_{i \in L_{t_v,h}^m(s,a)} = \sum_{u=1}^v \sum_{i \in l_{t_u,h}^m(s,a)}$ 
and the last equality holds by the definition of  $\tau_v$.  

To further control \eqref{eq:recursion-original}, we introduce the following lemma that bounds the expectation form \eqref{eq:recursion-original} by an empirical version; the proof is postponed to Appendix~\ref{proof:e4a-aux1}.
    \begin{lemma} \label{lemma:e4a-aux1} Consider any $\delta \in (0,1)$. For any $h \in [H],$ the following holds:
      \begin{align}\label{eq:e4a-aux1}
        & \sum_{(s,a) \in \cS \times \cA}   \sum_{v=1}^{\nsyn(K)} \sum_{t_{v-1} <j \le t_v} d_h^{\pi^{\star}}(s, a) \sum_{u=1}^{v}  P_{h,s,a}(V^{\star}_{h+1}-V_{t_{u-1}, h+1}) \psi_{u,v,h}(s,a) \cr
      &\lesssim \frac{1}{M} \sum_{(s,a) \in \cS \times \cA} \sum_{v=1}^{\nsyn(K)}    \frac{d_h^{\pi^{\star}}(s, a)}{ d^{\mathsf{avg}}_h(s,a)} n_{t_v, h}(s,a) \sum_{u=1}^{v}  P_{h,s,a}(V^{\star}_{h+1}-V_{t_{u-1}, h+1}) \psi_{u,v,h}(s,a) +  \sigma_{\mathsf{aux},1}
      \end{align}
      at least with probability $1-\delta$,
      where
      \begin{align}
        \sigma_{\mathsf{aux},1}
        \lesssim  \sqrt{ \frac{H^2  K S \clipavgC}{M}  }  + \frac{ H^2 S \clipavgC  }{M}
      \end{align}
    \end{lemma}

   Then, applying concentration bounds, $D_{4,h}$ is bounded as follows:
    \begin{align} \label{eq:d4-aux1-interm}
       D_{4,h} 
       &\stackrel{\mathrm{(i)}}{\lesssim} \frac{1}{M} \sum_{(s,a) \in \cS \times \cA} \sum_{v=1}^{\nsyn(K)} \sum_{u=1}^{v}   \frac{d_h^{\pi^{\star}}(s, a)}{ d^{\mathsf{avg}}_h(s,a)} n_{t_v, h}(s,a)   P_{h,s,a}(V^{\star}_{h+1}-V_{t_{u-1}, h+1}) \psi_{u,v,h}(s,a) + \sigma_{\mathsf{aux},1} \cr        
      &= \frac{1}{M} \sum_{(s,a) \in \cS \times \cA} \sum_{u=1}^{\nsyn(K)}     \frac{d_h^{\pi^{\star}}(s, a)}{ d^{\mathsf{avg}}_h(s,a)}    P_{h,s,a}(V^{\star}_{h+1}-V_{t_{u-1}, h+1}) \sum_{v=u}^{\nsyn(K)} n_{t_v, h}(s,a) \psi_{u,v,h}(s,a) + \sigma_{\mathsf{aux},1}\cr        
      &\stackrel{\mathrm{(ii)}}{\le}  \frac{1}{ M} \sum_{(s,a) \in \cS \times \cA} \sum_{u=1}^{\nsyn(K)}     \frac{d_h^{\pi^{\star}}(s, a)}{  d^{\mathsf{avg}}_h(s,a)}    P_{h,s,a}(V^{\star}_{h+1}-V_{t_{u-1}, h+1})  n_{t_u,h}(s,a)  \Big(1+ \frac{1}{H} \Big)  + \sigma_{\mathsf{aux},1}  
    \end{align}
    where (i) follows from Lemma~\ref{lemma:e4a-aux1}, and (ii) holds because
    \begin{align}
      \sum_{v \ge u}^{\infty} n_{t_v, h}(s,a) \sum_{m=1}^M \sum_{i \in l_{t_u,h}^m(s,a)} \omega_{i,t_v,h}^m(s,a)
        \le n_{t_u,h}(s,a) \Big(1+\frac{1}{H} \Big)
    \end{align}
    according \eqref{eq:lr-wsum-future} in Lemma~\ref{lemma:lr_bound}. 
    
To continue, we introduce the following lemma that transfers the distribution at time step $h$ to the distribution at time step $h+1$; the proof is provided in Appendix~\ref{proof:e4a-aux2}.
   \begin{lemma} \label{lemma:e4a-aux2} Consider any $\delta \in (0,1)$. For any $h \in [H],$ the following holds:
      \begin{align} \label{eq:e4a-aux2}
        &\sum_{u=1}^{\nsyn(K)} \sum_{(s,a) \in \cS \times \cA}  \frac{n_{t_u,h}(s,a)}{ M d^{\mathsf{avg}}_h(s,a)} d_h^{\pi^{\star}}(s, a)    P_{h,s,a}(V^{\star}_{h+1}-V_{t_{u-1}, h+1}) \cr            
        &\lesssim \sum_{u=1}^{\nsyn(K)} \tau_u \sum_{s \in \cS}  d_{h+1}^{\pi^{\star}}(s)  ( V^{\star}_{h+1} (s)-V_{t_{u-1}, h+1} (s)) + \sigma_{\mathsf{aux},2}
      \end{align}
      at least with probability $1-\delta$,
      where
      \begin{align*}
        \sigma_{\mathsf{aux},2} = \sqrt{\frac{H^2 K S \clipavgC }{M  }} + \frac{H S \clipavgC }{M}.
      \end{align*}
    \end{lemma}

Armed with the above lemma, rearranging the terms in \eqref{eq:d4-aux1-interm} and applying Lemma~\ref{lemma:e4a-aux2},
   \begin{align*}
       D_4       
      &\lesssim \Big(1+ \frac{1}{H}\Big)   \sum_{u=1}^{\nsyn(K)} \sum_{(s,a) \in \cS \times \cA}   \frac{n_{t_u,h}(s,a)}{M d^{\mathsf{avg}}_h(s,a)} d_h^{\pi^{\star}}(s, a)    P_{h,s,a}(V^{\star}_{h+1}-V_{t_{u-1}, h+1}) + \sigma_{\mathsf{aux},1}   \cr        
      &\lesssim \Big(1+ \frac{1}{H}\Big)  \sum_{u=1}^{\nsyn(K)} \tau_u \sum_{s \in \cS}  d_{h+1}^{\pi^{\star}}(s)  ( V^{\star}_{h+1} (s)-V_{t_{u-1}, h+1} (s)) + \underbrace{\sigma_{\mathsf{aux},1} + \sigma_{\mathsf{aux},2}}_{=: \sigma_{\mathsf{aux}}},
    \end{align*}  
    and this completes the proof.

    \subsubsection{Proof of Lemma~\ref{lemma:e4a-aux1}} \label{proof:e4a-aux1}
    Consider any given $(s,a) \in \cS \times \cA$ and $v \in [1, \nsyn(K)]$.
    Before proceeding, we introduce some notation and auxiliary terms. Let
       \begin{equation}\label{eq:defnition-G}
           G_{v,h}(s,a) \defn \sum_{u=1}^{v}  P_{h,s,a}(V^{\star}_{h+1}-V_{t_{u-1}, h+1}) \psi_{u,v,h}(s,a).
       \end{equation}
    Then, for any $t_{v-1}< j \le t_v$, we introduce the following auxiliary variables:
      \begin{align}
        Y_{j,h}^m
        &\defn \sum_{(s,a) \in \cS \times \cA}
        \left (d^{\mathsf{avg}}_h(s,a)  -  \one\{(s, a) = (s_{j,h}^m,a_{j,h}^m)\} \right ) \frac{d_h^{\pi^{\star}}(s,a)}{d^{\mathsf{avg}}_h(s,a)} G_{v,h}(s,a) \\
        \widetilde{Y}_{j,h}^m
        &\defn \sum_{(s,a) \in \cS \times \cA} 
        \left ( d_h^m(s,a)  -   \one\{(s,a) = (s_{j,h}^m,a_{j,h}^m)\} \right ) \frac{d_h^{\pi^{\star}}(s,a)}{d^{\mathsf{avg}}_h(s,a)} 
          \widetilde{G}^{-j,m}_{v,h}(s,a),
      \end{align}\label{eq:Y-defn}
      where we define
      \begin{align} \label{eq:Gapprox-recursion}
        \widetilde{G}^{-j,m}_{v,h}(s,a) 
        &\defn
          \begin{cases}
            \widetilde{\psi}^{-j,m}_{v,v,h}(s,a) P_{h,s,a}(V^{\star}_{h+1}-V_{t_{v-1}, h+1})  + (1-\widetilde{\psi}^{-j,m}_{v,v,h}(s,a))G_{v-1,h}(s,a) ~~&\text{if}~~v>1 \\
            P_{h,s,a}(V^{\star}_{h+1}-V_{0, h+1}) ~~&\text{if}~~v=1 
          \end{cases}
      \end{align}
      and
      \begin{align}\label{eq:definition-psi-new}
        \widetilde{\psi}_{v,v,h}^{-j,m}(s,a)
        &\defn
          \frac{ (H+1)(n_{t_v,h}(s,a) -  \one\{(s,a) = (s_{j,h}^m,a_{j,h}^m) \})}{N_{t_{v-1},h}(s,a) +  (H+1)(n_{t_v,h}(s,a) -  \one\{(s,a) = (s_{j,h}^m,a_{j,h}^m) \}) } \cr
        &=
          \frac{ (H+1)( \sum_{(m',j') \in [M] \times (t_{v-1}, t_v] \setminus  \{(j,m)\}}  \one\{(s,a) = (s_{j',h}^{m'},a_{j',h}^{m'}) \})}{N_{t_{v-1},h}(s,a) +  (H+1)( \sum_{(m',j') \in [M] \times (t_{v-1}, t_v] \setminus  \{(j,m)\}}  \one\{(s,a) = (s_{j',h}^{m'},a_{j',h}^{m'}) \}) } . 
      \end{align}
      We replaced $G_{v,h}(s,a)$ with a surrogate $\widetilde{G}^{-j,m}_{v,h}(s,a)$, where the visits of agent $m$ on $(s,a)$ at the $j$-th episode are masked regardless of the actual visits of agent $m$ on $(s,a)$. The surrogate is carefully designed to remove the dependency on the event $\one\{(s,a) = (s_{j,h}^m,a_{j,h}^m)\}$ from $G_{v,h}(s,a)$ while maintaining close distance to the original value $G_{v,h}(s,a)$.

     Before continuing, we introduce some useful properties of the above defined auxiliary terms whose proofs are provided in Appendix~\ref{proof:Gapprox-property}: for any $v \in [\nsyn(K)]$, 
      \begin{subequations} \label{eq:Gapprox-property}
        \begin{align}
           \label{eq:G-recursion}
          &G_{v,h}(s,a) =
          \begin{cases}
            \psi_{v,v,h}(s,a) P_{h,s,a}(V^{\star}_{h+1}-V_{t_{v-1}, h+1})  + (1-\psi_{v,v,h}(s,a))G_{v-1,h}(s,a) ~&\text{if}~ v > 1\\
            P_{h,s,a}(V^{\star}_{h+1}-V_{0, h+1}) ~&\text{if}~ v = 1\\
            \end{cases} , \\
        \label{eq:Gapprox-property-Gbound} & 0 \le \widetilde{G}^{-j,m}_{v,h}(s,a) ,~  G_{v,h}(s,a) \le H , \\ 
        \label{eq:Gapprox-property-Gapproxgap} &|\widetilde{G}^{-j,m}_{v,h}(s,a) - G_{v,h}(s,a) | \le \min \left \{H, \frac{2H^2  }{N_{t_v, h}(s,a)} \right \}.
      \end{align}
    \end{subequations}

      Now, we are ready to prove \eqref{eq:e4a-aux1}. Towards this, we first observe that moving the first term in the right-hand side of \eqref{eq:e4a-aux1} to the left-hand side, and multiplying by a factor of $M$, yields
      \begin{align}
       & \sum_{(s,a) \in \cS \times \cA}  \sum_{v=1}^{\nsyn(K)} \left( \sum_{m=1}^M \sum_{t_{v-1} <j \le t_v} d_h^{\pi^{\star}}(s, a) - \frac{d_h^{\pi^{\star}}(s, a)}{ d^{\mathsf{avg}}_h(s,a)} n_{t_v, h}(s,a) \right) \sum_{u=1}^{v}  P_{h,s,a}(V^{\star}_{h+1}-V_{t_{u-1}, h+1}) \psi_{u,v,h}(s,a) \cr
       & \overset{\mathrm{(i)}}{=} \sum_{(s,a) \in \cS \times \cA}   \sum_{v=1}^{\nsyn(K)} \left( \sum_{m=1}^M \sum_{t_{v-1} <j \le t_v} d^{\mathsf{avg}}_h(s,a) -   \sum_{m=1}^M  n^{m}_{t_v, h}(s,a) \right)  \frac{d_h^{\pi^{\star}}(s, a)}{ d^{\mathsf{avg}}_h(s,a)}  G_{v,h}(s,a) \cr
       & \overset{\mathrm{(ii)}}{=} \sum_{(s,a) \in \cS \times \cA}  \sum_{m=1}^M  \left(  \sum_{j=1}^K d^{\mathsf{avg}}_h(s,a) - \sum_{j=1}^K \one\{(s, a) = (s_{j,h}^m,a_{j,h}^m)\} \right)  \frac{d_h^{\pi^{\star}}(s, a)}{ d^{\mathsf{avg}}_h(s,a)}  G_{v,h}(s,a) \notag \\
       &= \sum_{j=1}^{K} \sum_{m=1}^M \sum_{(s,a) \in \cS \times \cA}  \left( d^{\mathsf{avg}}_h(s,a) - \one\{(s, a) = (s_{j,h}^m,a_{j,h}^m)\} \right)  \frac{d_h^{\pi^{\star}}(s, a)}{ d^{\mathsf{avg}}_h(s,a)}  G_{v,h}(s,a) = \sum_{j=1}^{K} \sum_{m=1}^M Y_{j,h}^m ,
      \end{align}
where (i) holds by plugging in \eqref{eq:defnition-G} and $ n_{t_v, h}(s,a)= \sum_{m=1}^M  n^{m}_{t_v, h}(s,a)   $, (ii) follows from $\sum_{v=1}^{\nsyn(K)} \sum_{t_{v-1} <j \le t_v} 1 
 = K $ and $\sum_{v=1}^{\nsyn(K)}n^m_{t_v, h}(s,a) =  \sum_{j=1}^K \one\{(s, a) = (s_{j,h}^m,a_{j,h}^m)\}$, and the last equality arise from the definition of $Y_{j,h}^m$ in \eqref{eq:Y-defn}. 
      
Therefore, the above fact shows that to prove  \eqref{eq:e4a-aux1}, it is suffices to show:
      \begin{align} \label{eq:e4a-aux1-decomp}
        \left | \sum_{j=1}^{K} \sum_{m=1}^M  Y_{j,h}^m \right |
        \le \left |  \sum_{j=1}^{K} \sum_{m=1}^M  \widetilde{Y}_{j,h}^m \right | 
        +  \left| \sum_{j=1}^{K} \sum_{m=1}^M  \left(Y_{j,h}^m-\widetilde{Y}_{j,h}^m \right) \right| \lesssim M \sigma_{\mathsf{aux},1}.
      \end{align}
We will control the two essential terms separately as below:
      \begin{itemize}
        \item \textbf{Controlling $\left |  \sum_{j=1}^{K} \sum_{m=1}^M  \widetilde{Y}_{j,h}^m \right |$.} To begin with, we observe  that the approximate $\widetilde{G}^{-j,m}_{v,h}(s,a)$ (defined in \eqref{eq:Gapprox-recursion}) is independent of agent $m$'s visits on $(s,a)$ at the $j$-th episode since
        $V_{t_{v-1}, h+1}$, $G_{v-1,h}(s,a)$ are independent of the $j$-th episode and  $\widetilde{\psi}^{-j,m}_{v,v,h}(s,a)$ is independent from agent $m$'s visits on $(s,a)$ at the $j$-th episode (see \eqref{eq:definition-psi-new}).
 It follows that   $\mexp_{j-1}[ \widetilde{Y}_{j,h}^m] = 0$, where we denote
          $$\mexp_{j-1} [\cdot]= \mexp\left[\cdot \mymid \{ (s_{i,h}^{m'},a_{i,h}^{m'}), V_{i,h+1}^{m'} \}_{i<j, m' \in [M]} \right].$$
    Thus,
    applying the Freedman's inequality for each $h \in [H]$, we can show that the following holds:
    \begin{align} \label{eq:Y-freedman}
        \left | \sum_{j=1}^{K} \sum_{m=1}^M  \widetilde{Y}_{j,h}^m \right |
        &\le \sqrt{ 8 W \log{\frac{2H}{\delta}} } + \frac{8}{3} B \log{\frac{2H}{\delta}} \cr
          &\lesssim  \sqrt{ H^2 M K S \clipavgC } + H S \clipavgC
      \end{align}
      at least with probability $1-\delta$, where $B$ and $W$ is obtained as follows:
    \begin{align}
      \left |\widetilde{Y}_{j,h}^m \right|
      &\le  2 \clipavgC (1+ d_h^{\pi^{\star}}(s, \pi^{\star}(s)) S) \max_{s \in \cS} \widetilde{G}^{-j,m}_{\nsyn(j),h}(s,\pi^{\star}(s)) 
        \le 4 S \clipavgC H =: B \\
      \sum_{j=1}^{K} \sum_{m=1}^M  \mexp_{j-1} \left [ \left (\widetilde{Y}_{j,h}^m \right )^2 \right ]
      &\le \sum_{j=1}^{K} \sum_{m=1}^M  \mexp_{(s_{j,h}^m,a_{j,h}^m) \sim d_{h}^m} \left [ \left ( \frac{d_h^{\pi^{\star}}(s_{j,h}^m,a_{j,h}^m) }{d^{\mathsf{avg}}_h(s_{j,h}^m,a_{j,h}^m)} \widetilde{G}^{-j,m}_{\nsyn(j),h}(s_{j,h}^m,a_{j,h}^m) \right )^2 \right ] \cr
      &\le  \sum_{j=1}^{K} \sum_{m=1}^M \sum_{s \in \cS} d_h^m(s,\pi^{\star}(s))
        \left (\frac{d_h^{\pi^{\star}}(s,\pi^{\star}(s))}{d^{\mathsf{avg}}_h(s,\pi^{\star}(s))} \widetilde{G}^{-j,m}_{\nsyn(j),h}(s,\pi^{\star}(s)) \right )^2 \cr 
      &\le  H^2 \clipavgC \sum_{j=1}^{K} \sum_{s \in \cS} \sum_{m=1}^M  d_h^m(s,\pi^{\star}(s)) \frac{d_h^{\pi^{\star}}(s,\pi^{\star}(s))}{d^{\mathsf{avg}}_h(s,\pi^{\star}(s))} (1+ d_h^{\pi^{\star}}(s, \pi^{\star}(s)) S)\cr 
      &\le  H^2 \clipavgC \sum_{j=1}^{K} \sum_{s \in \cS} M d_h^{\pi^{\star}}(s, \pi^{\star}(s)) (1+ d_h^{\pi^{\star}}(s, \pi^{\star}(s)) S) \cr 
       &\le 2  H^2 S \clipavgC MK  =: W   
    \end{align}
    using the fact that $|\widetilde{G}^{-j,m}_{\nsyn(j),h}(s_{j,h}^m,a_{j,h}^m)| \le H$ shown in  \eqref{eq:Gapprox-property-Gbound} and $\frac{d_h^{\pi^{\star}}(s, \pi^{\star}(s))}{\min\{d_h^{\pi^{\star}}(s, \pi^{\star}(s)), 1/S\}} \le 1+ d_h^{\pi^{\star}}(s, \pi^{\star}(s)) S$.

  \item \textbf{Bound on the approximation gap of $\widetilde{Y}_{j,h}^m$.}
    The approximation gap  of $\widetilde{Y}_{j,h}^m$ is bounded as follows:
    \begin{align} \label{eq:Y-approx-gap}
        &\left | \sum_{j=1}^{K} \sum_{m=1}^M \left(\widetilde{Y}_{j,h}^m - Y_{j,h}^m \right ) \right | \cr 
        &= \left | \sum_{v=1}^{\nsyn(K)} \sum_{m=1}^M \sum_{t_{v-1} <j \le t_v} \sum_{(s,a) \in \cS \times \cA} 
        \left ( d_h^m(s,a)  -   \one\{(s,a) = (s_{j,h}^m,a_{j,h}^m)\} \right ) \frac{d_h^{\pi^{\star}}(s,a)}{d^{\mathsf{avg}}_h(s,a)} 
          (\widetilde{G}^{-j,m}_{v,h}(s,a) - G_{v,h}(s,a)) \right | \cr 
        &\stackrel{\mathrm{(i)}}{=} \sum_{v=1}^{\nsyn(K)} \sum_{m=1}^M \sum_{t_{v-1} <j \le t_v}\sum_{(s,a) \in \cS \times \cA} 
          \one\{(s,a) = (s_{j,h}^m,a_{j,h}^m)\} \left ( 1-d_h^m(s,a)  \right ) \frac{d_h^{\pi^{\star}}(s,a)}{d^{\mathsf{avg}}_h(s,a)} 
          \left | \widetilde{G}^{-j,m}_{v,h}(s,a) - G_{v,h}(s,a) \right | \cr 
        &\stackrel{\mathrm{(ii)}}{\le} \sum_{v=1}^{\nsyn(K)} \sum_{m=1}^M \sum_{t_{v-1} <j \le t_v} \sum_{(s,a) \in \cS \times \cA} 
          \one\{(s,a) = (s_{j,h}^m,a_{j,h}^m)\}  \frac{d_h^{\pi^{\star}}(s,a)}{d^{\mathsf{avg}}_h(s,a)}  
          \min \left \{ \frac{2H^2}{N_{t_v, h}(s,a)}, H \right \} \cr 
        &\stackrel{\mathrm{(iii)}}{\le}  \clipavgC \sum_{s \in \cS}\sum_{v=1}^{\nsyn(K)}   
          n_{t_v,h}(s,\pi^{\star}(s))   
          \frac{d_h^{\pi^{\star}}(s, \pi^{\star}(s))}{\min\{d_h^{\pi^{\star}}(s, \pi^{\star}(s)), 1/S\}}
          \min \left \{ \frac{2H^2}{N_{t_v, h}(s,\pi^{\star}(s))}, H \right \} \cr 
      &\stackrel{\mathrm{(iv)}}{\le}  2H^2 \clipavgC \sum_{s \in \cS} (1+ d_h^{\pi^{\star}}(s, \pi^{\star}(s)) S)
        \sum_{v=1}^{\nsyn(K)}   
          \min \left \{ \frac{ n_{t_v,h}(s,\pi^{\star}(s))   }{ N_{t_v, h}(s,\pi^{\star}(s))},  n_{t_v,h}(s,\pi^{\star}(s))    \right \} \cr 
      &\stackrel{\mathrm{(v)}}{\lesssim}  \clipavgC   H^2 S  
      \end{align}
      where (i) holds because $\widetilde{\psi}^{-j,m}_{v,v,h}(s,a) = \psi^{-j,m}_{v,v,h}(s,a)$ if $(s_{j,h}^m,a_{j,h}^m) \neq (s,a)$ and $\widetilde{G}^{-j,m}_{v,h}(s,a) = G_{v,h}(s,a)$  according to \eqref{eq:G-recursion}, (ii) follows from \eqref{eq:Gapprox-property-Gapproxgap}, (iii) naturally holds according to the definition of $\clipavgC$, (iv) holds because $\frac{d_h^{\pi^{\star}}(s, \pi^{\star}(s))}{\min\{d_h^{\pi^{\star}}(s, \pi^{\star}(s)), 1/S\}} \le 1+ d_h^{\pi^{\star}}(s, \pi^{\star}(s)) S$, and (v) holds because for any $z\in [\nsyn(K)]$,
  \begin{align} \label{eq:ntoNsum}
    \sum_{v=1}^{z}   
    \frac{ n_{t_v,h}(s,\pi^{\star}(s))   }{ N_{t_v, h}(s,\pi^{\star}(s))}
    \le 1 + \log{ ( N_{t_z, h}(s,\pi^{\star}(s)))},
  \end{align}
  according to Lemma~\ref{lemma:seqsum}.

\end{itemize}

  Now, combining the bounds obtained above (cf.~\eqref{eq:Y-freedman} and \eqref{eq:Y-approx-gap}) into \eqref{eq:e4a-aux1-decomp}, we conclude that
  \begin{align}
    \left | \sum_{j=1}^{K} \sum_{m=1}^M  Y_{j,h}^m \right |
    \lesssim  \sqrt{ H^2 M K S \clipavgC  }  +  H^2 S \clipavgC  
    = M \left (  \sqrt{ \frac{H^2  K S \clipavgC}{M}  }  + \frac{H^2 S \clipavgC   }{M} \right )
  \end{align}
  which completes the proof.
  
  \subsubsection{Proof of \eqref{eq:Gapprox-property}}
  \label{proof:Gapprox-property}
  \paragraph{Proof of \eqref{eq:G-recursion}.}
  We will proof \eqref{eq:G-recursion} by considering different cases separately. When $v=1$, we have
  \begin{align}
      G_{v,h}(s,a) &=   P_{h,s,a}(V^{\star}_{h+1}-V_{t_{v-1}, h+1}) \psi_{1,1,h}(s,a) \notag \\
      & =  P_{h,s,a}(V^{\star}_{h+1}-V_{0, h+1})  \sum_{m=1}^M \left (\sum_{i \in l_{t_1,h}^m(s,a)} \omega_{i,t_1,h}^m(s,a)  \right ) = P_{h,s,a}(V^{\star}_{h+1}-V_{0, h+1}) 
  \end{align}
  where the second equality follows from the definition of $\psi_{u,v,h}(s,a)$ in \eqref{eq:recursion-original}, and the last equality holds since
$$\sum_{m=1}^M \sum_{i \in l_{t_1,h}^m(s,a)} \omega_{i,t_1,h}^m(s,a)   = \frac{ (H+1)n_{t_1, h}}{N_{t_1,h} +  H n_{t_1,h}} = \frac{ (H+1)n_{t_1, h}}{( H+1) n_{t_1,h}} =1. $$

When $v>1$, invoking the definition of $\omega_{i,t_v,h}^m$ in   \eqref{eq:def-omegaik} yields that for any $u<v$,
    \begin{align}
      \psi_{u,v,h}(s,a)
      &= \sum_{m=1}^M \sum_{i \in l_{t_u,h}^m(s,a)} \omega_{i,t_v,h}^m(s,a) \cr
      &= \frac{ (H+1)n_{t_u, h}}{N_{t_v,h} +  H n_{t_v,h}} \left (\prod_{x=u}^{v-1} \frac{ N_{t_{x},h}}{N_{t_x,h} + H n_{t_x,h}} \right)\cr
      &= \frac{ (H+1)n_{t_u, h}}{N_{t_{v-1},h} +  H n_{t_{v-1},h}} \left (\prod_{x=u}^{v-2} \frac{ N_{t_{x},h}}{N_{t_x,h} + H n_{t_x,h}} \right)
        \frac{N_{t_{v-1},h} }{N_{t_{v},h} +  H n_{t_{v},h}} \cr 
      &=  \psi_{u,v-1,h}(s,a) (1-\psi_{v,v,h}(s,a)).
    \end{align}
where the second equality holds by $\phi(i) = u$ for all $ i \in l_{t_u,h}^m(s,a)$ and the fact $\sum_{m=1}^M \sum_{i \in l_{t_u,h}^m(s,a)} 1 = n_{t_u,h}$, and the last equality holds by
$1-\psi_{v,v,h}(s,a) = 1 -  \frac{ (H+1)n_{t_v, h}}{N_{t_v,h} +  H n_{t_v,h}} = \frac{ N_{t_{v-1},h} +  (H+1) n_{t_v,h} - (H+1)n_{t_v, h}}{N_{t_{v},h} +  H n_{t_v,h}} = \frac{ N_{t_{v-1},h}}{N_{t_{v},h} +  H n_{t_v,h}}$.

Consequently, inserting the above fact back into \eqref{eq:defnition-G} complete the proof by showing that
    \begin{align}
      G_{v,h}(s,a)
      &= \sum_{u=1}^{v}  P_{h,s,a}(V^{\star}_{h+1}-V_{t_{u-1}, h+1}) \psi_{u,v,h}(s,a) \cr
      &= P_{h,s,a}(V^{\star}_{h+1}-V_{t_{v-1}, h+1}) \psi_{v,v,h}(s,a) + \sum_{u=1}^{v-1}  P_{h,s,a}(V^{\star}_{h+1}-V_{t_{u-1}, h+1}) \psi_{u,v,h}(s,a) \cr      &= P_{h,s,a}(V^{\star}_{h+1}-V_{t_{v-1}, h+1}) \psi_{v,v,h}(s,a) + (1-\psi_{v,v,h}(s,a)) \sum_{u=1}^{v-1}  P_{h,s,a}(V^{\star}_{h+1}-V_{t_{u-1}, h+1}) \psi_{u,v-1,h}(s,a) \cr
      &= P_{h,s,a}(V^{\star}_{h+1}-V_{t_{v-1}, h+1}) \psi_{v,v,h}(s,a) + (1-\psi_{v,v,h}(s,a)) G_{v-1,h}(s,a).
    \end{align}
  
  \paragraph{Proof of \eqref{eq:Gapprox-property-Gbound}.}
  First, applying  \eqref{eq:V_lower} in Lemma~\ref{lemma:pessimistic_value} gives $  G_{v,h}(s,a) \geq 0$. Then we focus on deriving the upper bound $G_{v,h}(s,a)$. Towards this, we observe that
      \begin{align} \label{eq:G-bound}
       G_{v,h}(s,a)
        &= \sum_{u=1}^{v}  P_{h,s,a}(V^{\star}_{h+1}-V_{t_{u-1}, h+1}) \psi_{u,v,h}(s,a) \cr
        &\le   P_{h,s,a}(V^{\star}_{h+1}-V_{0, h+1}) \sum_{u=1}^{v} \psi_{u,v,h}(s,a) \cr
        &\le H \sum_{u=1}^{v}\psi_{u,v,h}(s,a) \notag \\
        & = H \sum_{u=1}^{v} \sum_{m=1}^M \left (\sum_{i \in l_{t_u,h}^m(s,a)} \omega_{i,t_v,h}^m(s,a)  \right ) \le H,
      \end{align}
      where the first and second inequalities hold by the fact 
     $P_{h,s,a}(V^{\star}_{h+1}-V_{t_x, h+1}) \le P_{h,s,a}(V^{\star}_{h+1}-V_{0, h+1})  \le H$ for any $x \in [\nsyn(K)]$ (see the monotonicity of the value estimates in \eqref{eq:vglobal-monotone} and the basic bound $\|V^{\star}_{h+1}\|_\infty \leq H$), the last equality arises from the definition of $\psi_{u,v,h}(s,a)$ in \eqref{eq:recursion-original}, and the last inequality follows from \eqref{eq:lr-wsum} in Lemma~\ref{lemma:lr_bound}. 

      Similarly, the same facts hold for $\widetilde{G}^{-j,m}_{v,h}(s,a)$, which can be derived in the same manner. We omit it for conciseness.
      
  \paragraph{Proof of \eqref{eq:Gapprox-property-Gapproxgap}.} 
  Consider  $v = \nsyn(j)$.
  If $v=1$, combing \eqref{eq:G-recursion} and  \eqref{eq:Gapprox-recursion} directly gives $\widetilde{G}^{-j,m}_{v,h}(s,a) = G_{v,h}(s,a)$. 
Then we turn to the case when $v>1$ and bound the term of interest in two different cases, respectively.
  \begin{itemize}
      \item When $(s_{j,h}^m,a_{j,h}^m) \neq (s,a)$. In this case, invoking the definition in \eqref{eq:definition-psi-new} gives
      \begin{align}
          \widetilde{\psi}^{-j,m}_{v,v,h}(s,a) =  \frac{ (H+1)n_{t_v,h}(s,a) }{N_{t_{v-1},h}(s,a) +  (H+1)n_{t_v,h}(s,a) } = \psi^{-j,m}_{v,v,h}(s,a),
      \end{align}
      which indicates (see the definition in \eqref{eq:Gapprox-recursion})
      \begin{align}\label{eq:diff-G-case1}
          \widetilde{G}^{-j,m}_{v,h}(s,a) = G_{v,h}(s,a) 
      \end{align}

      \item When $(s_{j,h}^m,a_{j,h}^m) = (s,a)$. In view of \eqref{eq:G-recursion} and  \eqref{eq:Gapprox-recursion}, it holds that:
      \begin{align} \label{eq:Gapproxgap-bound}
        &| \widetilde{G}^{-j,m}_{v,h}(s,a) - G_{v,h}(s,a) | \cr
        &= \left | (\widetilde{\psi}^{-j,m}_{v,v,h}(s,a)- \psi_{v,v,h}(s,a)) P_{h,s,a}(V^{\star}_{h+1}-V_{t_{v-1}, h+1})  + (\psi_{v,v,h}(s,a)-\widetilde{\psi}^{-j,m}_{v,v,h}(s,a))G_{v-1,h}(s,a) \right | \cr
        &= \left | (\psi_{v,v,h}(s,a) - \widetilde{\psi}^{-j,m}_{v,v,h}(s,a))  ( G_{v-1,h}(s,a) - P_{h,s,a}(V^{\star}_{h+1}-V_{t_{v-1}, h+1})   \right | \cr
        &\le \left | \psi_{v,v,h}(s,a) - \widetilde{\psi}^{-j,m}_{v,v,h}(s,a) \right| \max\left\{ G_{v-1,h}(s,a) ,  \|P_{h,s,a}\|_1 \left\|V^{\star}_{h+1}-V_{t_{v-1}, h+1}\right\|_\infty \right\} \cr
        &\stackrel{\mathrm{(i)}}{\le} H \left|\psi_{v,v,h}(s,a) - \widetilde{\psi}^{-j,m}_{v,v,h}(s,a)\right| \cr
        &\stackrel{\mathrm{(ii)}}{\le} \min \left \{H, \frac{2H^2  }{N_{t_v, h}(s,a)} \right \}, 
      \end{align}
      where (i) holds by \eqref{eq:Gapprox-property-Gbound}, $\|P_{h,s,a}\|_1 = 1$, and $\left\|V^{\star}_{h+1}-V_{t_{v-1}, h+1}\right\|_\infty \leq H$. Here, (ii)  can be verified by
      \begin{align} \label{eq:psiapproxgap-bound}
        0 &\stackrel{\mathrm{(iii)}}{\le}  \psi_{v,v,h}(s,a) - \widetilde{\psi}^{-j,m}_{v,v,h}(s,a) \cr
        &= \frac{(H+1) n_{t_v, h}(s,a) }{ N_{t_{v-1}, h}(s,a) + (H+1)n_{t_v, h}(s,a)} - \frac{(H+1) (n_{t_v, h}(s,a) -\one\{(s,a) = (s_{j,h}^m,a_{j,h}^m)\} ) }{ N_{t_{v-1}, h}(s,a) + (H+1) (n_{t_v, h}(s,a)-\one\{(s,a) = (s_{j,h}^m,a_{j,h}^m) \}) } \cr 
        &= \frac{(H+1) n_{t_v, h}(s,a) }{ N_{t_{v-1}, h}(s,a) + (H+1)n_{t_v, h}(s,a)} - \frac{(H+1) (n_{t_v, h}(s,a) - 1  ) }{ N_{t_{v-1}, h}(s,a) + (H+1) (n_{t_v, h}(s,a)-1  )} \cr 
        &\le \frac{(H+1)  }{N_{t_{v-1}, h}(s,a) + (H+1)n_{t_v, h}(s,a)} \cr
        &\le \min \left \{1, \frac{2H  }{N_{t_v, h}(s,a)} \right \}.
      \end{align}
      where (iii) holds by the fact that $\frac{x}{a +x}$ is monotonically increasing with $x$ when $a,x>0$.
  \end{itemize}

    \subsubsection{Proof of Lemma~\ref{lemma:e4a-aux2}} \label{proof:e4a-aux2}      
      For each $j \in [K]$, let 
      \begin{align}
        Z_{j,h}^m
        &\defn \sum_{(s,a) \in \cS \times \cA}  \left ( \one\{(s, a) = (s_{j,h}^m,a_{j,h}^m)\} -  d_h^m(s,a) \right ) \frac{d_h^{\pi^{\star}}(s, a)}{M d^{\mathsf{avg}}_h(s,a)} P_{h,s,a}(V^{\star}_{h+1}-V_{t_{\nsyn(j)-1}, h+1}).
      \end{align}
      Then, to prove Lemma~\ref{lemma:e4a-aux2}, it suffices to show
      $
        \left | \sum_{j=1}^{K} \sum_{m=1}^M  Z_{j,h}^m \right |
        \lesssim \sigma_{\mathsf{aux},2}.
        $

      Since $V_{t_{\nsyn(j)-1}, h+1}$ is fully determined by the events before the $j$-th episode, $\mexp_{j-1}[Z_{j,h}^m] = 0$, where we denote           $$\mexp_{j-1} [\cdot]= \mexp[\cdot| \{ (s_{i,h}^{m'},a_{i,h}^{m'}), ~ V^{m'}_{i, h+1} \}_{i<j, m' \in [M]} ].$$
      Thus,
      we can apply the Freedman's inequality as follows:
      \begin{align}
        \left | \sum_{j=1}^{K} \sum_{m=1}^M  Z_{j,h}^m \right |
        &\le \sqrt{ 8 W \log{\frac{2H}{\delta}} } + \frac{8}{3} B \log{\frac{2H}{\delta}}
          \lesssim \sqrt{\frac{H^2 K S \clipavgC  }{M  }} + \frac{H S \clipavgC}{M}
      \end{align}
      using the following properties:
        \begin{align}
          |Z_{j,h}^m|
          &\le \frac{2 \clipavgC H }{M}  \left(\sum_{s \in \cS} (1+ d_h^{\pi^{\star}}(s, \pi^{\star}(s)) S)   \right)   
            \le \frac{4 H S \clipavgC}{M}  =: B\\
          \sum_{j=1}^{K} \sum_{m=1}^M \mexp_{j-1} [(Z_{j,h}^m)^2]
          &\le \sum_{j=1}^{K} \sum_{m=1}^M  \mexp_{(s_{j,h}^m,a_{j,h}^m) \sim d_{h}^m} \left [ \left ( \frac{d_h^{\pi^{\star}}(s_{j,h}^m,a_{j,h}^m) }{M d^{\mathsf{avg}}_h(s_{j,h}^m,a_{j,h}^m)} P_{h,s,a}(V^{\star}_{h+1}-V_{t_{\nsyn(j)-1}, h+1}) \right)^2 \right]\cr
      &\le H^2 \sum_{j=1}^{K} \sum_{m=1}^M \sum_{s \in \cS} d_h^m(s,\pi^{\star}(s))
        \left (\frac{d_h^{\pi^{\star}}(s,\pi^{\star}(s))}{M d^{\mathsf{avg}}_h(s,\pi^{\star}(s))} \right )^2 \cr
      &\le \frac{H^2 \clipavgC }{M} \sum_{s \in \cS} \sum_{j=1}^{K} \left (\frac{d_h^{\pi^{\star}}(s,\pi^{\star}(s))}{M d^{\mathsf{avg}}_h(s,\pi^{\star}(s))}  \right ) (1+ d_h^{\pi^{\star}}(s, \pi^{\star}(s)) S)  \sum_{m=1}^M  d_h^m(s,\pi^{\star}(s)) \cr 
      &= \frac{H^2 \clipavgC}{M} \sum_{s \in \cS} \sum_{j=1}^{K} d_h^{\pi^{\star}}(s,\pi^{\star}(s)) (1+ d_h^{\pi^{\star}}(s, \pi^{\star}(s)) S)    \cr 
      &= \frac{2 H^2 K S \clipavgC }{M} =:W,    
        \end{align}
        which follows from that fact $0 \le \|V^{\star}_{h+1}-V_{t_{\nsyn(j)-1}, h+1} \|_{\infty}\le H$ and $\frac{d_h^{\pi^{\star}}(s, \pi^{\star}(s))}{\min\{d_h^{\pi^{\star}}(s, \pi^{\star}(s)), 1/S\}} \le 1+ d_h^{\pi^{\star}}(s, \pi^{\star}(s)) S$.

\subsection{Proof of Corollary~\ref{cor:fedq-comm}} 
Note that if $T \asymp \frac{ H^7 S \clipavgC}{M \varepsilon^2}$, it always holds that
\begin{align}
    MT \gtrsim H^5 S \clipavgC ~~\text{and}~~ H \le \sqrt{\frac{H S\clipavgC T}{M}},
\end{align}
as long as $\varepsilon \le H$ and $\varepsilon \le \frac{H^3 S \clipavgC}{M}$.
Now, we obtain the number of communication rounds of the specified schedules, periodic and exponential synchronization.

  \paragraph{Periodic synchronization.}
  Consider $\tau \asymp \sqrt{\frac{H S\clipavgC T}{M}}$.
  Then, since $MT \gtrsim H S \clipavgC$, the value gap is bounded as
  \begin{align}
    V_{1}^{\star}(\rho) - V_{1}^{\hat{\pi}}(\rho)
    \lesssim \frac{H^4 S \clipavgC }{MT} 
    + \sqrt{\frac{H^7 S \clipavgC }{MT}}
    + \frac{H^3}{T} \sqrt{\frac{H S\clipavgC T}{M}}
    \lesssim \sqrt{\frac{H^7 S \clipavgC }{MT}}.
  \end{align}
  In this case, the number of synchronizations $\nsyn(K) = |\synset_{\mathsf{period}}(K, \tau)|$ is
  $$\nsyn(K)  =  \Big \lceil \frac{K}{\tau}  \Big \rceil  \lesssim \sqrt{\frac{MK}{H^2S\clipavgC }} \asymp \sqrt{\frac{MT}{H^3 S\clipavgC }} \asymp \frac{H^2}{\varepsilon}.$$

  \paragraph{Exponential synchronization.}
  Using the fact that $MT \gtrsim H S \clipavgC$ and $\tau_1 = H \le \sqrt{\frac{H S\clipavgC T}{M}}$ when $\varepsilon \le \frac{H^3 S \clipavgC}{M}$, the value gap is bounded as
  \begin{align}
    V_{1}^{\star}(\rho) - V_{1}^{\hat{\pi}}(\rho)
    \lesssim \frac{H^4 S \clipavgC }{MT} 
    + \sqrt{\frac{H^7 S \clipavgC }{MT}}
    + \frac{H^3}{T} \sqrt{\frac{H S\clipavgC T}{M}}
    \lesssim \sqrt{\frac{H^7 S \clipavgC }{MT}}.
  \end{align}
  To continue, note that if $\gamma = \frac{2}{H}$ and $\tau_1 = H$, for any $u \ge 1$, $\tau_u$  is bounded as
  $$\Big(1+\frac{1}{H} \Big)^{u-1} H \le \tau_u \le \Big(1+\frac{2}{H} \Big)^{u-1} H, $$
  since $$\Big(1+\frac{1}{H} \Big) \tau_{i} \le \Big(1+\frac{2}{H} \Big) \tau_{i} -1 \le \tau_{i+1} = \left\lfloor \Big(1+\frac{2}{H} \Big) \tau_{i} \right\rfloor \le \Big(1+\frac{2}{H} \Big) \tau_{i}$$
  given the fact that $\tau_i \ge H$ for any $i\ge 1$.
    Then, considering the minimum number of synchronizations $\nsyn(K) = |\synset_{\mathsf{exp}}(K,\gamma)| $ satisfying 
    $$\sum_{u=1}^{\nsyn(K)} \tau_{u} \ge H \sum_{u=1}^{\nsyn(K)} \Big(1+\frac{1}{H} \Big)^{u-1} = H^2 \Big( \Big(1+\frac{1}{H} \Big)^{\nsyn(K)} - 1 \Big) \ge K,$$
    we obtain
    \begin{align}
        \nsyn(K) =  \left \lceil \frac{\log{(\frac{K}{H^2} + 1)}}{\log{( 1+\frac{1}{H})}} \right \rceil \le 1+ (1+H)\log{\Big(\frac{K}{H^2} + 1 \Big)} \lesssim H
    \end{align}
    because $\frac{x}{x+1}\le \log(1+x)$ for any $x>-1$.

\end{document}